\documentclass[11pt]{article} 
\usepackage{latexsym,amssymb,enumerate,amsmath,epsfig,amsthm,authblk,dsfont}
\usepackage{soul}
\usepackage{url}
\usepackage{smile}
\usepackage{graphicx,subfigure} 
\usepackage{algorithm}
\usepackage{algorithmic}
\usepackage{todonotes}
\usepackage{epstopdf}
\usepackage{wrapfig}
\usepackage[colorlinks, linkcolor=blue, anchorcolor=blue, citecolor=blue,hypertexnames=false]{hyperref}
\usepackage[margin=1in]{geometry}
\usepackage[normalem]{ulem}
\usepackage[export]{adjustbox}
\usepackage{mathtools, cuted}
\usepackage{mathrsfs}
\usepackage{natbib}
\usepackage{enumerate}
\usepackage{lineno}

\newcommand{\ReLU}{\mathrm{ReLU}}
\newcommand{\var}{\mathrm{Var}}
\newcommand{\calH}{\mathcal{H}}

\newcommand{\calT}{\mathcal{T}}
\newcommand{\calD}{\mathcal{D}}
\newcommand{\calC}{\mathcal{C}}
\newcommand{\calP}{\mathcal{P}}
\newcommand{\calA}{\mathcal{A}}
\newcommand{\asj}{\calA^s_{\theta}}

\newcommand{\tp}{\tilde{p}}

\newcommand{\Cjk}{C_{j,k}}
\newcommand{\pjk}{p_{j,k}}
\newcommand{\jk}{{j,k}}
\newcommand{\cjk}{\mathbf{c}_{j,k}}
\newcommand{\din}{{d_{\rm in}}}

\newcommand{\ttimes}{\widetilde{\times}}

\newtheorem{examplea}{Example}

\newtheorem{exampleb}{Example}

\newtheorem{examplec}{Example}

\allowdisplaybreaks

\date{}

\newcommand{\commentout}[1]{}

\begin{document}

	\title{Deep Neural Networks are Adaptive to Function Regularity and Data Distribution in Approximation and Estimation}
	\author{
		Hao Liu, Jiahui Cheng, Wenjing Liao
		\thanks{Hao Liu is affiliated with the Math department of Hong Kong Baptist University; Jiahui Cheng and Wenjing Liao are affiliated with the School of Math at Georgia Tech; Email: \text{haoliu@hkbu.edu.hk, $\{$jcheng328, wliao60$\}$@gatech.edu}. This research is partially supported by National Natural Science Foundation of China  12201530, HKRGC ECS 22302123, HKBU 179356, NSF DMS--2012652,  NSF DMS-2145167 and DOE SC0024348.}}
	
	\maketitle
	
	\begin{abstract}

		Deep learning has exhibited remarkable results across diverse areas. To understand its success, substantial research has been directed towards its theoretical foundations. Nevertheless, the majority of these studies examine how well deep neural networks can model functions with uniform regularity. In this paper, we explore a different angle: how deep neural networks can adapt to different
		regularity in functions across different locations and scales and nonuniform data distributions. More precisely, we focus on a broad class of functions defined by nonlinear tree-based approximation. This class encompasses a range of function types, such as functions with uniform regularity and discontinuous functions. We develop nonparametric  approximation and estimation theories  for this function class using deep ReLU networks. Our results show that deep neural networks are adaptive to different
		regularity of functions and nonuniform data distributions at different locations and scales. We apply our results to several function classes, and derive the corresponding approximation and generalization errors. The validity of our results is demonstrated through  numerical experiments.
	\end{abstract}
	
	\section{Introduction}

	Deep learning has achieved significant success in  practical applications with high-dimensional data, such as computer vision \citep{krizhevsky2012imagenet}, natural language processing \citep{graves2013speech,young2018recent}, health care \citep{miotto2018deep,jiang2017artificial} and bioinformatics \citep{Alipanahi2015PredictingTS,zhou2015predicting}. 
	The success of deep learning  demonstrates the  power of neural networks in representing and learning complex operations on high-dimensional data.

	In the past decades, the representation power of neural networks has been extensively studied. Early works in literature focused on shallow (two-layer) networks with continuous sigmoidal activations (a function $\sigma(x)$ is sigmoidal, if $\sigma(x) \rightarrow 0$ as $x \rightarrow -\infty$, and $\sigma(x) \rightarrow 1$ as $x \rightarrow \infty$) for a universal approximation of continuous functions in a unit hypercube \citep{irie1988capabilities, funahashi1989approximate, cybenko1989approximation, hornik1991approximation, chui1992approximation, leshno1993multilayer,barron1993universal, mhaskar1996neural}. The universal approximation theory of feedforward neural networks with a ReLU activation $\sigma(x)=\max(0,x)$ was studied in  \cite{lu2017expressive,hanin2017universal,daubechies2022nonlinear,yarotsky2017error,schmidt2017nonparametric,suzuki2018adaptivity}. In particular, the approximation theories of ReLU networks have been established for the Sobolev $W^{k,\infty}$ \citep{yarotsky2017error}, H\"{o}lder \citep{schmidt2017nonparametric} and Besov \citep{suzuki2018adaptivity} functions. These works guarantee that, the Sobolev, H\"{o}lder, or Besov function class can be well approximated by a ReLU network function class with a properly chosen network architecture. The approximation error in these works was given in terms of certain function norm. Furthermore, the works in \cite{guhring2020error,hon2021simultaneous,liu2022benefits} proved the approximation error in terms of the  Sobelev norm, which guaranteed the approximation error for the function and its derivatives simultaneously. In terms of the network architecture, feedforward neural networks were considered in the vast majority of approximation theories. Convolutional neural networks were considered in \cite{zhou2020universality,petersen2020equivalence}, and convolutional residual networks were considered in \cite{oono2019approximation,liu2021besov}.

	It has been widely believed that deep neural networks are adaptive to complex data structures. Recent progresses have been made towards theoretical justifications that deep neural networks are adaptive to low-dimensional structures in data.
	Specifically, function approximation theories have been established for H\"older and Sobolev functions supported on low-dimensional manifolds \citep{chen2019efficient,schmidt2019deep, cloninger2020relu,nakada2020adaptive,liu2021besov}. The network size in these works crucially depends on the intrinsic dimension of data, instead of the ambient dimension. In the task of regression and classification, the sample complexity of neural networks \citep{chen2019nonparametric,nakada2020adaptive,liu2021besov} depends on the intrinsic dimension of data, while the ambient dimension does not affect the rate of convergence.

	This paper answers another interesting question about the adaptivity of deep neural networks: 
	\textit{ How does deep neural networks adapt to the function regularity and data distribution at different locations and scales?}
	The answer of this question is beyond the scope of existing function approximation and estimation theory of neural networks. The Sobolev $W^{k,\infty}$ and H\"older functions are uniformly regular within the whole domain. The analytical technique to build the approximation theory of these functions relies on accurate local approximations everywhere within the domain. In real applications, functions of interests often exhibit different regularity at different locations and scales. Empirical experiments have demonstrated that deep neural networks are capable of extracting interesting information at various locations and scales \citep{chung2016hierarchical,haber2018learning}. However, there are limited works on theoretical justifications of the adaptivity of neural networks.
	
	In this paper, we re-visit the nonlinear approximation theory \citep{devore1998nonlinear} in the classical multi-resolution analysis \citep{mallat1999wavelet,daubechies1992ten}. Nonlinear approximations allow one to approximate functions beyond linear spaces. The smoothness of the function can be defined according to the rate of approximation error versus the complexity of the elements. In many settings, such characterization of smoothness is significantly weaker than the uniform regularity condition in the Sobolev or H\"older class \citep{devore1998nonlinear}.

	We focus on the tree-based nonlinear approximations with piecewise polynomials \citep{binev2007universal,binev2005universal,cohen2001tree}. Specifically, the domain of functions is partitioned to multiscale dyadic cubes associated with a master tree. If we build piecewise polynomials on these multiscale dyadic cubes, we naturally obtain multiscale piecewise polynomial approximations. A refinement quantity is defined at every node to quantify how much the error decreases when the node is refined to its children. A thresholding of the master tree based on this refinement quantity gives rise to a truncated tree, as well as an adaptive partition of the domain. Thanks to the thresholding technique, we can define a function class whose regularity is characterized by how fast the size of the truncated tree grows with respect to the level of the threshold. This is a large function class containing the H\"older and piecewise H\"older functions as special cases. 
	
	Our main contributions can be summarized as:
	
	\begin{enumerate}
		\item We establish the approximation theory of deep ReLU networks for a large class of functions whose regularity is defined  according to the nonlinear tree-based approximation theory. This function class allows the regularity of the function to vary at different locations and scales.   
		
		\item We  provide several examples of functions in this class which exhibit different information at different locations and scales. These examples are beyond the characterization of function classes with uniform regularity, such as the H\"older class.
		
		\item A nonparametric estimation theory for this large function class is established with deep ReLU networks, which is validated by  numerical experiments. 
		
		\item Our results demonstrate that, when deep neural networks are representing functions, it does not require a uniform regularity everywhere on the domain. Deep neural networks are automatically adaptive to the regularity of functions at different locations and scales.
	\end{enumerate}
	
	In literature, adaptive function approximation and estimation has been  studied for classical methods \citep{devore1998nonlinear}, including free-knot spline \citep{jupp1978approximation}, adaptive smoothing splines \citep{wahba1995discussion,pintore2006spatially,liu2010data,wang2013smoothing}, nonlinear wavelet  \citep{cohen2001tree,donoho1998minimax,donoho1995wavelet},
	and adaptive piecewise polynomial approximation \citep{binev2007universal,binev2005universal}. 
	Based on traditional methods for estimating functions with uniform regularity, these methods allow the smoothing parameter, the kernel band width or knots placement to vary spatially to adapt to the varying regularity.  Kernel methods with variable bandwidth were studied in \citet{muller1987variable} and local polynomial estimators were studied in \citet{fan1996local}. Based on traditional smoothing splines with a global smoothing parameter, \citet{wahba1995discussion} suggested to replace the smoothing parameter by a roughness penalty function. This idea was then studied in \citet{pintore2006spatially,liu2010data} by using piecewise constant roughness penalty, and in \cite{wang2013smoothing} with a more general roughness penalty. 
	A locally penalized spline estimator was proposed and studied in \citet{ruppert2000theory}, in which a penalty function was applied to spline coefficients and was knot-dependent. 
	Adaptive wavelet shrinkage was studied in \cite{donoho1994ideal,donoho1995adapting,donoho1998minimax}, in which the authors used selective wavelet reconstruction, adaptive thresholding and nonlinear wavelet shrinkage to achieve  adaptation to spatially varying regularity, and proved the minimax optimality. 
	A Bayesian mixture of splines method was proposed in \citet{wood2002bayesian}, in which each component spline had a locally defined smoothing parameter. Other methods include regression splines \citep{fridedman1991multivariate,smith1996nonparametric,denison1998automatic},  hybrid  smoothing splines and regression splines \citep{luo1997hybrid}, and the trend filtering method \citep{tibshirani2014adaptive}.
	The minimax theory for adaptive nonparametric estimator was established in \cite{cai2012minimax}. Most of the works mentioned above focused on one-dimensional problems.  For high dimensional problems, an additive model was considered in \citet{ruppert2000theory}. Recently, the Bayesian additive regression trees  were studied in  \citet{jeong2023art} for estimating a class of sparse piecewise
	heterogeneous anisotropic H\"{o}lder continuous functions in high dimension. 
	
	Classical methods mentioned above adapt to varying regularity of the target functions through a careful selection of some adaptive parameter, such as the location of knots, kernel bandwidth, roughness penalty and adaptive tree structure.
	These methods require the knowledge about or an estimation of how the regularity of the target function changes. Compared to classical methods, deep learning solves the regression problem by minimizing the empirical risk in \eqref{eq.empirial.loss}, so the same optimization problem can be applied to various functions without explicitly figuring out where the regularity of the underlying function changes. Such kind of automatic adaptivity is crucial for real-world applications.

	The connection between neural networks and adaptive spline approximation has been studied in \cite{daubechies2022nonlinear,devore2021neural,liu2022adaptive,petersen2018optimal,imaizumi2019deep}. In particular, an adaptive network enhancement method was proposed in \cite{liu2022adaptive} for the best least-squares approximation using two-layer neural networks. The adaptivity of neural networks to data distributions was considered in \cite{zhang2023effective}, where the concept of an effective Minkowski dimension was introduced and applied to anisotropic Gaussian distributions. The approximation error and generalization error for learning piecewise H\"{o}lder functions in $\RR^d$ are developed in \citet{petersen2018optimal} and \citet{imaizumi2019deep}, respectively. In the settings of \citet{petersen2018optimal} and \citet{imaizumi2019deep}, each discontinuity boundary is parametrized by 
	a $(d-1)$-dimensional H\"{o}lder function, which is called a horizon function.
	In this paper, we consider a function class based on nonlinear tree-based approximation, and provide approximation and generalization theories of deep neural networks for this function class, as well as several examples related with practical applications, which are not implied by existing works. For piecewise H\"{o}lder functions, our setting only assumes the boundary of each piece has Minkowski dimension $d-1$, which is more general than that considered in \citet{petersen2018optimal,imaizumi2019deep}, see Section \ref{sec.pieceholdernd.approx} and \ref{sec.pieceholdernd.gene} for more detailed discussions.
	
	Our paper is organized as follows: In Section \ref{sec.notation}, we introduce notations and concepts  used in this paper. Tree based adaptive approximation and some examples are presented in Section \ref{sec.adaptive}. We present our main results, the adaptive approximation and generalization theories of deep neural networks in Section \ref{sec.main}, and the proofs are deferred to Section \ref{sec.proof}. Our theories is validated by  numerical experiments in Section \ref{sec.numerical}. This paper is concluded in Section \ref{sec.conclusion}.

	\section{Notation and Preliminaries}
	\label{sec.notation}
	In this section, we introduce our notation, some preliminary definitions and ReLU networks.
	
	\subsection{Notation}
	We use normal lower case letters to denote scalars, and bold lower case letters to denote vectors.
	For a vector $\xb\in \RR^d$, we use $x_i$ to denote the $i$-th entry of $\xb$. The standard 2-norm of $\xb$ is  $\|\xb\|_2 = (\sum_{i=1}^d x_i^2)^{\frac 1 2}$. For a scalar $a>0$, $\lfloor a \rfloor$ denotes the largest integer that is no larger than $a$, $\lceil a \rceil$ denotes the smallest integer that is no smaller than $a$.
	Let $I$ be a set. We use $\chi_I$ to denote the indicator function on $I$ such that $\chi_I(\xb) =1$ if $\xb \in I$ and $\chi_I(\xb) =0$ if $\xb \notin I$. The notation $\#I$ denotes the cardinality of $I$.

	Denote the domain $X =[0,1]^d$. For a function $f:X\rightarrow \RR$ and a multi-index $\balpha=[\alpha_1,\dots,\alpha_d]^\top$, $\partial^{\balpha} f$ denotes $\partial^{|\balpha|} f /\partial x_1^{\alpha_1}\cdots \partial x_d^{\alpha_d}$, where $|\balpha| =\sum_{k=1}^d\alpha_k$. We denote $\bx^{\balpha} = x_1^{\alpha_1}x_2^{\alpha_2} \cdots x_d^{\alpha_d}$. 
	Let $\rho$ be a measure on $X$. The $L^2$ norm of $f$ with respect to the measure $\rho$ is $\|f\|^2_{L^2(\rho)}=\int_{X} |f(\xb)|^2 d \rho$.  We say $f \in L^2(\rho)$ if $\|f\|^2_{L^2(\rho)} < \infty$. We  denote $\|f\|_{L^2(\rho(\Omega))}^2=\int_{\Omega} |f(\xb)|^2d\rho$ for any $\Omega\subset X$. 
	
	The notation of $f\lesssim g$ means that there exists  a constant $C$ independent of any variable upon which $f$ and $g$ depend, such that $f\le Cg$; similarly for $\gtrsim$. $f \asymp g$ means that $f\lesssim g$ and $f\gtrsim g$. 
	
	\subsection{Preliminaries}
	
	\begin{definition}[H\"{o}lder functions]
		A function $f:X \rightarrow \RR$ belongs to the H\"{o}lder space $\calH^{r}(X)$ with a H\"{o}lder index $r>0$, if 
		\begin{align}
			\|f\|_{\calH^{r}(X)} &= \max_{|\balpha|<\lceil r-1\rceil}\sup\limits_{\xb\in X} |\partial^{\balpha}f(\xb)| + \max\limits_{|\balpha|=\lceil r-1\rceil}
			\sup\limits_{ \xb\neq \zb \in X}\frac{|\partial^{\balpha}f(\xb)- \partial^{\balpha}f (\zb)|}{\|\xb-\zb\|_2^{r - \lceil r-1\rceil}} < \infty.
			\label{eq:holdernorm}
		\end{align}
		
	\end{definition}

	\begin{definition}[Minkowski dimension]
		\label{def.mikowski}
		Let $\Omega\subset [0,1]^d$. For any $\varepsilon>0$, $\cN(\varepsilon,\Omega,\|\cdot\|_\infty)$ denotes the fewest number of $\varepsilon$-balls that cover $\Omega$ in terms of $\|\cdot\|_\infty$.
		The (upper) Minkowski dimension of $\Omega$ is defined as
		\[d_{M}(\Omega):=\limsup_{\varepsilon\rightarrow 0^{+}} \frac{\log \cN(\varepsilon,\Omega,\|\cdot\|_\infty)}{
			\log(1/\varepsilon)}.\]
		{We further define the Minkowski dimension constant of $\Omega$ as
			$$
			c_{M}(\Omega)=\sup_{\varepsilon>0} \cN(\varepsilon,\Omega,\|\cdot\|_{\infty})\varepsilon^{d_M(\Omega)}.
			$$
			Such a constant is an upper bound on the rate of how $\cN(\varepsilon,\Omega,\|\cdot\|_{\infty})$ scales with $\varepsilon^{-d_M(\Omega)}$.}
	\end{definition}
	\
	
	\noindent\textbf{ReLU network.}
	In this paper, we consider the feedforward neural networks defined over in the form of 
	\begin{align}
		f_{\rm NN}(\xb)=  W_L\cdot\ReLU\big( W_{L-1}\cdots \ReLU(W_1\xb+\bbb_1) + \cdots +\bbb_{L-1}\big)+\bbb_L,
		\label{eq.FNN.f}
	\end{align}
	where $W_l$'s are weight matrices, $\bbb_l$'s are biases, and $\ReLU(a)=\max\{a,0\}$ denotes the rectified linear unit (ReLU). Define the network class as
	\begin{align}
		&\cF_{\rm NN}(L,w,K,\kappa,M)=\big\{f_{\rm NN}:\RR^{d}\rightarrow \RR| f_{\rm NN}(\xb) \mbox{ is in the form of (\ref{eq.FNN.f}) with } L \mbox{ layers,} \label{eq.FNN} \\
		&\mbox{width bounded by } w,  \|f\|_{L^\infty}\leq M, \ \|W_l\|_{\infty,\infty}\leq \kappa,  \|\bbb_l\|_{\infty}\leq \kappa,\  \mbox{ and } \textstyle \sum_{l=1}^L \|W_l\|_0+\|\bbb_l\|_0\leq K   \big\}, 
		\nonumber
	\end{align}
	$\|W\|_{\infty,\infty}=\max_{i,j} |W_{i,j}|,\ \|\bbb\|_{\infty}=\max_i |b_i|$ for any matrix $W$ and vector $\bbb$, and $\|\cdot\|_0$ denotes the number of nonzero elements of its argument.

	\section{Adaptive approximation}
	\label{sec.adaptive}
	This section is an introduction to tree-based nonlinear approximation and a function class whose regularity is defined through nonlinear approximation theory. We   re-visit  tree-based nonlinear approximations and define this function class in Subsection \ref{sectree}. Several examples of this function class are given in Subsection \ref{secexample}.
	
	\subsection{Tree-Based Nonlinear Approximations}
	\label{sectree}
	In  the classical tree-based nonlinear approximations \citep{binev2007universal,binev2005universal,cohen2001tree}, piecewise polynomials are used to approximate the target function on an adaptive partition. For simplicity, we focus on the case that the function domain is $X= [0,1]^d$. Let $\rho$ be a probability measure on $X$ and $f\in L^2(\rho)$.   The multiscale dyadic partitions of $X$ give rise to a tree structure. It is natural to consider nonlinear approximations based on this tree structure.
	
	Let $\calC_j = \{C_{j,k}\}_{k=1}^{2^{jd}}$ be the collection of dyadic subcubes of $X$ of sidelength $2^{-j}$. Here $j$ denotes the scale of $C_{j,k}$ with a small $j$ and $k$ denotes the location.
	A small $j$ represents the coarse scale, and a large $j$ represents the fine scale. These dyadic cubes are naturally associated with a tree $\calT$. Each node of this tree corresponds to a cube $\Cjk$. The dyadic partition of the 2D cube $[0,1]^2$ and its associated tree are illustrated in Figure \ref{figdyadictree}. Every node $\Cjk$ at scale $j$ has $2^d$ children at scale $j+1$. We denote the set of children of $\Cjk$ by $\calC(\Cjk)$. When the node $ \Cjk$ is a child of the node $C_{j-1,k'}$, we call $C_{j-1,k'}$ the parent of $ \Cjk$, denoted by $\calP(\Cjk)$. A proper subtree $\calT_0$ of $\calT$ is a collection of nodes such that: (1) the root node $X$ is in $\calT_0$; (2) if $\Cjk \neq X$ is in $\calT_0$, then its parent is also in $\calT_0$. Given a proper subtree $\calT_0$ of $\calT$, the outer leaves of $\calT_0$ contain all $\Cjk \in \calT$ such that $\Cjk \notin \calT_0$ but the parent of $\Cjk$ belongs to $\calT_0$: $\Cjk \notin \calT_0$ but $\calP(\Cjk) \in \calT_0$. The collection of the outer leaves of $\calT_0$, denoted by $\Lambda =\Lambda(\calT_0)$, forms a partition of $X$.

	\begin{figure}[t]
		\centering
		\subfigure[Dyadic partition]{\includegraphics[width=0.2\textwidth]{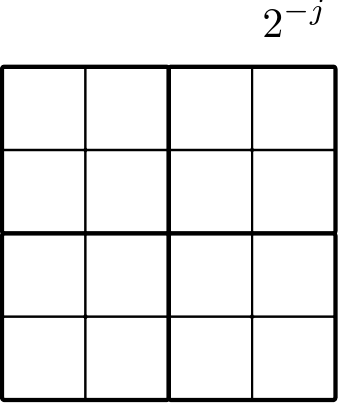}}
		\hspace{2cm}
		\subfigure[Tree]{\includegraphics[width=0.4\textwidth]{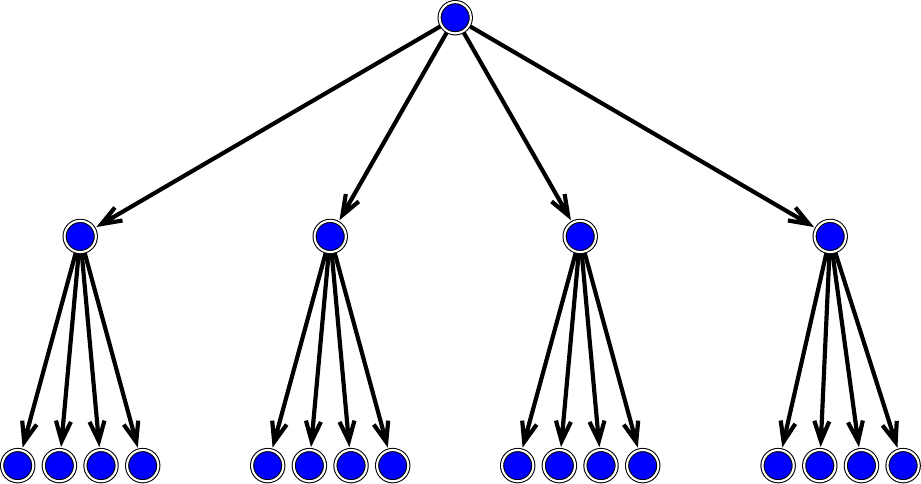}}
		\caption{The dyadic partition of the 2D unit cube $[0,1]^2$ and the associated tree.}
		\label{figdyadictree}
	\end{figure}
	
	The tree-based nonlinear approximation generates an adaptive partition with a thresholding technique. In certain cases, this thresholding technique boils down to wavelet thresholding. Specifically, one defines a refinement quantity on each node of the tree, and then thresholds the tree to the smallest proper subtree containing all the nodes whose refinement quantity is above certain value. Adaptive partitions are given by the outer leaves of this proper subtree after thresholding.
	
	We consider piecewise polynomial approximations of $f$ by polynomials of degree $\theta$, where $\theta$ is a nonnegative integer. 
	Let $\calP_{\theta}$ be the space of $d$-variable polynomials of degree no more than $\theta$. 
	For any cube $\Cjk$, the best polynomial approximating $f$ on $\Cjk$ is 
	\begin{equation}
		p_{j,k} = p_{j,k}(f) = \argmin_{p \in \calP_{\theta}} \|(f-p)\chi_{\Cjk}\|_{L^2(\rho)}.
		\label{eqpjk}
	\end{equation}
	At a fixed scale $j$, $f$ can be approximated by the piecewise polynomial
	$f_j = \sum_{k} p_{j,k}\chi_{\Cjk}$. Denote $V_j$ as the space of $\theta$-order piecewise polynomial functions on the partition $\cup_k \Cjk$. By definition, $V_j$ is a linear subspace and $V_j \subset V_{j+1}$. We have $f_j\in V_j$ and $f_j$ is the best approximation of $f$ in $V_j$.
	Let $V^{\perp}_j$ be the orthogonal complement of $V_j$ in $V_{j+1}$, and then  
	$$V_{j+1} = V_j \oplus V^{\perp}_j \ \text{ and  } \ V^{\perp}_j \perp V^{\perp}_{j'} \text{ if } j \neq j'.$$

	When the node $\Cjk$ is refined to its children $\calC(\Cjk)$, the difference of the approximations between these two scales on $\Cjk$ is defined as
	\begin{equation}
		\psi_\jk = \psi_\jk(f) = \sum_{C_{j+1,k'} \in \calC(\Cjk)} p_{j+1,k'}(f)\chi_{C_{j+1,k'}} - p_\jk(f) \chi_{\Cjk} .
		\label{eqrefinementquantity1}
	\end{equation}
	For $C_{0,1} = X$, we let $\psi_{0,1} = p_{0,1}$. Note that $\sum_k \psi_\jk \in V^{\perp}_j$ and therefore $\sum_k \psi_\jk $ and $\sum_{k'} \psi_{j',k'} $ are orthogonal if $j \neq j'$.

	The refinement quantity on the node $\Cjk$ is defined as the norm of $\psi_\jk$:
	\begin{equation}
		\delta_\jk = \delta_\jk(f) = \|\psi_\jk\|_{L^2(\rho)} .
		\label{eqrefinementquantity2}
	\end{equation}
	In the case piecewise constant approximations, i.e. $\theta=0$, $\psi_\jk(f)$ corresponds to the Haar wavelet coefficient of $f$, and $\delta_\jk$ is the magnitude of the Haar wavelet coefficient.
	
	The target function $f$ can be decomposed as 
	$$f =\sum_{j \ge 0, k} \psi_\jk(f).$$
	Due to the orthogonality of the $\psi_\jk$'s, we have
	$$\|f\|^2_{L^2(\rho)} = \sum_{j \ge 0, k} [\delta_\jk(f)]^2.$$

	\begin{figure}[t]
		\centering
		\subfigure[Red nodes satisfy $\delta_\jk >\eta$]{\includegraphics[width=0.4\textwidth]{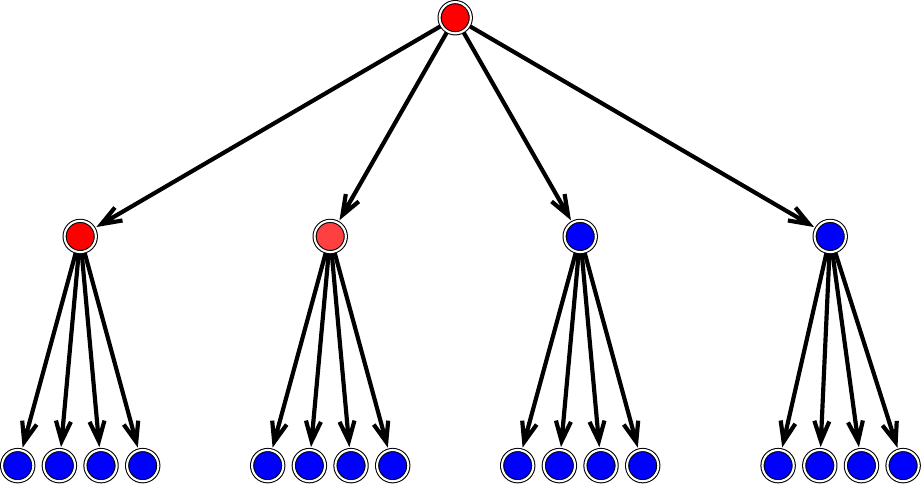}}
		\hspace{1cm}
		\subfigure[The truncated tree]{\includegraphics[width=0.4\textwidth]{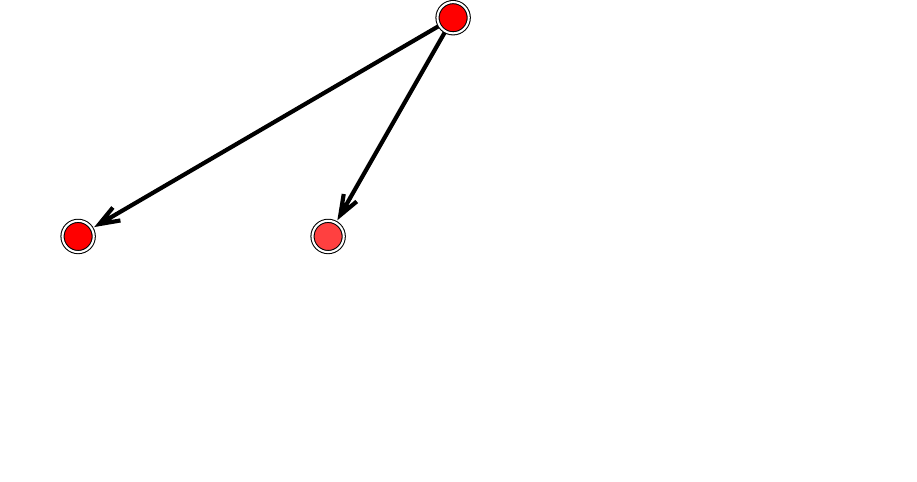}}
		\\
		\subfigure[Outer leaves given by the green nodes]{\includegraphics[width=0.41\textwidth]{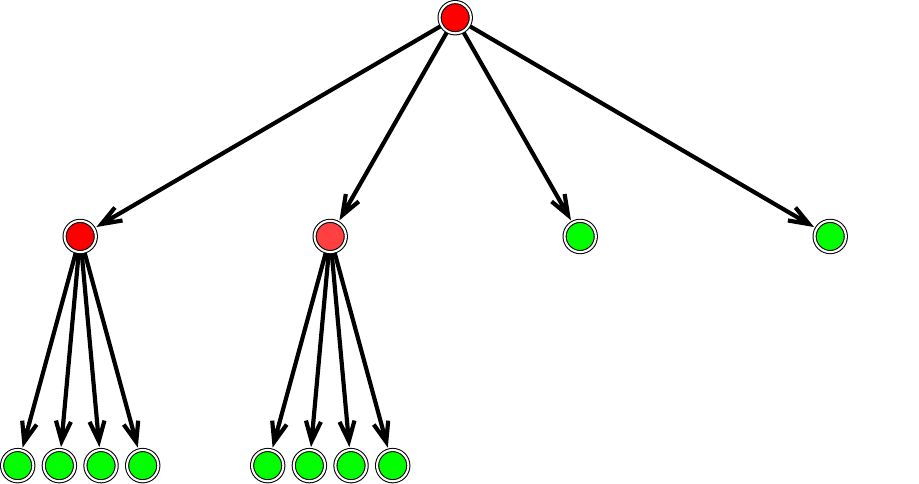}}
		\hspace{2.5cm}
		\subfigure[Dyadic partition]{\includegraphics[width=0.3\textwidth]{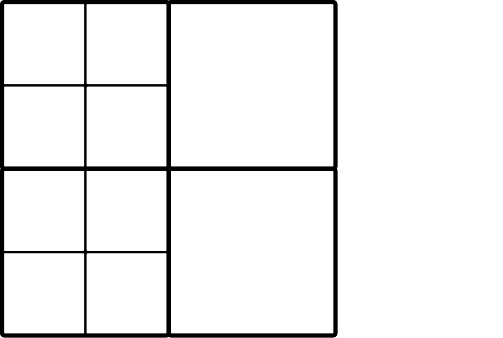}}
		\caption{(a) For a fixed $\eta>0$, the red nodes have the refinement quantity above $\eta$: $\delta_\jk(f) >\eta$. The master tree is then truncated to the smallest subtree containing the red nodes in (b). In (c), the outer leaves of the truncated tree are given by the green nodes. The corresponding adaptive partition is given in (d).}
		\label{figtruncatetree}
	\end{figure}
	
	In the tree-based nonlinear approximation, one fixes a threshold value $\eta>0$, and truncate $\calT$ to $\calT(f,\eta)$ -- the smallest subtree that contains all $\Cjk \in \calD$ with $\delta_\jk(f) >\eta$. The collection of outer leaves of $\calT(f,\eta)$,  denoted by $\Lambda(f,\eta)$, gives rise to an adaptive partition. This truncation procedure is illustrated in Figure \ref{figtruncatetree}. In Figure \ref{figtruncatetree}, the red nodes have the refinement quantity above $\eta$, and then  the master tree $\calT$ is truncated to the smallest subtree containing the red nodes in (b). The outer leaves of this truncated tree are given by the green nodes in Figure \ref{figtruncatetree} (c), and the corresponding adaptive partition is given in Figure \ref{figtruncatetree} (d).

	The piecewise polynomial approximation of $f$ on this adaptive partition is
	\begin{align}
		p_{\Lambda(f,\eta)} = \sum_{\Cjk \in \Lambda(f,\eta)}p_\jk(f)\chi_{\Cjk}.
		\label{eq.plambda}
	\end{align}
	In the adaptive approximation, the regularity of $f$ can be defined by the size of the tree $\#\calT(f,\eta)$. 
	\begin{definition}[(2.19) in \cite{binev2007universal}]
		\label{defasl}
		For a fixed $s>0$, a polynomial degree $\theta$, we let the function class $\calA^s_{\theta}$ be the collection of all $f \in L^2(X)$, such that
		\begin{equation}
			|f|_{\asj}^m =\sup_{\eta>0} \eta^m \#\calT(f,\eta) <\infty, \quad \text{with }\ m=\frac{2}{2s+1},
			\label{eqasldef}
		\end{equation}
		where $\calT(f,\eta)$ is the truncated tree of approximating $f$ with piecewise $\theta$-th order polynomials with threshold $\eta$.
	\end{definition}
	
	In Definition \ref{defasl}, the complexity of the adaptive approximation is measured by the cardinality of the truncated tree  $\calT(f,\eta)$. In fact, the cardinality of the adaptive partition $\Lambda(f,\eta)$ is related with the cardinality of the truncated tree  $\calT(f,\eta)$ such that
	\begin{equation}
		\#\calT(f,\eta ) \le   \#\Lambda(f,\eta) \le 
		2^d\#\calT(f,\eta ) .
		\label{eq:cardinality}
	\end{equation}
	The lower bound follows from 
	\begin{align*}
		\#\calT(f,\eta ) &\le  \sum_{k=1}^{\infty} \frac{\#\Lambda(f,\eta)}{(2^d)^k} = \frac{\#\Lambda(f,\eta)}{2^d(1-\frac{1}{2^d})} \le \#\Lambda(f,\eta).
	\end{align*}

	The definition of $\asj$ does not explicitly depend on the dimension $d$. The dimension $d$ is actually hidden in the regularity parameter $s$ (see Example \ref{exampleholderinasl}). This way of definition has the advantage of adapting to low-dimensional structures in the data distribution (see Example \ref{examplelowd}).

	When $f \in \asj$, we have the approximation error
	\begin{equation}
		\|f-p_{\Lambda(f,\eta)}\|_{L^2(\rho)}^2 \le C_s |f|_{\asj}^{m}\eta^{2-m} \le C_s |f|_{\asj}^{2} (\#\calT(f,\eta))^{-{2s}},
		\label{eqaslapproxerror}
	\end{equation}
	where 
	\begin{equation}
		C_s =2^m \sum_{\ell \ge 0} 2^{\ell(m-2)}, \quad \text{with }\ m=\frac{2}{2s+1}.
		\label{eq:cm}
	\end{equation}
	The approximation error  in \eqref{eqaslapproxerror} is proved in Appendix \ref{appendixproofeqaslapproxerror}. The original proof can be found in \cite{binev2007universal,binev2005universal,cohen2001tree}.

	\subsection{Case Study of the $\asj$ Function Class}
	\label{secexample}
	
	The  $\asj$  class contains a large collection of functions, including  H\"older functions, piecewise H\"older functions, functions which are irregular on a set of measure zero, and regular functions with distribution concentrated on a low-dimensional manifold. For some examples to be studied below, we make the following assumption on the measure $\rho$:
	\begin{assumption} \label{assum.rho}
		There exists a constant $C_{\rho}>0$ such that  any subset $S\subset X$ satisfies $$\rho(S)\leq C_{\rho}|S|,$$ where $|S|$ is the Lebesgue measure of $S$.
	\end{assumption}

	\subsubsection{H\"older functions}
	\begin{examplea}[H\"older functions]
		\label{exampleholderinasl} 
		Let $r>0$. Under Assumption \ref{assum.rho}, the $r$-H\"older function class  $\calH^r(X)$ belongs to $\calA^{r/d}_{\lceil r-1\rceil}$. 
		If $f \in \calH^{r}(X)$, then we have $f \in \calA^{r/d}_{\lceil r-1\rceil}$. Furthermore, if $\|f\|_{\calH^{r}(X)} \le 1$, then 
		\begin{equation}
			|f|_{\calA^{r/d}_{\lceil r-1\rceil}} \le C(r,d,C_\rho)
			\label{eq:1a}
		\end{equation}
		for some constant $C(r,d,C_\rho)$ depending on $r,d$ and $C_\rho$ in Assumption \ref{assum.rho}. 
	\end{examplea}
	Example \ref{exampleholderinasl}  is proved in Section \ref{sec.proof.exampleholderinasl}. At the end of Example \ref{exampleholderinasl},  we assume $\|f\|_{\calH^{r}(X)} \le 1$ without loss of generality. The same statement holds if $\|f\|_{\calH^{r}(X)}$ is bounded by an absolute constant. Such a constant only changes the $|\cdot|_{\calA^{r/d}_{\lceil r-1\rceil}}$ bound in \eqref{eq:1a}, i.e. $C(r,d,C_\rho)$ will  depends on this constant.

	The neural network approximation theory for H\"older functions is given in Example \ref{ex.holder.approx} and the generalization theory is given in Example \ref{ex.holder.gene}.

	\begin{figure}[t!]
		\centering
		\subfigure[1D piecewise H\"older function]{\includegraphics[width=0.42\textwidth]{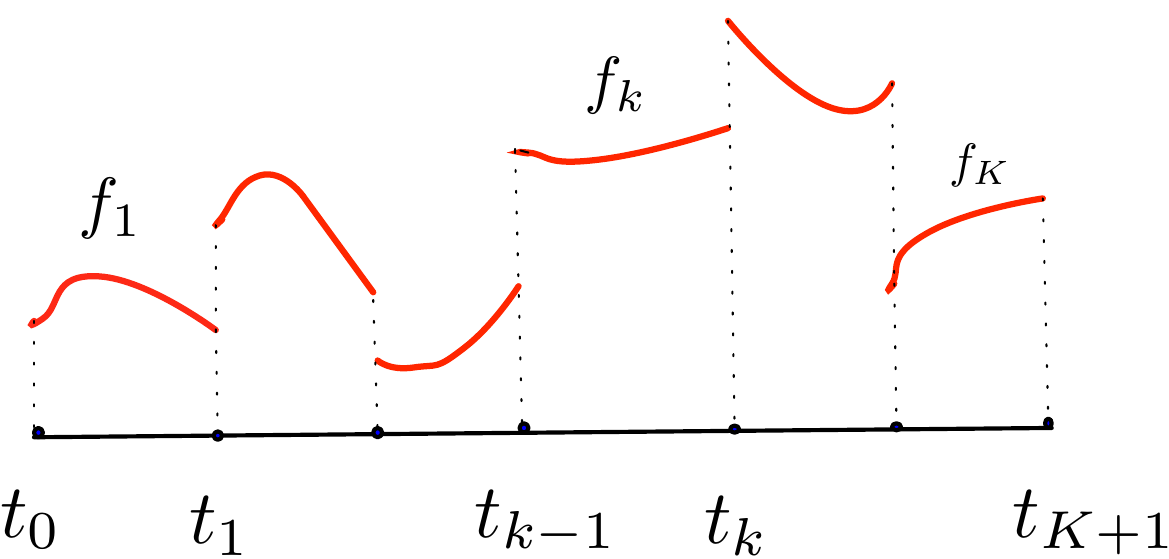}}
		\hspace{1cm}
		\subfigure[2D piecewise domain]{\includegraphics[width=0.21\textwidth]{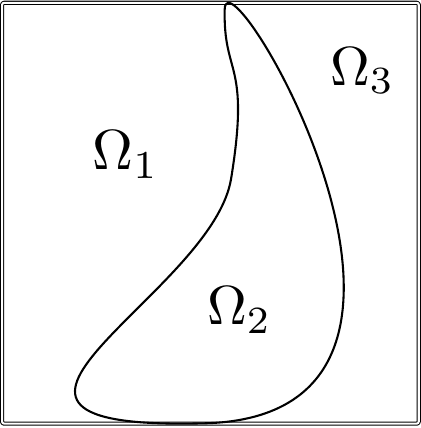}}
		\caption{(a) Example \ref{examplepieceholderinasl1d}: 1D piecewise H\"older function with $K$ discontinuity points; (b) Example \ref{examplepieceholderinasl}: A 2D piecewise domain. The functions in Example \ref{examplepieceholderinasl} are $r$-H\"older in the interior of $\Omega_1,\Omega_2,\Omega_3$.}
		\label{figexample}
	\end{figure}
	\subsubsection{Piecewise H\"older functions in 1D}
	\begin{examplea}[Piecewise H\"older functions in 1D]
		\label{examplepieceholderinasl1d}
		Let $d=1$, $r>0$ and $K$ be a positive integer. 
		Under Assumption \ref{assum.rho}, all bounded piecewise  $r$-H\"older functions with $K$ discontinuity points belong to
		$f\in \calA^{r}_{\lceil r-1\rceil}$. Specifically, let $f$ be a piecewise $r$-H\"older function such that $f = \sum_{k=1}^{K+1} f_k\chi_{[t_{k-1},t_{k})}$, where  $0= t_0 <t_1<\ldots<t_K<t_{K+1} = 1$. Each function  $f_k: [t_{k-1},t_{k}) \rightarrow \RR $ is $r$-H\"older in $(t_{k-1},t_{k})$, and $f$ is discontinuous at $t_1,t_2,\ldots t_K$. Assume $f$ is bounded such that $\|f\|_{L^{\infty}([0,1])} \le 1$. In this case, we have $f\in \calA^{r}_{\lceil r-1\rceil}$. See Figure \ref{figexample} (a) for an illustration of a piecewise H\"older function in 1D. Furthermore, if $\max_k\|f\|_{\calH^{r}(t_k,t_{k+1})} \le 1,$ then $$|f|_{\calA^{r}_{\lceil r-1\rceil}} \le C(r,d,C_\rho,K)$$ for some constant $C(r,d,C_\rho,K)$ depending on $r,d,K$ and $C_\rho$ in Assumption \ref{assum.rho}, and does not depend on specific $t_k$'s.
	\end{examplea}

	Example \ref{examplepieceholderinasl1d} is proved in Section \ref{sec.proof.examplepieceholderinasl1d}. 
	Example \ref{examplepieceholderinasl1d} demonstrates that, for 1D bounded piecewise $r$-H\"older functions with a finite number of discontinuities, the overall regularity index is $s =r$ under Definition \ref{defasl}. In comparison with Example \ref{exampleholderinasl}, we prove that, a finite number of discontinuities in 1D does not affect the regularity index in Definition \ref{defasl}.  
	
	The neural network approximation theory for piecewise H\"older functions in 1D is given in Example \ref{ex.pieceiwiseholder1d.approx} and the generalization theory is given in Example \ref{ex.pieceiwiseholder1d.gene}.
	
	\subsubsection{Piecewise H\"older functions in multi-dimensions}
	In the next example, we will see that, for piecewise $r$-H\"older functions in multi-dimensions, the overall approximation error is dominated either by the approximation error in the interior of each piece or by the error along the discontinuity. The overall regularity index $ s$ depends on $r$ and $d$. 
	
	\begin{examplea}[Piecewise H\"older functions in multi-dimensions]
		\label{examplepieceholderinasl}
		Let $d\ge 2$, $r>0$ and $\{\Omega_t\}_{t=1}^T$ be subsets of $[0,1]^d$ such that $\cup_{t=1}^T \Omega_t = [0,1]^d$ and the $\Omega_t$'s only overlap at their boundaries.  Each $\Omega_t$ is a connected subset of $[0,1]^d$ and the union of their boundaries  $\cup_t\partial\Omega_t $ has upper Minkowski dimension $d-1$.  See Figure \ref{figexample} (b) for an illustration of the $\Omega_t$'s. 
		When $\rho$ satisfies Assumption \ref{assum.rho}, all piecewise $r$-H\"older functions with discontinuity on $\cup_t\partial \Omega_t$
		belong to 
		\begin{equation} 
			\calA^{s}_{\lceil r-1\rceil},\ \text{where}\  s=\min\left\{ \frac{r}{d}, \frac{1}{2(d-1)}\right\}.
			\label{eqtildes}
		\end{equation} 
		Specifically, let $f$ be a piecewise $r$-H\"older function such that $f = \sum_{t=1}^T f_t\chi_{\Omega_t}$ where $\chi_{\Omega_t}$ is the indicator function on $\Omega_t$.
		Each function  $f_t: \Omega_t \rightarrow \RR $ is $r$-H\"older in the interior of $\Omega_t$: $f_t \in \calH^{r}(\Omega_t^o)$ where $\Omega_t^o$ denotes the interior of $\Omega_t$, and $f$ is discontinuous at $\cup_t\partial \Omega_t$. 
		Assume $f$ is bounded such that $\|f\|_{L^{\infty}([0,1]^d)} \le 1$.
		In this case, $f\in \calA^{s}_{\lceil r-1\rceil}$ with the $ s$ given in \eqref{eqtildes}. Furthermore, if 
		$\max_t\|f\|_{\calH^{r}(\Omega_t^o)} \le 1$, then 
		$$|f|_{\calA^{s}_{\lceil r-1\rceil}} \le C(r,d,c_M(\cup_t\partial\Omega_t),C_{\rho})$$
		for some $C(r,d,c_M(\cup_t\partial\Omega_t),C_{\rho})$ depending on $r,d,c_M(\cup_t\partial\Omega_t)$  (the Minkowski dimension constant of $\cup_t \partial\Omega_t$ defined in Definition \ref{def.mikowski}) and $C_\rho$ in Assumption \ref{assum.rho}. 
	\end{examplea} 
	
	Example \ref{examplepieceholderinasl} is proved in Section \ref{sec.proof.examplepieceholderinasl}. 
	Example \ref{examplepieceholderinasl} demonstrates that, for piecewise $r$-H\"older functions with discontinuity on a subset with upper Minkowski dimension $d-1$, the overall regularity index $ s$ has a phase transition. 
	When $\frac r d \le \frac{1}{2(d-1)}$, the approximation error is dominated by that in the interior of the $\Omega_t$'s. When $\frac r d > \frac{1}{2(d-1)}$, the approximation error is dominated by that around the boundary of the $\Omega_t$'s. As the result, the overall regularity index $s$ is the minimum of $\frac r d$ and $\frac{1}{2(d-1)}$.

	The neural network approximation theory for piecewise H\"older functions in multi-dimensions is given in Example \ref{ex.pieceiwiseholdernd.approx} and the generalization theory is given in Example \ref{ex.pieceiwiseholdernd.gene}.

	\subsubsection{Functions  irregular on a set of measure zero}
	The definition of $\asj$ is dependent on the measure $\rho$, since the refinement quantity $\delta_\jk$ is the $L^2$ norm with respect to $\rho$.
	This measure-dependent definition is not only adaptive to the regularity of $f$, but also adaptive to the distribution $\rho$.
	In the following example, we show that, the definition of $\asj$ allows the function $f$ to be irregular on a set of measure zero. For $\delta>0$, $\Omega_\delta$ denotes the set within $\delta$ distance to $\Omega$ such that 
	$\Omega_\delta = \{x \in X:\ {\rm dist}(x,\Omega) = \inf_{z\in \Omega} \|x-z\|\le \delta\}$.
	
	\begin{examplea}[Functions  irregular on a set of measure zero]
		\label{examplemeasure0}
		Let $\delta>0$, $\Omega$ be a subset of $X=[0,1]^d$ and $\Omega^\complement = X\setminus \Omega$.  If $f $ is an $r$-H\"older function on $\Omega_\delta$ and $\rho(\Omega^\complement) = 0$, then  $f\in \calA^{r/d}_{\lceil r-1\rceil}$. Furthermore, if $\|f\|_{\calH^{r}(\Omega_\delta)} \le 1$, then 
		$$|f|_{\calA^{r/d}_{\lceil r-1\rceil}} \le C(r,d,C_\rho)$$ for some constant $C(r,d,C_\rho)$ depending on $r,d$ and $C_\rho$ in Assumption \ref{assum.rho}.

	\end{examplea}
	
	Example \ref{examplemeasure0} is proved in Section \ref{sec.proof.examplemeasure0}. 
	In Example \ref{examplemeasure0}, $f$ is $r$-H\"older on $\Omega_\delta$, and can be irregular on a measure zero set. In comparison with Example \ref{exampleholderinasl}, such  irregularity on a set of measure zero does not affect the smoothness parameter in Definition \ref{defasl}. In this example, we set $\Omega_\delta$ to be a larger set than $\Omega$, in order to avoid the discontinuity effect at the boundary of $\Omega$.
	
	The neural network approximation theory for functions in Example \ref{examplemeasure0} is given in Example \ref{ex.holderirregular.approx} and the generalization theory is given in Example \ref{ex.holderirregular.gene}.

	\subsubsection{H\"{o}lder functions  with distribution concentrated on a low-dimensional manifold}
	Since the definition of $\asj$ is dependent on the probability measure $\rho$, this definition is also adaptive to lower-dimensional sets in $X$. We next consider  a probability measure $\rho$  concentrated on a $d_{\rm in}$-dimensional manifold isometrically embedded in $X$.

	\begin{examplea}[H\"{o}lder functions with distribution concentrated on a low-dimensional manifold] 
		\label{examplelowd}
		Let $r>0$.
		Suppose $X=[0,1]^d$ can be decomposed to $\Omega$ and $\Omega^\complement$, i.e. $X = \Omega \cup \Omega^\complement$ where $\Omega$ is a compact $d_{\rm in}$-dimensional Riemannian manifold isometrically embedded in $X$. Assume that $\rho(\Omega^\complement) =0$, and $\rho$ conditioned on $\Omega$ is the uniform distribution on $\Omega$. If $f \in \calH^r(X)$, then $f\in \calA^{r/{d_{\rm in}}}_{\lceil r-1\rceil}$. Furthermore, if $\|f\|_{\cH^r(X)}\leq 1$, then $$|f|_{\cA^{r/d_{\rm in}}_{\lceil r-1\rceil}} <C(r,d,d_{\rm in},\tau,|\Omega|)$$ with $C(r,d,d_{\rm in},\tau,|\Omega|)$ depending on $r,d,d_{\rm in},\tau$ and $|\Omega|$, where $\tau$ is the reach \citep{federer1959curvature} of $\Omega$ and $|\Omega|$ is the surface area of $\Omega$.
		
	\end{examplea}
	
	Example \ref{examplelowd} is proved in Section \ref{sec.proof.examplelowd}. In this example, the function $f$ is $r$-H\"older on $X$, but the measure $\rho$ is supported on a lower dimensional manifold with intrinsic dimension $d_{\rm in}$. The regularity index under Definition \ref{defasl} is $r/{d_{\rm in}}$ instead of $r/d$.

	\section{Adaptive approximation and generalization theory of deep neural networks}
	\label{sec.main}
	
	This section contains our main results: approximation and generalization theories of deep ReLU networks for the $\asj$ function class. We  present some preliminaries in Subsection \ref{subsec:proofprelim}, the approximation theory in Subsection \ref{subsecapprox}, case studies of the approximation error in Subsection \ref{subseccasea}, the generalization theory in Subsection \ref{subsecgeneralization}, and case studies of the generalization error in Subsection \ref{subseccaseg}.

	\subsection{Preliminaries}
	\label{subsec:proofprelim}
	Each $C_{j,k}$ is a hypercube in the form of $\otimes_{\ell=1}^d[r_{\ell,j,k},r_{\ell,j,k}+2^{-j}]$, where $r_{\ell,j,k}\in [0,1]$ is a scalar and  
	\begin{align}
		\otimes_{\ell=1}^d[r_{\ell,j,k},r_{\ell,j,k}+2^{-j}]=[r_{1,j,k},r_{1,j,k}+2^{-j}]\times \cdots \times [r_{d,j,k},r_{d,j,k}+2^{-j}]
		\label{eq.rb}
	\end{align}
	is a hypercube with edge length $2^{-j}$ in $\RR^d$.
	
	The collection of polynomials $ \left((\xb-\rb_{j,k})/{2^{-j}}\right)^{\balpha}$ form a basis for the space of $d$-variable polynomials of degree no more than $\theta$. Let $\rho$ be a measure on $[0,1]^d$ and  $f\in L^2(\rho)$. The piecewise polynomial approximator $p_{j,k}$ for $f$ can be written as  \begin{align}
		p_{j,k}(\xb)=\sum_{|\balpha|\leq \theta} a_{\balpha}\left(\frac{\xb-\rb_{j,k}}{2^{-j}}\right)^{\balpha},
		\label{eq.poly.bound.form}
	\end{align}
	where $\rb_{j,k}=[r_{1,j,k},...,r_{d,j,k}]$.

	In this paper, we focus on the set of functions with bounded coefficients in the piecewise polynomial approximation.
	
	\begin{assumption}\label{assum.f}
		Let $s,R_{\cA},R, R_p>0$, $\theta$ be a nonnegative integer, and $\rho$ be a probability measure on $[0,1]^d$. We assume $f\in \asj$ and 
		\begin{itemize}
			\item[(i)]  $|f|_{\asj}\leq R_{\cA}$,
			\item[(ii)] $\|f\|_{L^{\infty}([0,1]^d)}\leq R$,
			\item[(iii)] On every $\Cjk$, the polynomial approximator $p_{j,k}$ for $f$ in the form of (\ref{eq.poly.bound.form}) satisfies $|a_{\balpha}|\leq R_p$ for all $\balpha$ with $ |\balpha|\leq \theta$. 
		\end{itemize} 
	\end{assumption}
	By Assumption \ref{assum.f} (i) and (ii), $f$ has a bounded $|\cdot|_{\asj}$ quantity and $L^{\infty}$ norm, which is a common assumption in nonparametric estimation theory \citep{gyorfi2002distribution}. Assumption \ref{assum.f} (iii) requires the polynomial coefficients in the best polynomial approximating $f$ 
	on every $\Cjk$ to be uniformly bounded by $R_p$.  
	The following lemma shows that Assumption \ref{assum.f} (iii) can be implied from Assumption \ref{assum.f} (ii) when $\rho$ is the Lebesgue measure on $X=[0,1]^d$.

	\begin{lemma}\label{lem.poly.uniform}
		Let $R>0$, $\theta$ be a fixed nonnegative integer, and $\rho$ be the Lebesgue measure on $X=[0,1]^d$.  There exists a constant $R_p>0$ depending on $\theta,d$ and $R$ such that, for any  function $f$ on $[0,1]^d$ satisfying $\|f\|_{L^{\infty}(X)}\leq R$, the $p_{j,k}$ in (\ref{eqpjk}) has the form of (\ref{eq.poly.bound.form}) with $|a_{\balpha}|\leq R_p, \ \forall \balpha $ with $ |\balpha|\leq \theta$ and $\|p_{j,k}\|_{L^{\infty}(C_{j,k})}\leq CR_p$ for some $C$ depending on $d$ and $\theta$.
	\end{lemma}

	Lemma \ref{lem.poly.uniform} is proved in Appendix  \ref{proof.lem.poly.uniform}.
	Lemma \ref{lem.poly.uniform} implies that under the Lebesgue measure, for any $p_{j,k}$ in the form of (\ref{eq.poly.bound.form}), the coefficients $a_{\balpha}$'s are uniformly bounded by a constant  depending on $\theta,d,R$, and is independent to the index $(j,k)$. Thus Assumption \ref{assum.f}(iii) holds.

	\subsection{Approximation Theory}
	\label{subsecapprox}

	Our approximation theory shows that deep neural networks give rise to universal approximations for functions in the $\asj$ class under Assumption \ref{assum.f} if the network architecture is properly chosen.

	\begin{theorem}[Approximation] \label{thm.approx}
		Let $s,d,C_{\rho},R_{\cA},R,R_p>0$ and $\theta$ be a nonnegative integer.
		For any $\varepsilon>0$, there is a ReLU network class $\cF=\cF_{\rm NN}(L,w,K,\kappa,M)$ with parameters
		\begin{align}
			&L=O\left(\log \frac{1}{\varepsilon}\right), \ w=O(\varepsilon^{-\frac{1}{s}}), \ K=O\left(\varepsilon^{-\frac{1}{s}}\log \frac{1}{\varepsilon}\right),\ 
			\kappa=O(\varepsilon^{-\max\{2,\frac{1}{s}\}}), \  M=R,
			\label{thm1.parameter}
		\end{align}
		such that,
		for any $\rho$ satisfying Assumption \ref{assum.rho} and any $f\in \calA^s_\theta$ satisfying Assumption \ref{assum.f}, if the weight parameters of the network are properly chosen,
		the network yields a function  $\widetilde{f} \in \cF$  such that 
		\begin{align}
			\|\widetilde{f}-f\|_{L^{2}(\rho)}^2\leq \varepsilon.
			\label{eq.approx}
		\end{align}
		The constants hidden in $O(\cdot)$ depends on $d$ (the  dimension of the domain for $f$), 
		$C_{\rho}$ (in Assumption \ref{assum.rho}), $\theta$ (the polynomial order), $s$, $R$, $R_p$  and $R_{\cA}$ (in Assumption \ref{assum.f}). 
	\end{theorem}
	
	Theorem \ref{thm.approx} is proved in Section \ref{proof.thm.approx}. Theorem \ref{thm.approx} demonstrates the universal approximation power of deep neural networks for the $\asj$ function class. The parameters in \eqref{thm1.parameter} specifies the network architecture. To approximate a specific function $f\in\asj$, there exist some proper weight parameters which give rise to a network function $\widetilde f$ to approximate $f$.

	\subsection{Case Studies of the Approximation Error}
	\label{subseccasea}

	In this subsection, we apply Theorem \ref{thm.approx} to the examples in Subsection \ref{secexample}, and derive the approximation theory for each example. In the following case studies, we need Assumptions \ref{assum.rho}, \ref{assum.f} (iii)  and \ref{assum.general}, but not Assumption \ref{assum.f} (i) - (ii). In each case, we have shown in Subsection \ref{secexample} that Assumption \ref{assum.f} (i) holds: $f\in \asj$ with a proper $\theta$ and the regularity index $s$ depends on each specific case.
	
	\subsubsection{H\"older functions}
	Consider H\"{o}lder functions  in Example \ref{exampleholderinasl}  such that $\cH^r\subset \calA^{r/d}_{\lceil r-1\rceil}$. Applying Theorem \ref{thm.approx} gives rise to the following neural network approximation theory for H\"{o}lder functions.
	\begin{exampleb}[H\"older functions]
		\label{ex.holder.approx}
		Let $r,d,C_{\rho},R_p>0$. 
		For any $\varepsilon>0$, there is a ReLU network class $\cF=\cF_{\rm NN}(L,w,K,\kappa,M)$ with parameters
		\begin{align*}
			&L=O\left(\log \frac{1}{\varepsilon}\right), \ w=O(\varepsilon^{-\frac{d}{r}}), \ K=O\left(\varepsilon^{-\frac{d}{r}}\log \frac{1}{\varepsilon}\right),\ 
			\kappa=O(\varepsilon^{-\max\{2,\frac{d}{r}\}}), \  M=1,
		\end{align*}
		such that for any $\rho$ satisfying Assumption \ref{assum.rho} and  any function $f\in \cH^r(X)$ satisfying $\|f\|_{\cH^r(X)}\leq 1$ and Assumption \ref{assum.f} (iii), if the weight parameters of the network are properly chosen,
		the network yields a function  $\widetilde{f} \in \cF$  such that 
		\begin{align*}
			\|\widetilde{f}-f\|_{L^{2}(\rho)}^2\leq \varepsilon.
		\end{align*}
		The constant hidden in $O(\cdot)$ depends on $r,d,C_\rho, R_p $.
	\end{exampleb}
	Example \ref{ex.holder.approx} is a corollary of Theorem \ref{thm.approx} with  $s=r/d$. Note that Assumption \ref{assum.f}(i) and (ii) are implied by the condition $\|f\|_{\cH^r(X)}\leq 1$. In particular, $\|f\|_{\cH^r(X)}\leq 1$ implies $\|f\|_{L^{\infty}(X)}\leq 1$ according to \eqref{eq:holdernorm}. Furthermore,  If $\|f\|_{\cH^r(X)}\leq 1$, then $|f|_ {\calA^{r/d}_{\lceil r-1\rceil}} \leq C(r,d,C_{\rho})$ according to our argument in Example \ref{exampleholderinasl}.

	In litertuare, the approximation theory of ReLU networks for Sobolev  functions in $W^{r,\infty}$ has been established in \citet[Theorem 1]{yarotsky2017error}. The proof in \citet[Theorem 1]{yarotsky2017error} can be applied to H\"older functions.
	Our network size in Example \ref{ex.holder.approx} is comparable to that in \citet[Theorem 1]{yarotsky2017error}.

	\subsubsection{Piecewise H\"older functions in 1D}
	Considering 1D piecewise H\"older functions in Example \ref{examplepieceholderinasl1d}, we have the following approximation theory:
	\begin{exampleb}[Piecewise H\"older functions in 1D]
		\label{ex.pieceiwiseholder1d.approx}
		Let $r,C_{\rho},R_p,K>0$.
		For any $\varepsilon>0$, there is a ReLU network class $\cF=\cF_{\rm NN}(L,w,K,\kappa,M)$ with parameters
		\begin{align}
			&L=O\left(\log \frac{1}{\varepsilon}\right), \ w=O(\varepsilon^{-\frac{1}{r}}), \ K=O\left(\varepsilon^{-\frac{1}{r}}\log \frac{1}{\varepsilon}\right),\ 
			\kappa=O(\varepsilon^{-\max\{2,\frac{1}{r}\}}), \  M=1,
		\end{align}
		such that for any $\rho$ satisfing Assumption \ref{assum.rho}, and any 
		piecewise $r$-H\"older function in the form of $f = \sum_{k=1}^{K+1} f_k\chi_{[t_{k-1},t_{k})}$ in Example \ref{examplepieceholderinasl1d} satisfying $\|f\|_{L^{\infty}([0,1])}\leq 1$, $\max_k\|f_k\|_{\cH^r(t_k,t_{k+1})}\leq 1$ 
		and Assumption \ref{assum.f}(iii), if the weight parameters of the network are properly chosen,
		the network yields a function  $\widetilde{f} \in \cF$  such that 
		\begin{align*}
			\|\widetilde{f}-f\|_{L^{2}(\rho)}^2\leq \varepsilon.
		\end{align*}
		The constant hidden in $O(\cdot)$ depends on $r,C_{\rho},R_p,K$. 
	\end{exampleb}
	
	Example \ref{ex.pieceiwiseholder1d.approx} is a corollary of Theorem \ref{thm.approx} with  $f\in \calA^{r}_{\lceil r-1\rceil}$. Assumption \ref{assum.f} (i) is not explicitly enforced in Example \ref{ex.pieceiwiseholder1d.approx} , but is implied from the condition of $\max_k\|f_k\|_{\cH^r(t_k,t_{k+1})}\leq 1$ by Example \ref{examplepieceholderinasl1d}.
	Example \ref{ex.pieceiwiseholder1d.approx} shows that, to achieve an $\varepsilon$ approximation error for 1D functions with a finite number of discontinuities, the network size is comparable to that for 1D H\"older functions in Example \ref{ex.holder.approx}.
	
	\subsubsection{Piecewise H\"{o}lder function in multi-dimensions}
	\label{sec.pieceholdernd.approx}
	Considering  piecewise H\"older functions in multi-dimensions  in Example \ref{examplepieceholderinasl}, we have the following approximation theory:
	\begin{exampleb}[Piecewise H\"older functions in multi-dimensions]
		\label{ex.pieceiwiseholdernd.approx}
		Let $d\ge 2$, $r,C_{\rho},R_p,T>0$ and $\{\Omega_t\}_{t=1}^T$ be subsets of $[0,1]^d$ such that $\cup_{t=1}^T \Omega_t = [0,1]^d$ and the $\Omega_t$'s only overlap at their boundaries.  Each $\Omega_t$ is a connected subset of $[0,1]^d$ and the union of their boundaries  $\cup_{t=1}^T \partial\Omega_t $ has upper Minkowski dimension $d-1$. Denote the Minkowski dimension constant of $\cup_{t=1}^T \partial \Omega_t$ by $c_M(\cup_t\partial\Omega_t)$.
		For any $\varepsilon>0$, there is a ReLU network class $\cF=\cF_{\rm NN}(L,w,K,\kappa,M)$ with parameters
		\begin{align*}
			&L=O\left(\log \frac{1}{\varepsilon}\right), \ w=O(\varepsilon^{-\max\left\{\frac{d}{r},2(d-1)\right\}}), \ K=O\left(\varepsilon^{-\max\left\{\frac{d}{r},2(d-1)\right\}}\log \frac{1}{\varepsilon}\right),\\ 
			&\kappa=O(\varepsilon^{-\max\{2,\frac{d}{r},2(d-1)\}}), \  M=1,
		\end{align*}
		such that for any $\rho$ satisfying Assumption \ref{assum.rho}, and any piecewise $r$-H\"{o}lder function in the form of $f = \sum_{t=1}^T f_t\chi_{\Omega_t}$ in Example \ref{examplepieceholderinasl}  satisfying $\|f\|_{L^{\infty}([0,1]^d)}\leq 1$, $\max_t \|f_t\|_{\cH^r(\Omega_t^o)}\leq 1$ and Assumption \ref{assum.f}(iii), if the weight parameters of the network are properly chosen,
		the network yields a function  $\widetilde{f} \in \cF$  such that 
		\begin{align*}
			\|\widetilde{f}-f\|_{L^{2}(\rho)}^2\leq \varepsilon.
		\end{align*}
		The constant hidden in $O(\cdot)$ depends on $r,d, c_M(\cup_t\partial\Omega_t),C_\rho, R_p$.
	\end{exampleb}
	
	Example \ref{ex.pieceiwiseholdernd.approx} is a corollary of Theorem \ref{thm.approx} with  $f\in \calA^{s}_{\lceil r-1\rceil}$ for the $s$ given in \eqref{eqtildes}. 
	Assumption \ref{assum.f} (i) is not explicitly enforced in Example \ref{ex.pieceiwiseholdernd.approx} , but is implied from the condition of $\max_k\|f_k\|_{\cH^r(t_k,t_{k+1})}\leq 1$ by Example \ref{examplepieceholderinasl}.
	Example \ref{ex.pieceiwiseholdernd.approx} implies that to approximate a piecewise H\"{o}lder function  in Example \ref{examplepieceholderinasl}, the number of nonzero weight parameters is in the order of $\varepsilon^{-\max\left\{\frac{d}{r},2(d-1)\right\}}\log \frac{1}{\varepsilon}$. The discontinuity set in Example \ref{ex.pieceiwiseholdernd.approx} has upper Minkowski dimension $d-1$, which is a weak assumption without additional regularity assumption or low-dimensional structures.

	Neural network approximation theory for piecewise smooth functions has been considered in \cite{petersen2018optimal}. The setting in \cite{petersen2018optimal} is similar but different from that of Example \ref{ex.pieceiwiseholdernd.approx}. In \cite{petersen2018optimal}, the authors considered piecewise functions $f: [-1/2,1/2]^d \rightarrow \RR$, where the different \lq\lq smooth regions\rq\rq \ of $f$ are separated by $\cH^\beta$ hypersurfaces.  If $f$ is a piecewise $r$-H\"{o}lder function and the H\"{o}lder norm of $f$ on each piece is bounded, it is shown in \citet[Corollary 3.7]{petersen2018optimal} that such $f$ can be universally approximated by a ReLU network with at most $c\varepsilon^{-p(d-1)/\beta}$ weight parameters to guarantee an  $L^p$ approximation error $\varepsilon$. 
	It is further shown in \citet[Theorem 4.2]{petersen2018optimal} that, to achieve an $L^p$ approximation error $\varepsilon$, the optimal required number of  weight parameters is lower bounded in the order of $\varepsilon^{-{p(d-1)}/{\beta}}/\log \frac{1}{\varepsilon}$.
	Our result in Example \ref{ex.pieceiwiseholdernd.approx} is comparable to that of \cite{{petersen2018optimal}} when $\beta = 1$ and $p=2$. When the discontinuity hypersurface has higher order regularity, i.e. $\beta >1$, the network size in \cite{{petersen2018optimal}} is smaller/better than that in Example \ref{ex.pieceiwiseholdernd.approx} since the higher order smoothness of the discontinuity hypersurface is exploited in the approximation theory. However, Example \ref{ex.pieceiwiseholdernd.approx} imposes a weaker assumption on the discontinuity hypersurface. Example \ref{ex.pieceiwiseholdernd.approx} requires the discontinuity hypersurface to have upper Minkowski dimension $d-1$, while the  the discontinuity hypersurface in \citet[Definition 3.3]{petersen2018optimal} is the  graph of a $\cH^\beta$ function on $d-1$ coordinates.
	In this sense, Example \ref{ex.pieceiwiseholdernd.approx} can be applied to a wider class of piecewise H\"older functions.

	\subsubsection{Functions  irregular on a set of measure zero}
	
	Functions in the $\asj$ class can be irregular on a set of measure zero, as in Example \ref{examplemeasure0}. The approximation theory for functions  in Example \ref{examplemeasure0} is given below:
	\begin{exampleb}[Functions  irregular on a set of measure zero]
		\label{ex.holderirregular.approx}
		Let $r,d,C_{\rho}, R_p,\delta>0$ and $\Omega$ be a subset of $X$.  For any $\varepsilon>0$, there is a ReLU network class $\cF=\cF_{\rm NN}(L,w,K,\kappa,M)$ with parameters
		\begin{align*}
			&L=O\left(\log \frac{1}{\varepsilon}\right), \ w=O(\varepsilon^{-\frac{d}{r}}), \ K=O\left(\varepsilon^{-\frac{d}{r}}\log \frac{1}{\varepsilon}\right),\ 
			\kappa=O(\varepsilon^{-\max\{2,\frac{d}{r}\}}), \  M=1.
		\end{align*} 
		For any $\rho$ satisfying Assumption \ref{assum.rho} and $\rho(\Omega^{\complement})=0$, and for any $f: X \rightarrow \RR$ satisfying $\|f\|_{\cH^r(\Omega_{\delta})}\leq 1$ and Assumption \ref{assum.f}(iii), if the weight parameters of the network are properly chosen,
		the network class yields a function  $\widetilde{f} \in \cF$  such that 
		\begin{align*}
			\|\widetilde{f}-f\|_{L^{2}(\rho)}^2\leq \varepsilon.
		\end{align*}
		The constant hidden in $O(\cdot)$ depends on $r,d,C_\rho, R_p$.
	\end{exampleb}
	
	Example \ref{ex.holderirregular.approx} is a corollary of Theorem \ref{thm.approx} with  $f\in \calA^{r/d}_{\lceil r-1\rceil}$ and $|f|_{\calA^{r/d}_{\lceil r-1\rceil}} \le C(r,d,C_\rho)$ by Example \ref{examplemeasure0}. 
	Example \ref{ex.holderirregular.approx} shows that function irregularity on a set of measure zero does not affect the network size in  approximation theory.

	\subsubsection{H\"{o}lder functions  with distribution concentrated on a low-dimensional manifold}
	Theorem \ref{thm.approx} cannot be directly applied to  Example \ref{examplelowd}, since Assumption \ref{assum.rho} is violated when $\rho$ is  supported on a low-dimensional manifold. In literature, neural network approximation theory  has been established in \cite{chen2019efficient} for functions on a low-dimensional manifold, and in  \cite{nakada2020adaptive} for H\"older functions in $[0,1]^D$ while the support of measure has a low Minkowski dimension. 
	See Section \ref{sec.low-d.gene} for a detailed discussion about the generalization error.
	
	\subsection{Generalization Error}
	\label{subsecgeneralization}
	
	Theorem \ref{thm.approx} proves the existence of a neural network $\widetilde{f}$ to approximate $f$ with an arbitrary accuracy $\varepsilon$, but it does not give an explicit method to find the weight parameters of  $\widetilde{f}$. In practice, the weight parameters are learned from data through the empirical risk minimization. The generalization error of the empirical risk minimizer is analyzed in this subsection.

	Suppose the training data set is $\cS=\{\xb_i,y_i\}_{i=1}^n$ where the $\xb_i$'s are i.i.d. samples from $\rho$, and the $y_i$'s have the form
	\begin{align}
		y_i=f(\xb_i)+\xi_i,
		\label{eqregmodel}
	\end{align}
	where the $\xi_i$'s are i.i.d. noise, independently of the $\xb_i$'s. In practice, one estimates the function $f$ by minimizing the empirical mean squared risk
	\begin{align}
		\widehat{f}=\argmin_{f_{\rm NN}\in \cF} \frac{1}{n}\sum_{i=1}^n |f_{\rm NN}(\xb_i)-y_i|^2
		\label{eq.empirial.loss}
	\end{align}
	for some network class $\cF$. The squared generalization error of $\widehat{f}$ is
	$$ \EE_{\cS}\EE_{\xb\sim \rho} |\widehat{f}(\xb)-f(\xb)|^2,$$
	where the expectation $\EE_{\cS}$ is taken over the joint the distribution of training data $\cS$.

	To establish an upper bound of the squared generalization error, we make the following assumption of noise.
	\begin{assumption}\label{assum.general}
		Suppose  the noise $\xi$ is a sub-Gaussian random variable  with mean $0$ and variance proxy $\sigma^2$.
	\end{assumption}
	
	Our second main theorem in this paper gives a generalization error bound of $\widehat{f}$.

	\begin{theorem}[Generalization error] \label{thm.general}
		Let $\sigma, s,d,C_{\rho},R_p,R_{\cA},R>0$ and $\theta$ be a nonnegative integer. Suppose $\rho$ satisfies Assumption \ref{assum.rho}, $f$ satisfies Assumption \ref{assum.f}, and Assumption \ref{assum.general} holds for noise.
		The training data $\cS$ are sampled according to \eqref{eqregmodel}.
		If the network class $\cF=\cF_{\rm NN}(L,w,K,\kappa,M)$ is set with parameters 
		\begin{align*}
			L=O(\log n), \ w=O(n^{\frac{1}{1+2s}}), \ K=O\left(n^{\frac{1}{1+2s}}\log n\right),\ \kappa=O\left(n^{\max\left\{\frac{2s}{1+2s}, \frac{1}{1+2s}\right\}}\right),\ M=R,
		\end{align*}
		then the minimizer $\widehat{f}$ of (\ref{eq.empirial.loss})
		satisfies
		\begin{align*}
			\EE_{\cS}\EE_{\xb\sim \rho} |\widehat{f}(\xb)-f(\xb)|^2\leq Cn^{-\frac{2s}{2s+1}}\log^3n.
		\end{align*}
		The constant $C$ and the constant hidden in $O(\cdot)$ depend on $\sigma$ (the variance proxy of noise in Assumption \ref{assum.general}), $d$ (the  dimension of the domain for $f$), 
		$C_{\rho}$ (in Assumption \ref{assum.rho}), $\theta$ (the polynomial order), $s$, $R$, $R_p$  and $R_{\cA}$ (in Assumption \ref{assum.f}).
	\end{theorem}
	Theorem \ref{thm.general} is proved in Section \ref{proof.thm.general}. Theorem \ref{thm.general} gives rise to a generalization error guarantee of deep neural networks for functions in the $\asj$ class. We will provide case studies in the following subsection.
	
	\subsection{Case Study of the Generalization Error}
	\label{subseccaseg}
	
	In this subsection, we apply Theorem \ref{thm.general} to the examples in Subsection \ref{secexample}, and derive the squared generalization error for each example. In the following case studies, we assume  Assumptions \ref{assum.rho}, \ref{assum.f} (iii)  and \ref{assum.general}, but not Assumption \ref{assum.f} (i) and (ii). 
	In each case, we have shown in Subsection \ref{secexample} that Assumption \ref{assum.f} (i) holds: $f\in \asj$ with a proper $\theta$ and the regularity index $s$ depends on each specific case.
	
	\subsubsection{H\"{o}lder functions} 
	Consider H\"{o}lder functions  in Example \ref{exampleholderinasl}  such that $\cH^r\subset \calA^{r/d}_{\lceil r-1\rceil}$. Applying Theorem \ref{thm.general} gives rise to the following generalization error bound for H\"{o}lder functions.

	\begin{examplec}[H\"older functions]
		\label{ex.holder.gene}
		Let $\sigma\geq0, r,d,C_{\rho},R_p>0$. Suppose $\rho$ satisfies Assumption \ref{assum.rho}, $f\in \cH^r(X)$ satisfies $\|f\|_{\cH^r(X)}\leq1$ and Assumption \ref{assum.f}(iii), and Assumptions \ref{assum.general} hold. The training data $\cS$ are sampled according to \eqref{eqregmodel}. If we set the network class $\cF_{\rm NN}(L,p,K,\kappa,M)$ as 
		\begin{align*}
			L=O(\log n), \ p=O(n^{\frac{d}{2r+d}}), \ K=O\left(n^{\frac{d}{2r+d}}\log n\right),\ \kappa=O\left(n^{\max\left\{\frac{2r}{2r+d}, \frac{d}{2r+d}\right\}}\right),\ M=1.
		\end{align*}
		Then the empirical minimizer $\widehat{f}$ of (\ref{eq.empirial.loss})
		satisfies
		\begin{align*}
			\EE_{\cS}\EE_{\xb\sim \rho} |\widehat{f}(\xb)-f(\xb)|^2\leq Cn^{-\frac{2r}{2r+d}}\log^3n.
		\end{align*}
		The constant $C$ and the constant hidden in $O(\cdot)$ depend on $\sigma,r,d,C_\rho, R_p$.
	\end{examplec}
	Example \ref{ex.holder.gene} is a corollary of Theorem \ref{thm.general} with  $s=r/d$. Our upper bound matches the rate in \citet[Theorem 1]{schmidt2017nonparametric} and is optimal up to a logarithmic factor in comparison with the minimax error given in \citet[Theorem 3.2]{gyorfi2002distribution}.

	\subsubsection{Piecewise H\"older functions in 1D}
	Considering 1D piecewise H\"older functions in Example \ref{examplepieceholderinasl1d}, we have the following generalization error bound:
	\begin{examplec}[Piecewise H\"older functions in 1D]
		\label{ex.pieceiwiseholder1d.gene}
		Let $\sigma\geq0, r,d,C_{\rho},R_p>0$. Suppose $\rho$ satisfies Assumption \ref{assum.rho}. 
		Let  $f$ be an 1D piecewise $r$-H\"{o}lder function in the form of $f = \sum_{k=1}^{K+1} f_k\chi_{[t_{k-1},t_{k})}$ in Example \ref{examplepieceholderinasl1d} satisfying  $\|f\|_{L^{\infty}([0,1])}\leq 1$, $\max_k\|f_k\|_{\cH^r(t_k,t_{k+1})}\leq 1$ and Assumption \ref{assum.f}(iii). Suppose Assumption \ref{assum.general} holds and the training data $\cS$ are sampled according to \eqref{eqregmodel}. Set the network class $\cF_{\rm NN}(L,p,K,\kappa,M)$ as 
		\begin{align}
			L=O(\log n), \ p=O(n^{\frac{1}{1+2r}}), \ K=O(n^{\frac{1}{1+2r}}\log n),\ \kappa=O(n^{\max\{\frac{2r}{1+2r}, \frac{1}{1+2r}\}}),\ M=1.
		\end{align}
		Then the empirical minimizer $\widehat{f}$ of (\ref{eq.empirial.loss}) satisfies
		\begin{align}
			\EE_{\cS}\EE_{\xb\sim \rho} |\widehat{f}(\xb)-f(\xb)|^2\leq Cn^{-\frac{2r}{2r+1}}\log^3n.
		\end{align}
		The constant $C$ and the constant hidden in $O(\cdot)$ depend on $\sigma,r, C_\rho, R_p,K$.
	\end{examplec}
	
	Example \ref{ex.pieceiwiseholder1d.gene} shows that a finite number of discontinuities in 1D does not affect the rate of convergence of the generalization error.
	
	\subsubsection{Piecewise H\"older functions in multi-dimensions}
	\label{sec.pieceholdernd.gene}
	Considering  piecewise H\"older functions in multi-dimensions  in Example \ref{examplepieceholderinasl}, we have the following generalization error bound:
	\begin{examplec}[Piecewise H\"older functions in multi-dimensions]
		\label{ex.pieceiwiseholdernd.gene}
		Let $\sigma\geq0, r,d,C_{\rho},R_p>0$. Let $\{\Omega_t\}_{t=1}^T$ be subsets of $[0,1]^d$ such that $\cup_{t=1}^T \Omega_t = [0,1]^d$ and the $\Omega_t$'s only overlap at their boundaries.  Each $\Omega_t$ is a connected subset of $[0,1]^d$ and the union of their boundaries  $\cup_t\partial\Omega_t $ has upper Minkowski dimension $d-1$. Denote the Minkowski dimension constant of $\cup_{t=1}^T \Omega_t$ by $c_M(\cup_t\partial\Omega_t)$. Suppose $\rho$ satisfies Assumption \ref{assum.rho}. Let $f$ be a piecewise $r$-H\"{o}lder function in the form of $f = \sum_{t=1}^T f_t\chi_{\Omega_t}$ in Example \ref{examplepieceholderinasl} satisfying $\|f\|_{L^{\infty}([0,1]^d)}\leq 1$, $\max_t \|f_t\|_{\cH^r(\Omega_t^o)}\leq 1$ and Assumption \ref{assum.f}(iii). Suppose Assumption \ref{assum.general} holds and the training data $\cS$ are sampled according to \eqref{eqregmodel}.
		Set the network class $\cF_{\rm NN}(L,p,K,\kappa,M)$ as 
		\begin{align}
			&L=O(\log n), \ p=O(n^{\max\{\frac{d}{2r+d},\frac{d-1}{d}\}}), \ K=O(n^{\max\{\frac{d}{2r+d},\frac{d-1}{d}\}}\log n),\nonumber\\
			&\kappa=O(n^{\max\left\{\min\{\frac{2r}{2r+d},\frac{1}{d}\}, \max\{\frac{d}{2r+d},\frac{d-1}{d}\}\right\}}),\ M=1.
		\end{align}
		Then the empirical minimizer $\widehat{f}$ of (\ref{eq.empirial.loss}) satisfies
		\begin{align}
			\EE_{\cS}\EE_{\xb\sim \rho} |\widehat{f}(\xb)-f(\xb)|^2\leq Cn^{-\min\{\frac{2r}{2r+d},\frac{1}{d}\}}\log^3n.
		\end{align}
		The constant $C$ and the constant hidden in $O(\cdot)$ depend on $\sigma,r,d,c_M(\cup_t\partial\Omega_t),C_\rho, R_p$.
	\end{examplec}
	Example \ref{ex.pieceiwiseholdernd.gene} shows that the convergence rate of the generalization error has a phase transition. When $\frac{r}{d}\leq \frac{1}{2(d-1)}$, the generalization error is dominated by that in the interior of the $\Omega_t$'s, so that the squared generalization error converges in the order of $n^{-\frac{2r}{2r+d}}$. When $\frac{r}{d}> \frac{1}{2(d-1)}$, the generalization error is dominated by that  around the boundary of the $\Omega_t$'s, so that the squared generalization error converges in the order of $n^{-\frac{1}{d}}$. As a result, the overall rate of convergence is $n^{-\min\{\frac{2r}{2r+d},\frac{1}{d}\}}$ up to log factors.

	Under the setting of \citet{petersen2018optimal} (discussed in  Section \ref{sec.pieceholdernd.approx} where  different \lq\lq smooth regions\rq\rq \ of $f$ are separated by $\cH^\beta$ hypersurfaces), the generalization error for estimating piecewise $r$-H\"{o}lder function by ReLU network is proved in the order of $\max(n^{-\frac{2r}{2r+d}},n^{-\frac{\beta}{\beta+d-1}})$ in \citet{imaizumi2019deep}.
	When $\beta=1$, our rate matches the result in \citet{imaizumi2019deep}. When $\beta>1$, the smoothness of boundaries is utilized in \citet{imaizumi2019deep}, leading to a better result than ours. 
	Nevertheless, the setting considered in this paper assumes a weaker assumption on the discontinuous boundaries, which are not required to be hypersurfaces.

	\subsubsection{Functions  irregular on a set of measure zero}
	
	Functions in the $\asj$ class can be irregular on a set of measure zero, as in Example \ref{examplemeasure0}. The generalization error for functions  in Example \ref{examplemeasure0} is given below:
	\begin{examplec}[Functions  irregular on a set of measure zero]
		\label{ex.holderirregular.gene}
		Let $\sigma\geq0, r,d,C_{\rho},R_p>0$.
		Let $\Omega$ be a subset of $X$ and $\delta>0$. Suppose $\rho$ satisfies Assumption \ref{assum.rho}, $f$ satisfies $\|f\|_{\cH^r(\Omega_{\delta})}\leq 1$ and Assumption \ref{assum.f}(iii), and Assumption \ref{assum.general} holds.
		Set the network class $\cF_{\rm NN}(L,w,K,\kappa,M)$ as 
		\begin{align*}
			L=O(\log n), \ p=O(n^{\frac{d}{2r+d}}), \ K=O\left(n^{\frac{d}{2r+d}}\log n\right),\ \kappa=O\left(n^{\max\left\{\frac{2r}{2r+d}, \frac{d}{2r+d}\right\}}\right),\ M=1.
		\end{align*}
		Then the empirical minimizer $\widehat{f}$ of (\ref{eq.empirial.loss})
		satisfies
		\begin{align*}
			\EE_{\cS}\EE_{\xb\sim \rho} |\widehat{f}(\xb)-f(\xb)|^2\leq Cn^{-\frac{2r}{2r+d}}\log^3n.
		\end{align*}
		The constant $C$ and the constant hidden in $O(\cdot)$ depend on $\sigma,r,d,C_\rho, R_p$.
	\end{examplec}
	
	Example \ref{ex.holderirregular.gene} shows that function irregularity on a set of measure zero does not affect the rate of convergence of the generalization error. Deep neural networks are adaptive to data distributions as well.
	
	\subsubsection{H\"{o}lder functions  with distribution concentrated on a low-dimensional manifold}
	\label{sec.low-d.gene}
	When the measure $\rho$ is concentrated on a low-dimensional manifold as considered in Example \ref{examplelowd}, Theorem \ref{thm.general} cannot be directly applied since Assumption \ref{assum.rho} is not satisfied. Instead, a more dedicated network structure can be designed to develop approximation and generalization error analysis. The setting of Example \ref{examplelowd} has been studied in \cite{nakada2020adaptive}, which assumes that the $\xb_i$'s are sampled from a measure supported on a set with an upper Minkowski dimension $d_{\rm in}$. If the target function $f$ is an $r-$H\"{o}lder function  on $[0,1]^d$ and if the network structure is properly set, the generalization error is in the order of $n^{-\frac{2r}{2r+d_{\rm in}}}$ up to a logarithmic factor \citep{nakada2020adaptive}. The result in \cite{nakada2020adaptive} can be applied to the setting of Example \ref{examplelowd}, since a $d_{\rm in}$ dimensional Riemannian manifold has the Minkowski dimension $d_{\rm in}$. 
	
	Another related setting is that $f$ is a $r$-H\"{o}lder function on a $d_{\rm in}$ dimensional Riemannian manifold embedded in $[0,1]^d$. This setting has been studied in \cite{chen2019efficient} for approximation theory and in \cite{chen2019nonparametric} for generalization theory by ReLU networks. In this setting, the squared generalization error converges in the order of $n^{-\frac{2r}{2r+d_{\rm in}}}$ up to a logarithmic factor.
	\section{Numerical experiments}
	\label{sec.numerical}
	In this section, we perform numerical experiments on 1D piecewise smooth functions, which fit in Example \ref{examplepieceholderinasl1d}.
	The following functions are included in our experiments: 
	\begin{itemize}
		
		\item Function with 1 discontinuity point, 
		\begin{equation*}
			f(x) =  
			\begin{cases}
				\sin \left(2\pi x\right)+1 & \text{ when } 0\leq x<\frac{1}{2}, \\ 
				\sin \left(2\pi x\right)-1 & \text{ when } \frac{1}{2}\leq x<1.
			\end{cases}
		\end{equation*}
		
		\item Function with 3 discontinuity points,
		\begin{equation*}
			f(x) = 
			\begin{cases}
				\sin \left(2\pi x \right) -1 & \text{ when } 0\leq x<\frac{1}{4}, \\ 
				\sin \left(2\pi x \right)  +1 & \text{ when } \frac{1}{4}\leq x<\frac{1}{2}, \\ 
				\sin \left(2\pi x \right)  -1 & \text{ when } \frac{1}{2}\leq x<\frac{3}{4}, \\ 
				\sin \left(2\pi x \right)  +1 & \text{ when } \frac{3}{4}\leq x<1.
			\end{cases}
		\end{equation*}
		
		\item Function with 5 discontinuity points,
		\begin{equation*}
			f(x) = 
			\begin{cases}
				\sin \left(2\pi x\right)-1 & \text{ when } 0\leq x<\frac{1}{6}, \\ 
				\sin \left(2\pi x\right) & \text{ when } \frac{1}{6}\leq x<\frac{1}{3}, \\ 
				\sin \left(2\pi x\right)+1 & \text{ when } \frac{1}{3}\leq x<\frac{1}{2}, \\ 
				\sin \left(2\pi x\right)-1 & \text{ when } \frac{1}{2}\leq x<1,\\
				\sin \left(2\pi x\right) & \text{ when } \frac{2}{3}\leq x<\frac{2}{3}, \\  
				\sin \left(2\pi x\right)+1 & \text{ when } \frac{5}{6}\leq x<1.
			\end{cases}
		\end{equation*}
		\item Function with 7 discontinuity points,
		\begin{equation*}
			f(x) =
			\begin{cases}
				\sin \left(2\pi x\right)-1 & \text{ when } 0\leq x<\frac{1}{8}, \\ 
				\sin \left(2\pi x\right)-\frac{1}{3} & \text{ when } \frac{1}{8}\leq x<\frac{1}{4}, \\ 
				\sin \left(2\pi x\right)+\frac{1}{3} & \text{ when } \frac{1}{4}\leq x<\frac{3}{8}, \\ 
				\sin \left(2\pi x\right)+1 & \text{ when } \frac{3}{8}\leq x<\frac{1}{2},\\
				\sin \left(2\pi x\right)-1 & \text{ when } \frac{1}{2}\leq x<\frac{5}{8} , \\ 
				\sin \left(2\pi x\right)-\frac{1}{3} & \text{ when } \frac{5}{8}\leq x<\frac{3}{4}, \\ 
				\sin \left(2\pi x\right)+\frac{1}{3} & \text{ when } \frac{3}{4}\leq x<\frac{7}{8}, \\ 
				\sin \left(2\pi x\right)+1 & \text{ when } \frac{7}{8}\leq x<1.
			\end{cases}
		\end{equation*}
	\end{itemize}

	\begin{figure}[t!]
		\centering
		\subfigure{\includegraphics[width=0.26\textwidth]{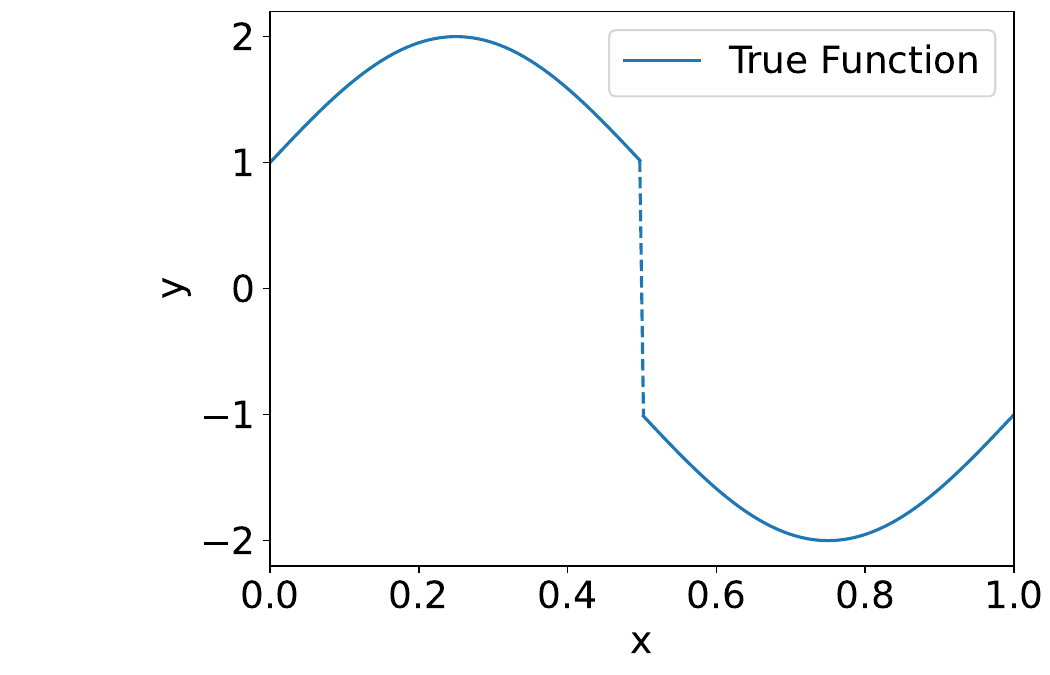}}
		\hspace{-0.1cm}
		\subfigure{\includegraphics[width=0.26\textwidth]{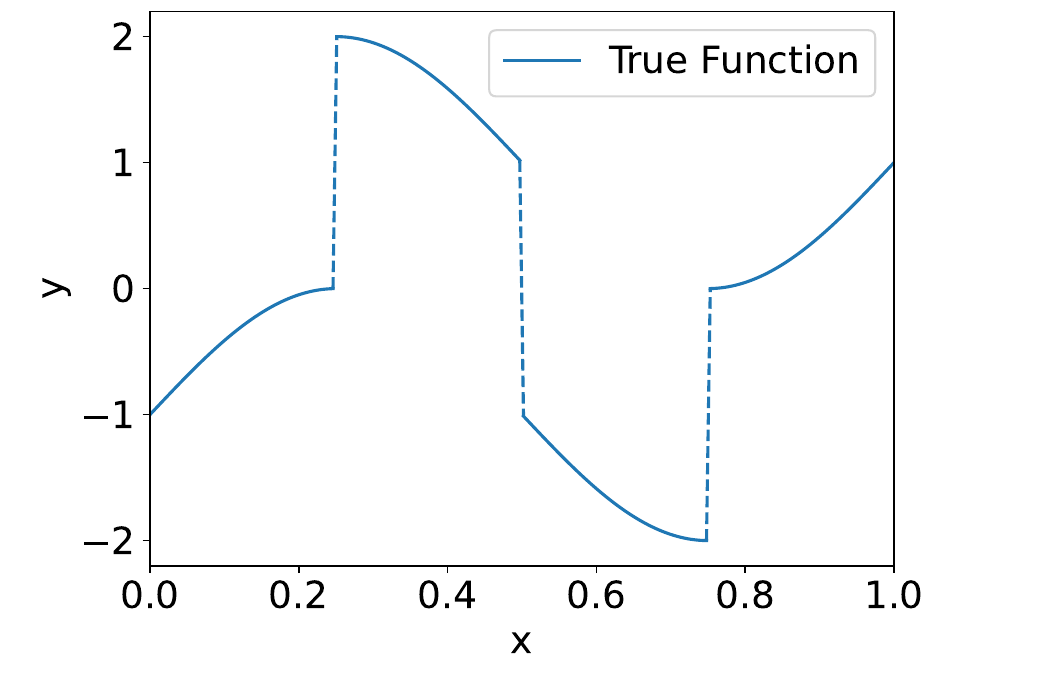}}
		\hspace{-1.2cm}
		\subfigure{\includegraphics[width=0.26\textwidth]{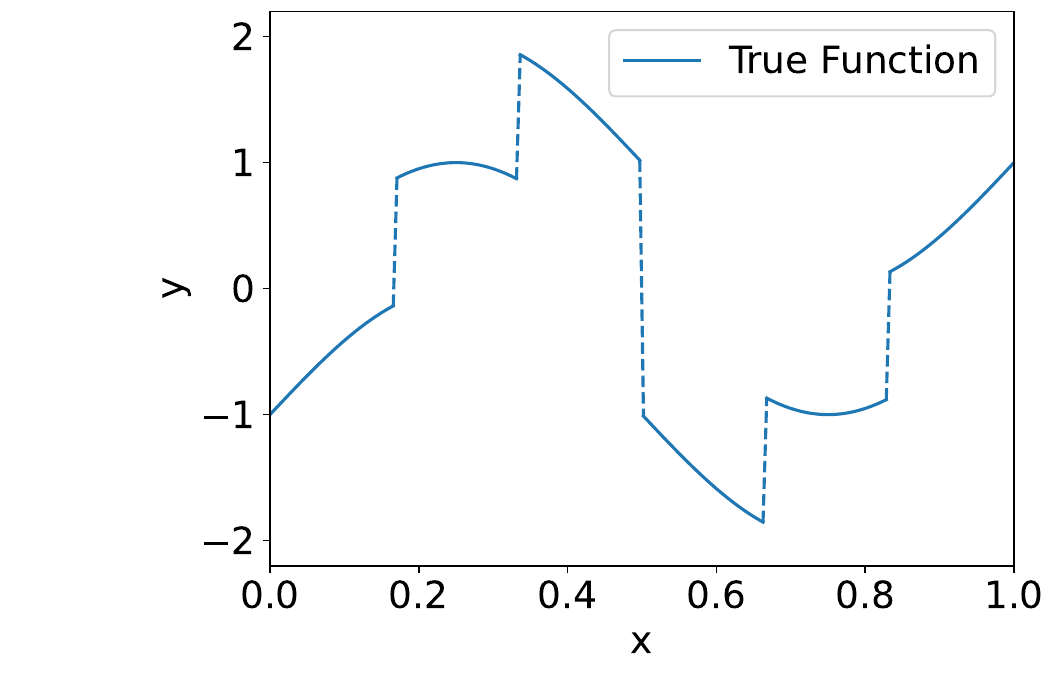}}
		\hspace{-0.1cm}
		\subfigure{\includegraphics[width=0.26\textwidth]{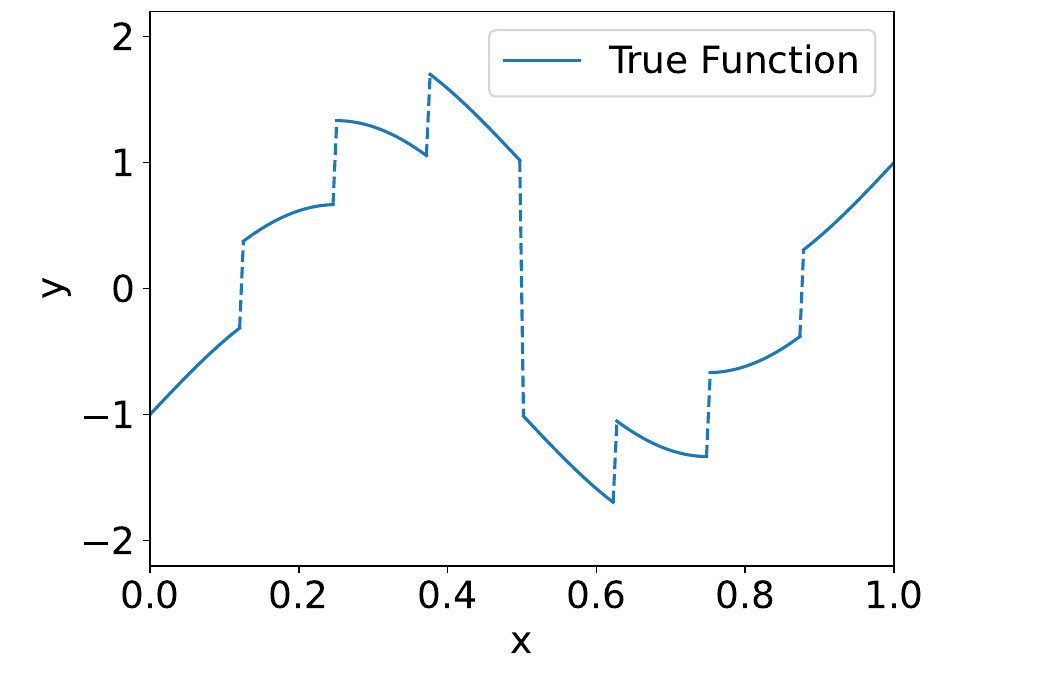}}
		\caption{Functions with different numbers of discontinuity points.}
		\label{fig:TargetFunction}
	\end{figure}

	\begin{figure}[t!]
		\centering    \subfigure[$\sigma = 0, n=16$]{\includegraphics[width=0.3\textwidth]{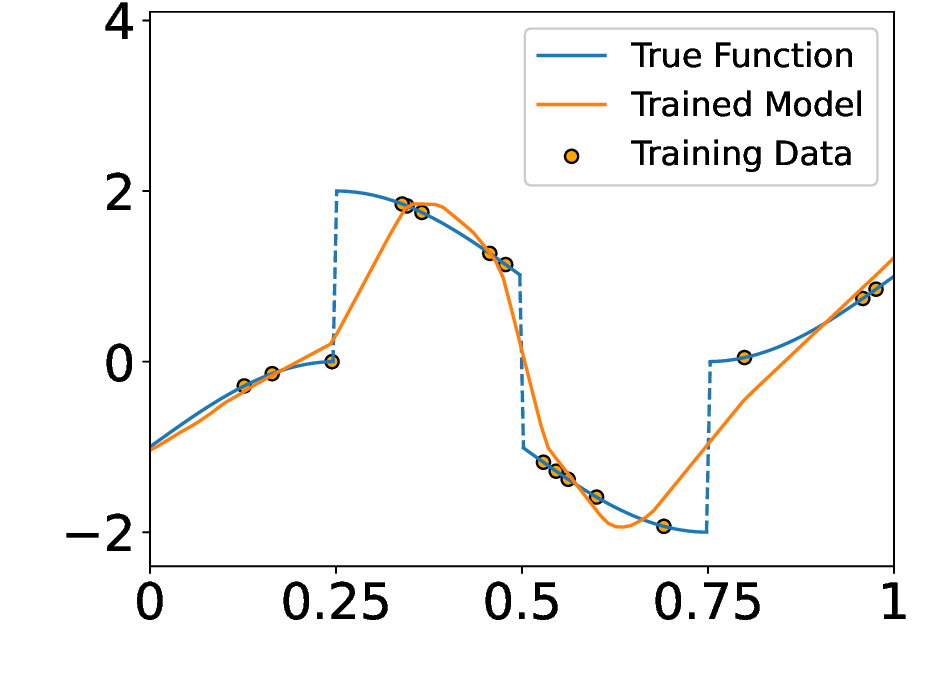}}
		\subfigure[$\sigma = 0, n=64$]{\includegraphics[width=0.3\textwidth]{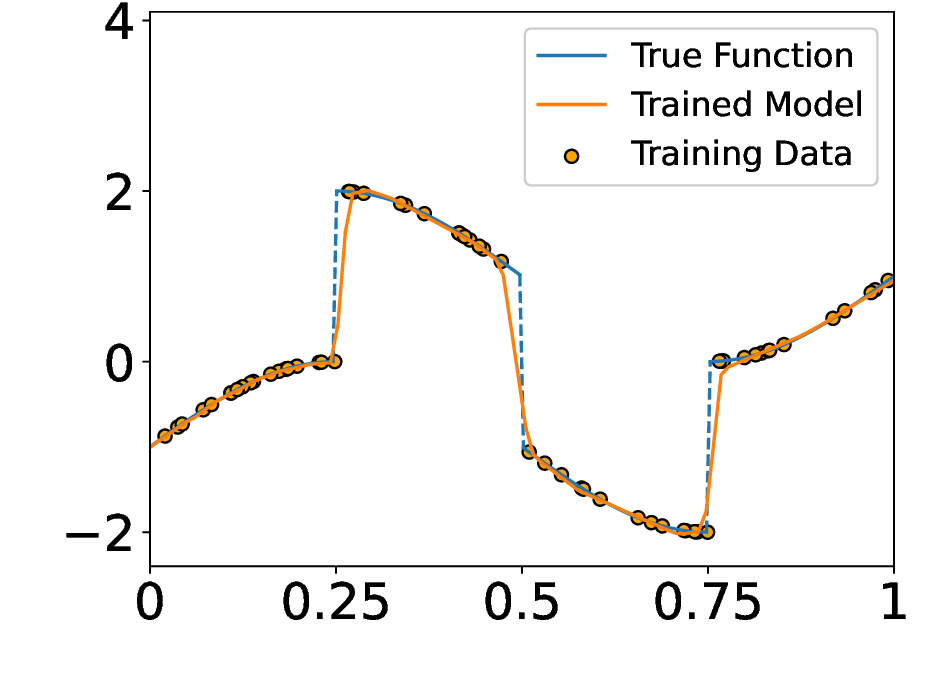}}
		\subfigure[$\sigma = 0, n=256$]{\includegraphics[width=0.3\textwidth]{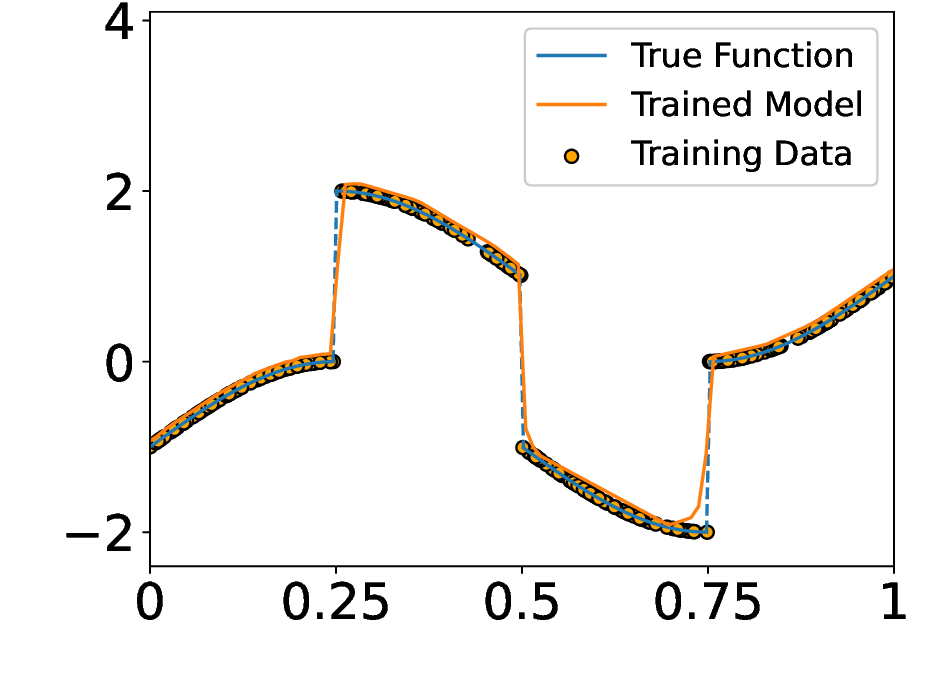}}
		\\
		\subfigure[$\sigma = 0.1, n=16$]{\includegraphics[width=0.3\textwidth]{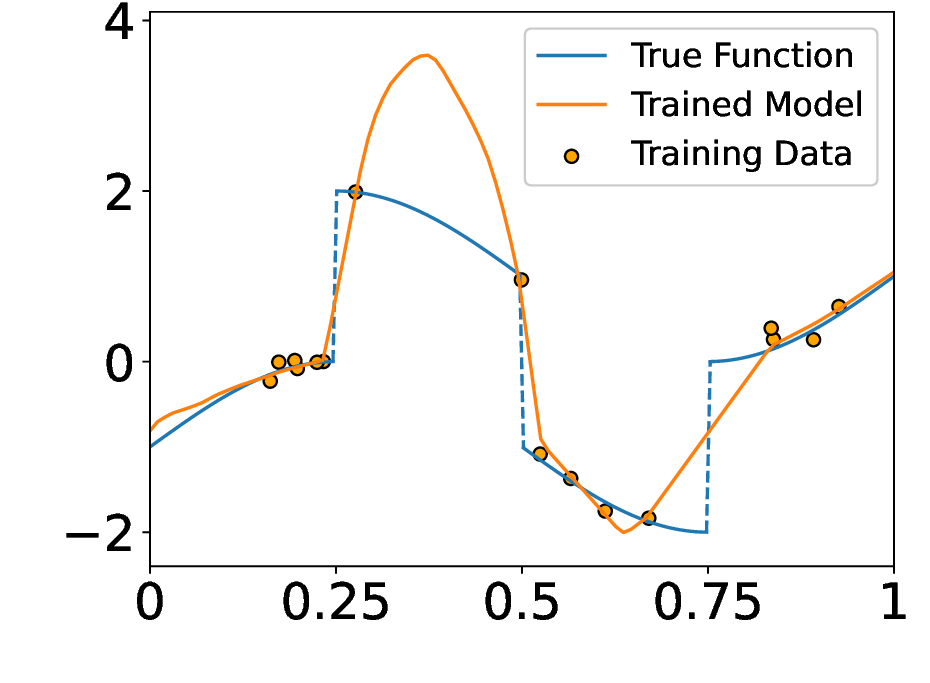}}
		\subfigure[$\sigma = 0.1, n=64$]{\includegraphics[width=0.3\textwidth]{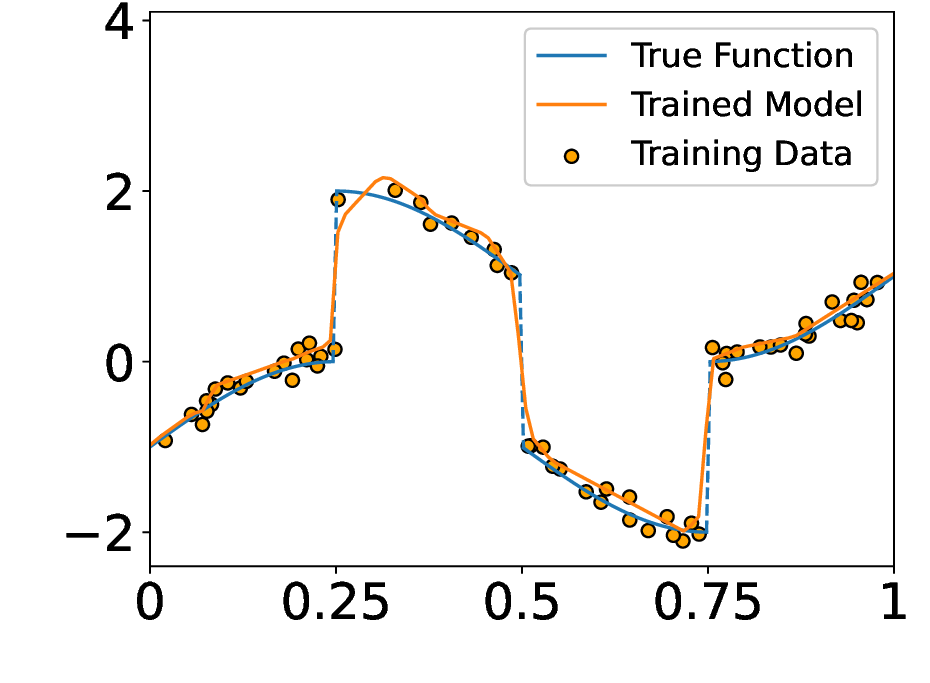}}
		\subfigure[$\sigma = 0.1, n=256$]{\includegraphics[width=0.3\textwidth]{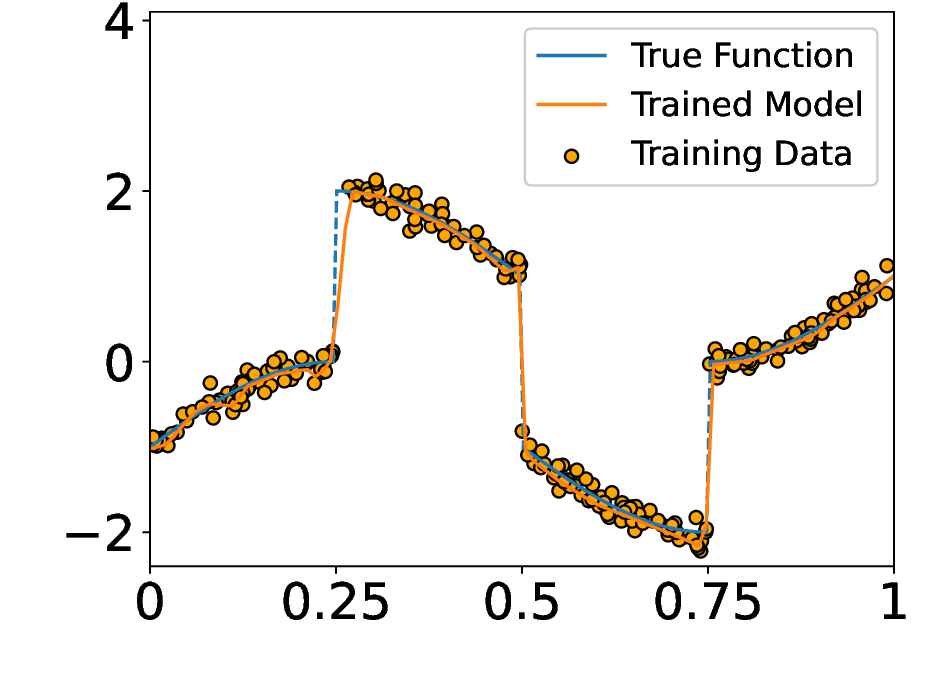}}
		\\
		\subfigure[$\sigma = 0.3, n=16$]{\includegraphics[width=0.3\textwidth]{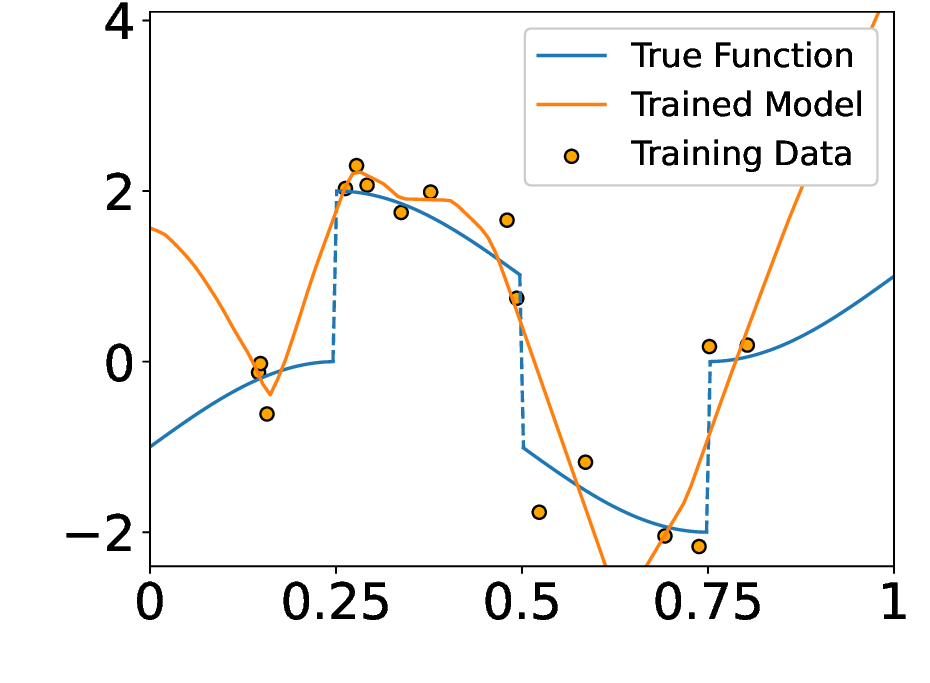}}
		\subfigure[$\sigma = 0.3, n=64$]{\includegraphics[width=0.3\textwidth]{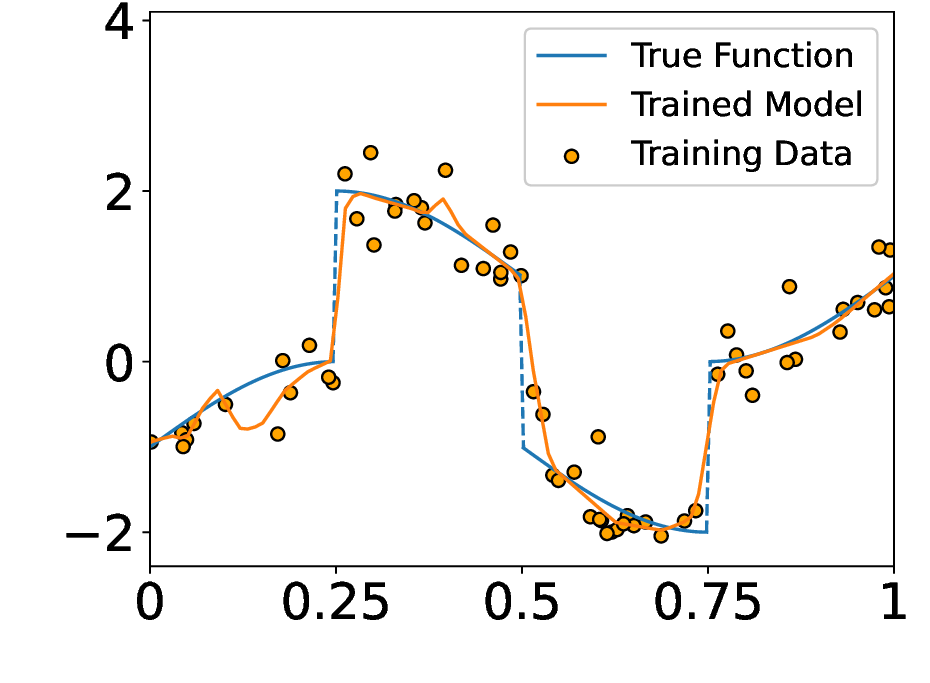}}
		\subfigure[$\sigma = 0.3, n=256$]{\includegraphics[width=0.3\textwidth]{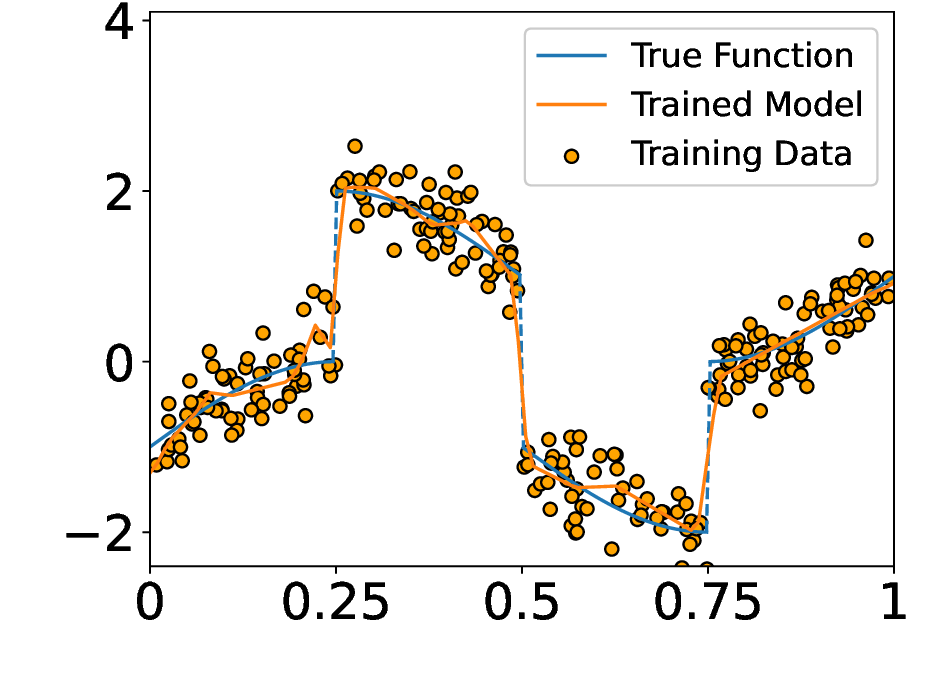}}
		\\
		\subfigure[$\sigma = 0.5, n=16$]{\includegraphics[width=0.3\textwidth]{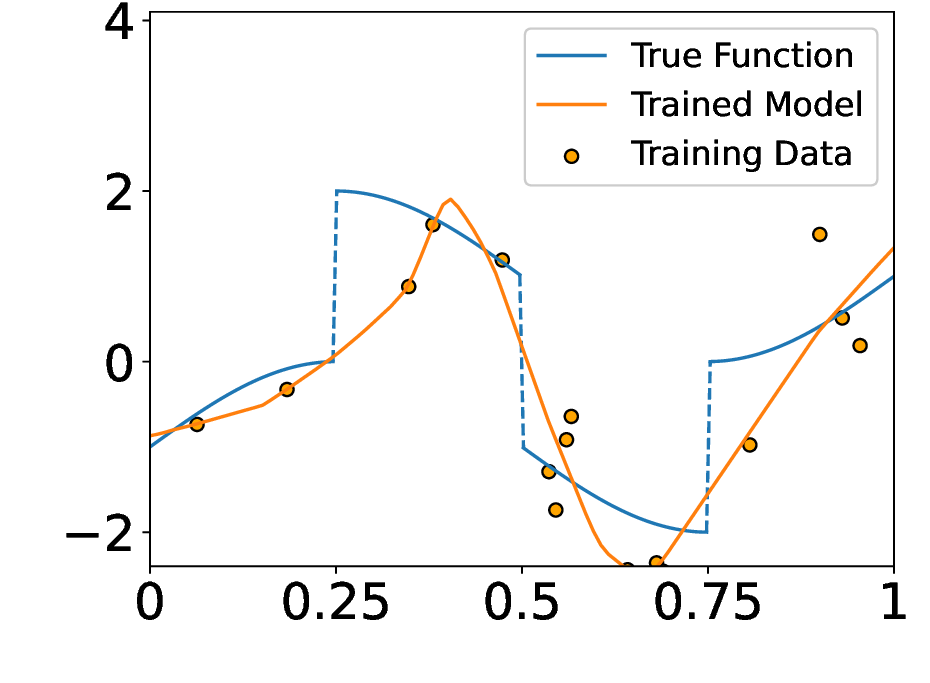}}
		\subfigure[$\sigma = 0.5, n=64$]{\includegraphics[width=0.3\textwidth]{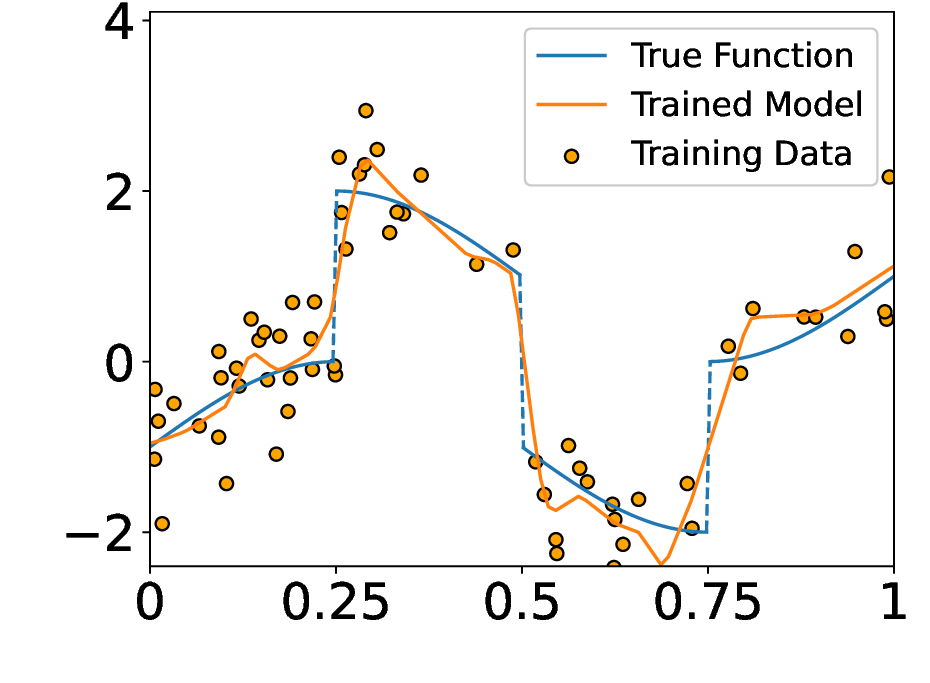}}
		\subfigure[$\sigma = 0.5, n=256$]{\includegraphics[width=0.3\textwidth]{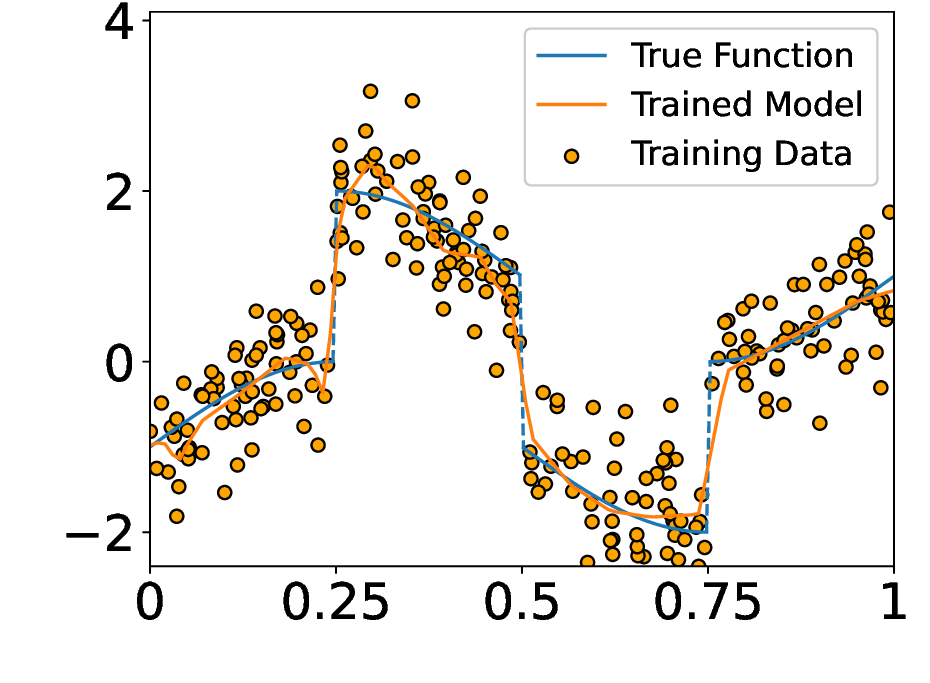}}
		\caption{Trained model with $n_{\rm train} =  16$ (1st column), $ 64$ (2nd column), $256$ (3rd column) when $\sigma = 0$ (1st row), $0.1$ (2nd row), $0.3$ (3rd row), and $0.5$ (4th row).}
		\label{fig:trained_model}
	\end{figure}

	\begin{figure}[t!]
		\centering
		\subfigure[$\sigma = 0$]{\includegraphics[width=0.44\textwidth]{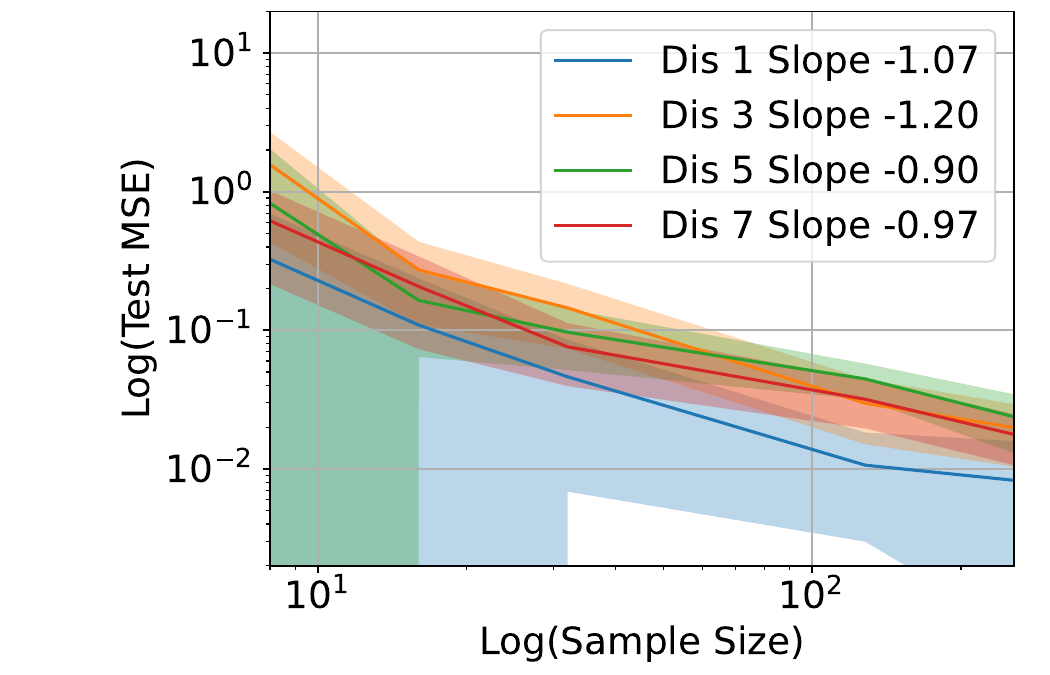}}
		\subfigure[$\sigma = 0.1$]{\includegraphics[width=0.44\textwidth]{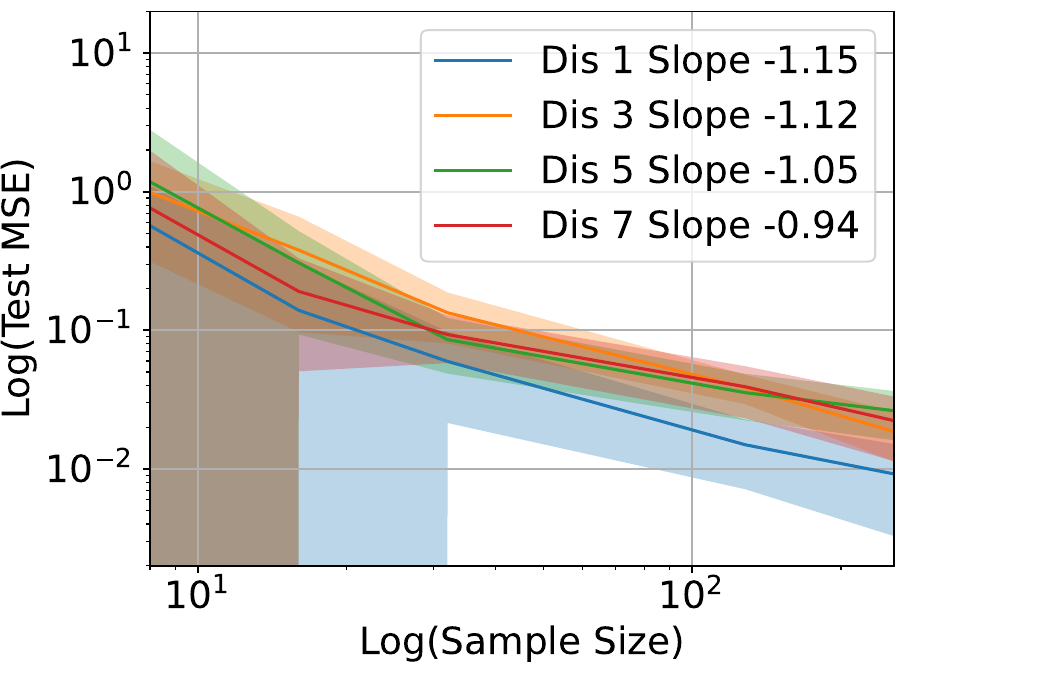}}
		\subfigure[$\sigma = 0.3$]{\includegraphics[width=0.44\textwidth]{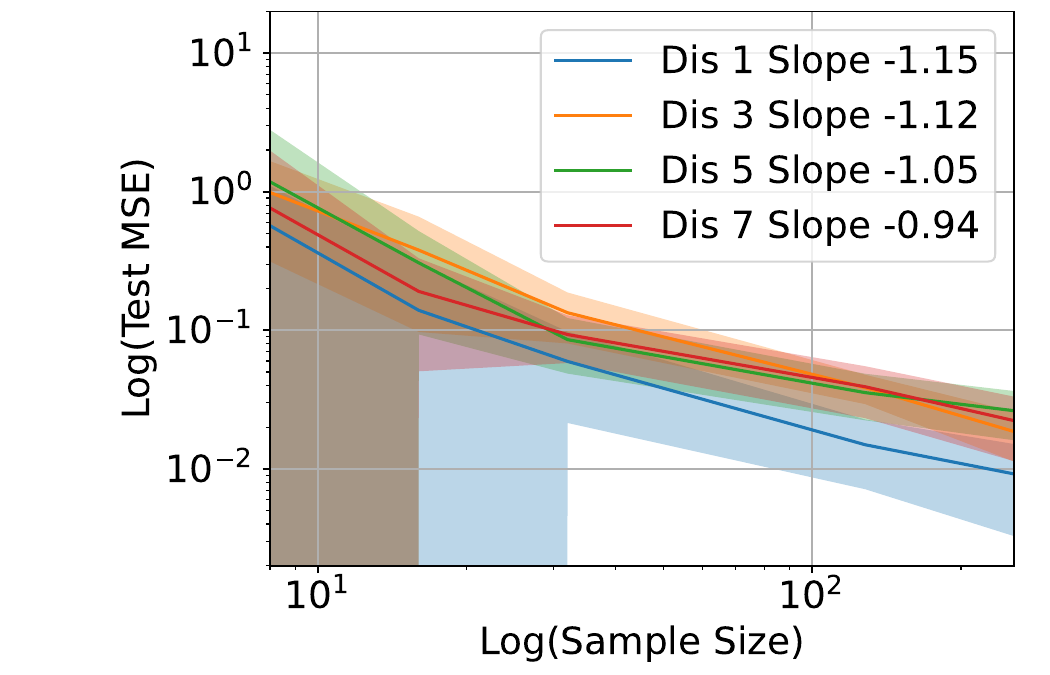}}
		\subfigure[$\sigma = 0.5$]{\includegraphics[width=0.44\textwidth]{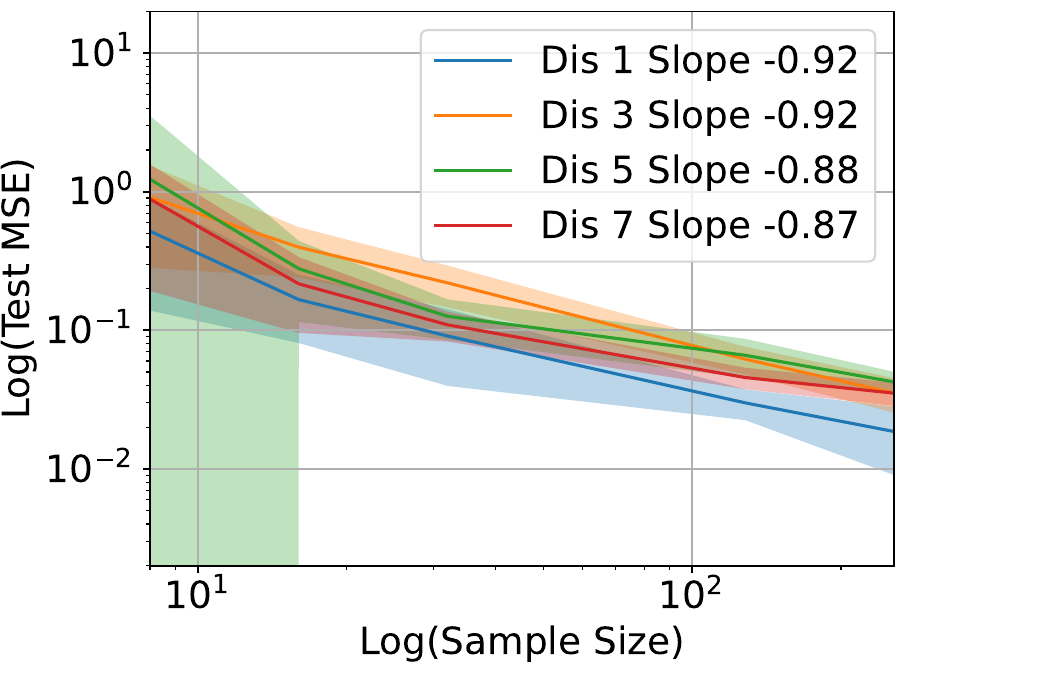}}
		\caption{
			Test MSE versus $n_{\rm train}$ in $\log$-$\log$ scale for the regression of functions with different number of discontinuity points (Dis) shown in Figure \ref{fig:TargetFunction}, when  $\sigma = 0$(a), $\sigma=0.1$(b), $\sigma=0.3$(c) and $\sigma=0.5$(d).
			We repeat 20 experiments for each setting. The curve represents the average test MSE in 20 experiments and the shade represents the standard deviation. A least-square fit of the curve gives rise to the slope in the legend.
		}
		\label{fig:TestMSEversusSample}
	\end{figure}
	These functions are shown in the Figure \ref{fig:TargetFunction}.
	In each experiment, we  sample $n_{\rm train}$ i.i.d. training samples $ \{x_i,  y_i\}_{i=1}^{n_{\rm train}}$ according to the model in \eqref{eqregmodel}.
	Specifically, the $x_i$'s are independently and uniformly sampled in $[0,1]$, and $y_i = f(x_i)+\xi_i$ with $\xi_i \sim \mathcal{N}(0,\sigma^2)$ being a normal random variable with zero mean and standard deviation $\sigma$. Given the training data, we train a neural network through 
	\begin{align}
		\widehat{f}=\argmin_{f_{\rm NN}\in \cF} \frac{1}{n_{\rm train}}\sum_{i=1}^{n_{\rm train} }|f_{\rm NN}(x_i)-y_i|^2,
		\label{eq.empirial.loss2}
	\end{align}
	where the ReLU neural  network class $\cF$ comprises four fully connected layer (64, 128, 64 neurons in each hidden layer).  All weight and bias parameters are initialized with a uniform distribution, and we use Adam optimizer with learning rate $0.001$ for training.
	
	Figure \ref{fig:trained_model} shows the ground-truth, training data, and the trained model with different number of training data and noise levels.  When $\sigma$ is fixed, the trained model approaches the ground-truth function when $n_{\rm train}$ increases, which is consistent with Theorem \ref{thm.general}.
	
	The test mean squared error (MSE) is evaluated on the test samples
	$\{x_j, f(x_j)\}_{n_{\rm test}}$ 
	\begin{align*}
		\text{Test MSE}= \frac{1}{n_{\rm test}}\sum_{j=1}^{n_{\rm test} }|\widehat{f}(x_j)-f(x_j)|^2
	\end{align*}
	with $n_{\rm test} = 10,000$. We use a large  $n_{\rm test}$ in order to reduced the variance in the evaluation of the test MSE.
	
	Figure \ref{fig:TestMSEversusSample} shows the Test MSE versus $n_{\rm train}$ in $\log$-$\log$ scale for the regression of functions with different number of discontinuity points shown in Figure \ref{fig:TargetFunction}, when  $\sigma = 0$ in (a), $\sigma=0.1$ in (b), $\sigma=0.3$ in (c) and $\sigma=0.5$ in (d), respectively.
	We repeat 20 experiments for each setting. The curve represents the average test MSE in 20 experiments and the shade represents the standard deviation. A least-square fit of the curve gives rise to the slope in the legend. These functions fit in Example \ref{ex.pieceiwiseholder1d.gene}, which follows 
	$\text{Test MSE} \lesssim {n_{\rm train}}^{-\frac{2r}{2r+1}}\log^3 n$ for a large $r$ since the functions are smooth except at the discontinuity points. Our Example \ref{ex.pieceiwiseholder1d.gene} predicts the slope about $-1$ in the $\log$-$\log$ plot of test MSE versus $n_{\rm train}$, and the slope is not sensitive to the number of (finite) discontinuity points. The numerical slopes in Figure \ref{fig:TestMSEversusSample} are consistent with our theory.

	\section{Proof of main results}
	\label{sec.proof}
	
	In this section, we present the proof of our main results. Some preliminaries for the proof are introduced in Subsection \ref{proof.pre}. Theorem \ref{thm.approx} is proved in Subsection \ref{proof.thm.approx} and  Theorem \ref{thm.general} is proved in Subsection \ref{proof.thm.general}.
	
	\subsection{Proof preliminaries}
	\label{proof.pre}
	
	We first introduce some preliminaries to be used in the proof.

	\subsubsection{Trapezoidal function and its neural network representation}
	
	{Given an interval $[a,b]\subset [0,1]$ and $0<\delta<b-a$, the function defined as
		\begin{align}
			&\psi_{[a,b]}(x)=\begin{cases}
				0 & \mbox{ if } x<a-\delta/2,\\
				\frac{x-(a-\delta/2)}{\delta} & \mbox{ if } a-\delta/2\leq x \leq a+\delta/2,\\
				1 & \mbox{ if } a+\delta/2< x < b-\delta/2,\\
				1-\frac{x-(b-\delta/2)}{\delta} & \mbox{ if } b-\delta/2\leq x \leq b+\delta/2,\\
				0 & \mbox{ if } x>b+\delta/2,
			\end{cases}  &&\mbox{ for } \quad a\neq0,\ b\neq 1,\nonumber\\
			&\psi_{[a,b]}(x)=\begin{cases}
				1 & \mbox{ if } a\leq  x < b-\delta/2,\\
				1-\frac{x-(b-\delta/2)}{\delta} & \mbox{ if } b-\delta/2\leq x \leq b+\delta/2,\\
				0 & \mbox{ if } x>b+\delta/2,
			\end{cases}  &&\mbox{ for } a=0, \nonumber\\
			&\psi_{[a,b]}(x)=\begin{cases}
				0 & \mbox{ if } x<a-\delta/2,\\
				\frac{x-(a-\delta/2)}{\delta} & \mbox{ if } a-\delta/2\leq x \leq a+\delta/2,\\
				1 & \mbox{ if } a+\delta/2< x \leq b
			\end{cases}  &&\mbox{ for } b=1
			\label{eq:psiab}
		\end{align} 
		is piecewise linear and supported on $[a-\delta/2,b+\delta/2]$ (or $[a,b+\delta/2]$ or $[a-\delta/2,b]$). In the rest of the proof, for simplicity, we only discuss the case for $0<a<b<1$. The case for $a=0$ or $b=1$ can be derived similarly.} Function $\psi_{[a,b]}$ can be realized by the following ReLU network with 1 layer and width 4:
	\begin{align*}
		\widetilde{\psi}_{[a,b]}(x)=\frac{1}{\delta}\Big(&\ReLU(x-(a-\delta/2))-\ReLU(x-(a+\delta/2)) \nonumber\\
		&- \ReLU(x-(b-\delta/2) + \ReLU(x-(b+\delta/2))\Big).
	\end{align*}

	\subsubsection{Multiplication operation and neural network approximation}
	The following lemma from \cite{yarotsky2017error} shows that the product operation can be well approximated  by a ReLU network.
	\begin{lemma}[Proposition 3 in \cite{yarotsky2017error}]
		For any $C>0$ and $0<\varepsilon<1$. If $|x|\leq C,|y|\leq C$, there is a ReLU network, denoted by $\ttimes(\cdot,\cdot)$, such that
		\begin{align*}
			&|\ttimes(x,y)-xy|<\varepsilon,
			\\
			\ttimes(x,0)=&\ttimes(y,0)=0, \ |\ttimes(x,y)|\leq C^2.
		\end{align*}
		Such a network has $O\left(\log \frac{1}{\varepsilon}\right)$ layers and parameters, where the constants hidden in $O$ depends on $C$. The width of each layer is bounded by 6 and all parameters are bounded by $O(C^2)$, where the constant hidden in $O$ is an  absolute constant. 
		\label{lem.multiplication}
	\end{lemma}
	Furthermore, the following lemma shows that composition of products can be well approximated  by a ReLU network (see a proof in Appendix \ref{proof.lem.multiprod})
	\begin{lemma}\label{lem.multiprod}
		Let $\{a_i\}_{i=1}^N$ be a set of  real numbers satisfying $|a_i|\leq C$ for any $i$. For any $0<\varepsilon<1$, there exists a neural network $\widetilde{\Pi}\in \cF_{\rm NN}(L,w,K,\kappa,M)$ such that
		\begin{align*}
			|\widetilde{\Pi}(a_1,...,a_N)-\prod_{i=1}^Na_i|\leq N\varepsilon. 
		\end{align*}
		The network $\cF_{\rm NN}(L,w,K,\kappa,M)$ has
		\begin{align}
			L=O\left(N\log \frac{1}{\varepsilon}\right), w=N+6, K=O\left(N\log \frac{1}{\varepsilon}\right), \kappa=O(C^N), M=C^N,
			\label{eq.multiprod.archi}
		\end{align}
		where the constant hidden in $O$ of $L$ and $K$ depends on $C$, and $\kappa$ is some absolute constant.
	\end{lemma}

	\subsection{Proof of Theorem \ref{thm.approx}}
	\label{proof.thm.approx}
	\begin{proof}[Proof of Theorem \ref{thm.approx}]
		To prove Theorem \ref{thm.approx}, we first  decompose the approximation error into two parts by applying the triangle inequality with the piecewise polynomial on the adaptive partition $p_{\Lambda(f,\eta)}$ defined in (\ref{eq.plambda}). The first part is the approximation error of $f$ by $p_{\Lambda(f,\eta)}$, which can be bounded by (\ref{eqaslapproxerror}). The second part is the network approximation error of $p_{\Lambda(f,\eta)}$. Then we show that $p_{\Lambda(f,\eta)}$ can be approximated by the given neural network with an arbitrary accuracy. Lastly, we estimate the total approximation error and quantify the network size. In the following we present the details of each step.

		\noindent$\bullet$ {\bf Decomposition of the approximation error.} For any $\widetilde{f}$ given by the network in (\ref{eq.approx}), we decompose the error as
		\begin{align}
			\|\widetilde{f}-f\|_{L^2(\rho)}^2\leq 2\|f-p_{\Lambda(f,\eta)}\|_{L^2(\rho)}^2 + 2\|\widetilde{f}-p_{\Lambda(f,\eta)}\|_{L^2(\rho)}^2,
			\label{eq.approx.tri}
		\end{align}
		where $\eta>0$ is to be determined later.
		The first term in (\ref{eq.approx.tri}) can be bounded by (\ref{eqaslapproxerror}) such that
		\begin{align}
			2\|f-p_{\Lambda(f,\eta)}\|_{L^2(\rho)}^2\leq  2C_s R_{\cA}^{2} (\#\calT(f,\eta))^{-{2s}}.
			\label{eq.approx.tree}
		\end{align}
		
		\noindent$\bullet$ {\bf Bounding the second term in (\ref{eq.approx.tri}).}
		We next derive an upper bound for the second term in (\ref{eq.approx.tri}) by showing that $p_{\Lambda(f,\eta)}$ can be  well approximated  by a network $\widetilde{f}$. This part contains four steps: 
		\begin{description}
			\item[Step 1] Estimate the finest scale of the truncated tree. 
			\item[Step 2] Construct a partition of unity of $X$ with respect to the truncated tree. Each element of the partition of unity is a network. 
			\item[Step 3] Based on the partition of unity, construct a network to approximate $p_{j,k}$ on each cube. 
			\item[Step 4] Estimate the approximation error. 
		\end{description}
		In the following we disucss details of each step.
		
		\noindent {\bf --- Step 1: Estimate the finest scale.}
		Denote the truncated tree and its outer leaves of $p_{\Lambda(f,\eta)}$ by $\cT$ and $\Lambda$ respectively for simplicity. Each $C_{j,k}\in \Lambda$ is a hypercube in the form of $\otimes_{\ell=1}^d[r_{\ell,j,k},r_{\ell,j,k}+2^{-j}]$, where $r_{\ell,j,k}\in [0,1]$ are scalars and  
		\begin{align*}
			\otimes_{\ell=1}^d[r_{\ell,j,k},r_{\ell,j,k}+2^{-j}]=[r_{1,j,k},r_{1,j,k}+2^{-j}]\times \cdots \times [r_{d,j,k},r_{d,j,k}+2^{-j}]
		\end{align*}
		is a hypercube with edge length $2^{-j}$ in $\RR^d$.
		
		We estimate the finest scale in $\Lambda$. Let $J>0$ be the largest integer such that $C_{J,k}\in \Lambda$ for some $k$. In other words, $2^{-J}$ is the finest scale of the cubes in $\Lambda(f,\eta)$.
		Let $C_1=R_{\cA}^{m}$ so that 
		\begin{align*}
			\sup_{\eta>0} \eta^m\#\cT(f,\eta)\leq R_{\cA}^m=C_1,
		\end{align*}
		with the $m$ given in (\ref{eqasldef}), which implies
		\begin{align}
			\eta\leq \left(\frac{C_1}{\# \cT(f,\eta)}\right)^{\frac{1}{m}}.
			\label{eq.approx.etaT}
		\end{align}
		For any $C_{j,k}\in \Lambda(f,\eta)$ and its parent $C_{j-1,k'}$, we have $\delta_{j-1,k'}>\eta$.
		Meanwhile, we have 
		$$
		\delta_{j-1,k'}\leq 2C^{\frac 1 2}_{\rho}\|f\|_{L^\infty(X)}|C_{j-1,k'}| \leq 2C_{\rho}^{\frac 1 2}R2^{-(j-1)d/2}.
		$$
		This implies 
		\begin{align}
			j<-\frac{2}{d}\log_2 \frac{\eta}{2C_{\rho}^{\frac 1 2}R}+1.
			\label{eq.approx.jeta}
		\end{align}
		Substituting (\ref{eq.approx.etaT}) into (\ref{eq.approx.jeta}) gives rise to
		\begin{align}
			j\leq \frac{2}{md} \log \#\cT(f,\eta)- \frac{2}{md} \log C_1+1+\frac{2}{d}\log_2 (2C_{\rho}^{\frac 1 2}R)\leq C_2\log \#\cT(f,\eta),
			\label{eq.approx.j.upper}
		\end{align}
		where $C_2$ is a constant depending on $d,m,C_{\rho}$ and $R$.
		Since (\ref{eq.approx.j.upper}) holds for any  $C_{j,k}\in \Lambda(f,\eta)$, we have 
		\begin{align}
			J\leq  C_2\log \# \cT(f,\eta).
			\label{eq.zeta}
		\end{align}

		\noindent {\bf --- Step 2: Construct a partition of unity.}
		Let $0<\delta\leq 2^{-(J+2)}$. For each $C_{j,k}\in \Lambda(f,\eta)$, we define two sets:
		\begin{align}
			&\mathring{C}_{j,k}=\otimes_{\ell=1}^d[r_{\ell,j,k}+\delta/2,r_{\ell,j,k}+2^{-j}-\delta/2]\cap X, \nonumber\\
			&\bar{C}_{j,k}=\left(\otimes_{\ell=1}^d[r_{\ell,j,k}-\delta/2,r_{\ell,j,k}+2^{-j}+\delta/2]\cap X\right)\backslash \mathring{C}_{j,k}.
			\label{eq.C.trap}
		\end{align}
		The $\mathring{C}_{j,k}$ set is in the interior of $\Cjk$, with $\delta/2$ distance to the boundary of $\Cjk$. The $\bar{C}_{j,k}$ set contains the boundary of $\Cjk$.
		The relations of $\bar{C}_{j,k},\mathring{C}_{j,k}$ and $\Cjk$ are illustrated in Figure \ref{fig.C}.
		\begin{figure}
			\centering
			\includegraphics[width=0.6\textwidth]{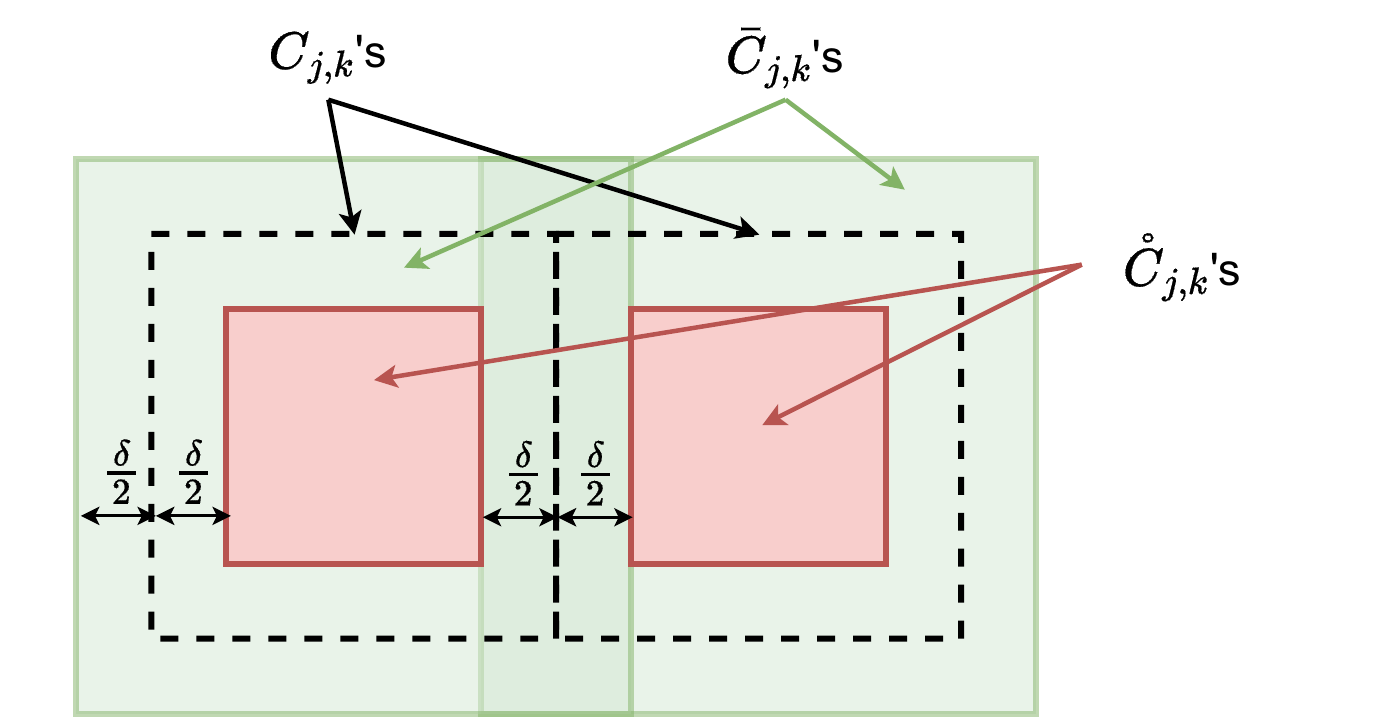}
			\caption{An illustration of the relation among $C_{j,k},\mathring{C}_{j,k}$ and $\bar{C}_{j,k}$.}
			\label{fig.C}
		\end{figure}

		For each $C_{j,k}$, we define the function
		\begin{align*}
			\phi_{j,k}(\xb)=\prod_{\ell=1}^d \psi_{[r_{\ell,j,k},r_{\ell,j,k}+2^{-j}]}(x_{\ell}),
		\end{align*}
		where  $\xb=[x_1,...,x_d]^{\top}$, and the $\psi$ function is defined in \eqref{eq:psiab}.
		The function $\phi_{j,k}$ has the following properties:
		\begin{enumerate}
			\item $\phi_{j,k}$ is piecewise linear. 
			\item $\phi_{j,k}$ is supported on $\mathring{C}_{j,k}\cup \bar{C}_{j,k}$, and  $\phi_{j,k}(\xb)=1$ when $\xb \in \mathring{C}_{j,k}$. 
			\item The $\phi_{j,k}$'s form a partition of unity of $X$: $\sum_{C_{j,k}\in \Lambda} \phi_{j,k}(\xb)=1$ when  $\xb \in X=[0,1]^d$.
		\end{enumerate}

		In this paper, we approximate $\phi_{j,k}(\xb)$ by 
		\begin{align*}
			\widetilde{\phi}_{j,k}(\xb)=&\widetilde{\Pi}\bigg(\psi_{[r_{1,j,k},r_{1,j,k}+2^{-j}]}(x_{1}),\psi_{[r_{2,j,k},r_{2,j,k}+2^{-j}]}(x_{2}), \cdots, \psi_{[r_{d,j,k},r_{d,j,k}+2^{-j}]}(x_{d})\bigg),
		\end{align*}
		where the network $\widetilde{\Pi}$ with $d$ inputs is defined in Lemma \ref{lem.multiprod} with accuracy $d\varepsilon_1$.
		We have $\widetilde{\phi}_{j,k}\in \cF_{1}=\cF(L_1,w_1,K_1,\kappa_1,M_1)$
		with
		\begin{align*}
			L_1=O\left(d\log \frac{1}{\varepsilon_1}\right), w_1=d+6, K_1=O\left(d\log \frac{1}{\varepsilon_1}\right), \kappa_1=O\left(\frac{1}{\delta}\right), M_1=1.
		\end{align*}
		By Lemma \ref{lem.multiprod}, we have
		\begin{align}
			\sup_{\xb \in [0,1]^d} |\widetilde{\phi}_{j,k}(\xb)-\phi_{j,k}(\xb)|\leq d\varepsilon_1.
			\label{eq.tildephi.err}
		\end{align}

		\noindent {\bf --- Step 3: Approximate $p_{j,k}$ on each cube.}
		According to (\ref{eq.poly.bound.form}), for each $C_{j,k}\in \Lambda(f,\eta)$,
		$p_{j,k}(f)$ is a polynomial of degree $\theta$ and is in the form of 
		\begin{align}
			p_{j,k}=\sum_{|\balpha|\leq \theta} a_{\balpha}(2^j(\xb-\rb_{j,k}))^{\balpha},
		\end{align}
		where $\rb_{j,k}=[r_{1,j,k},...,r_{d,j,k}]^{\top}$.

		We approximate $p_{j,k}$ by
		\begin{align*}
			\widetilde{p}_{j,k}(\xb)=\sum_{|\balpha|\leq \theta} a_{\balpha} \widetilde{\Pi}(\underbrace{2^j(x_1-r_{1,j,k}),...,2^j(x_1-r_{1,j,k})}_{\alpha_1 \mbox{ times }},...,\underbrace{2^j(x_d-r_{d,j,k}),...,2^j(x_d-r_{d,j,k})}_{\alpha_d \mbox{ times }}),
		\end{align*}
		where $\widetilde{\Pi}(\underbrace{2^j(x_1-r_{1,j,k}),...,2^j(x_1-r_{1,j,k})}_{\alpha_1 \mbox{ times }},...,\underbrace{2^j(x_d-r_{d,j,k}),...,2^j(x_d-r_{d,j,k})}_{\alpha_d \mbox{ times }})$ is the network approximation of $(2^j(\xb-\rb_{j,k}))^{\balpha}$ with accuracy $\theta\varepsilon_1$, according to Lemma \ref{lem.multiprod}.
		
		By Assumption \ref{assum.f}(ii), there exists  $R_p>0$ so that $|a_{\balpha}|$ is uniformly bounded by $R_p$ for any $|\balpha|\leq \theta$ for any $p_{j,k}$.
		By Lemma \ref{lem.multiprod}, we have
		\begin{align}
			&|\widetilde{p}_{j,k}(\xb)-p_{j,k}(\xb)| \nonumber \\
			\leq &\sum_{|\balpha|\leq \theta} a_{\balpha}\left| \widetilde{\Pi}(\underbrace{2^j(x_1-r_{1,j,k}),...,2^j(x_1-r_{1,j,k})}_{\alpha_1 \mbox{ times }},...,\underbrace{2^j(x_d-r_{d,j,k}),...,2^j(x_d-r_{d,j,k})}_{\alpha_d \mbox{ times }})-(2^j(\xb-\rb_{j,k}))^{\balpha}\right| \nonumber\\
			\leq &C_3R_p\theta \varepsilon_1
			\label{eq.tildep.err}
		\end{align}
		for some $C_3$ dpending on $d,\theta$, and $\widetilde{p}_{j,k}\in \cF_2=\cF(L_2,w_2,K_2,\kappa_2,M_2)$ with 
		\begin{align*}
			L_2=O(\theta\log \frac{1}{\varepsilon_1}), \ w_2=O(\theta d^\theta), \ K_2=O(\theta d^\theta\log \frac{1}{\varepsilon_1}), \ \kappa_2=O(2^{J }), \ M_2=R.
		\end{align*}
		
		\noindent {\bf --- Step 4: Estimate the network approximation error for $p_{\Lambda(f,\eta)}$.}
		We  approximate $$p_{\Lambda(f,\eta)}(\xb) =\sum_{\Cjk \in \Lambda(f,\eta)} \pjk(\xb)$$ by
		\begin{align}
			\widetilde{f}(\xb)=\sum_{C_{j,k}\in \Lambda(f,\eta)} \ttimes(\widetilde{\phi}_{j,k}(\xb),\widetilde{p}_{j,k}(\xb)),
		\end{align}
		where $\widetilde{\times}$ is the product network with accuracy $\varepsilon_1$, according to Lemma \ref{lem.multiplication}.
		
		Denote $X_1=\cup_{C_{j,k}\in \Lambda(f,\eta)} \mathring{C}_{j,k},\ X_2=\cup_{C_{j,k}\in \Lambda(f,\eta)}\bar{C}_{j,k}$.
		The error is estimated as
		\begin{align}
			\|\widetilde{f}-p_{\Lambda(f,\eta)}\|_{L^2(X)}^2=  \|\widetilde{f}-p_{\Lambda(f,\eta)}\|_{L^2(X_1)}^2 + \|\widetilde{f}-p_{\Lambda(f,\eta)}\|_{L^2(X_2)}^2.
			\label{eq.approx.split}
		\end{align}
		For the first term in (\ref{eq.approx.split}), we have 
		\begin{align}
			&\|\widetilde{f}-p_{\Lambda(f,\eta)}\|_{L^2(X_1)}^2 \\
			= &\sum_{C_{j,k}\in \Lambda(f,\eta)} \int_{\mathring{C}_{j,k}} |\ttimes(\widetilde{\phi}_{j,k}(\xb),\widetilde{p}_{j,k}(\xb))-p_{j,k}(\xb)|^2d\xb \nonumber\\
			\leq &\sum_{C_{j,k}\in \Lambda(f,\eta)} \int_{\mathring{C}_{j,k}} 3\Big[|\ttimes(\widetilde{\phi}_{j,k}(\xb),\widetilde{p}_{j,k}(\xb))-\widetilde{\phi}_{j,k}(\xb)\widetilde{p}_{j,k}(\xb)|^2 \nonumber \\
			&+ |\widetilde{\phi}_{j,k}(\xb)\widetilde{p}_{j,k}(\xb)-\phi_{j,k}(\xb)\widetilde{p}_{j,k}(\xb)|^2 +|\phi_{j,k}(\xb)\widetilde{p}_{j,k}(\xb)-p_{j,k}(\xb)|^2\Big]d\xb \nonumber \\
			\leq & \sum_{C_{j,k}\in \Lambda(f,\eta)} \int_{\mathring{C}_{j,k}} 3\Big[\varepsilon_1^2 + R^2d^2\varepsilon_1^2+|\widetilde{p}_{j,k}(\xb)-p_{j,k}(\xb)|^2 \Big]d\xb \nonumber\\
			\leq & 3(R^2d^2+1+C_3^2R_p^2\theta^2)\varepsilon_1^2 \sum_{C_{j,k}\in \Lambda(f,\eta)} |\mathring{C}_{j,k}| \nonumber \\
			\leq & 3(R^2d^2+1+C_3^2R_p^2\theta^2)\varepsilon_1^2.
			\label{eq.approx.X1}
		\end{align}
		In the derivation above, we used $\phi_{j,k}(\bx)=1$ when $\xb \in \mathring{C}_{j,k}$. Additionally, (\ref{eq.tildephi.err}) is used in the second inequality, (\ref{eq.tildep.err}) is used in the third inequality.

		We next derive an upper bound for the second term in (\ref{eq.approx.split}) by bounding the volume of $X_2$. We define the common boundaries of a set of cubes being the outer leaves of a truncated tree as the set of points that belong to at least two cubes.  
		We will use the following lemma to estimate the surface area for the common boundaries of $\Lambda(f,\eta)$ (see a proof in Appendix \ref{proof.lem.intersectionArea}).
		
		\begin{lemma}\label{lem.intersectionArea}
			Given a truncated tree $\cT$ and its outer leaves $\Lambda_{\cT}$, we denote the set of common boundaries  of the subcubes in  $\Lambda_{\cT}$ by $\cB(\Lambda_{\cT})$. The surface area of $\cB(\Lambda_{\cT})$ in $\RR^{d-1}$, denoted by $|\cB(\Lambda_{\cT})|$,  satisfies
			\begin{align}
				|\cB(\Lambda_{\cT})|\leq  2^{d+1}d(\# \cT)^{1/d}.
			\end{align}
		\end{lemma}

		By Lemma \ref{lem.intersectionArea} and Assumption \ref{assum.rho}, we have
		\begin{align}
			\|\widetilde{f}-p_{\Lambda(f,\eta)}\|_{L^2(X_2)}^2 \leq 4R^2|X_2| \leq 4R^2\delta |\cB(\Lambda_{\cT})|\leq 2^{d+3}dR^2\delta(\#\cT)^{1/d}.
			\label{eq.approx.X2}
		\end{align}

		Putting (\ref{eq.approx.X1}) and (\ref{eq.approx.X2}) together, we have
		\begin{align}
			\|\widetilde{f}-p_{\Lambda(f,\eta)}\|_{L^2(X)}^2\leq 3(R^2d^2+1+C_3^2R_p^2\theta^2)\varepsilon_1^2+2^{d+3}dR^2\delta(\#\cT)^{1/d}.
			\label{eq.approx.NN}
		\end{align}
		
		\noindent$\bullet$ {\bf Putting all terms together.}
		Set $\varepsilon_1=(\#\cT)^{-s}, \delta=\min\{ (\#\cT)^{-2s-1/d}, 2^{-\zeta-2}\}$, and substitute (\ref{eq.approx.NN}) and (\ref{eq.approx.tree}) into (\ref{eq.approx.tri}) gives rise to
		\begin{align}
			\|\widetilde{f}-f\|_{L^2(X)}^2\leq &2C_s R_{\cA}^2 (\#\calT(f,\eta))^{-{2s}}+ 6(R^2d^2+1+C_3^2R_p^2\theta^2)\varepsilon_1^2+2^{d+3}dR^2\delta\#\cT^{1/d} \nonumber\\
			\leq & \left(2C_sR_{\cA}^2+6(R^2d^2+1+C_3^2R_p^2\theta^2)+2^{d+3}dR^2\right)(\#\calT(f,\eta))^{-{2s}}.
		\end{align}
		We then quantify the network size of $\widetilde{f}$.
		\begin{itemize}
			\item[$\circ$] $\ttimes$: The product network has depth $O(\log \frac{1}{\varepsilon_1})$, width $6$, number of nonzero parameters $O(\log \frac{1}{\varepsilon_1})$, and all parameters are bounded by $R^2$.
			
			\item[$\circ$] $\widetilde{\phi}_{j,k}$: Each $\widetilde{\phi}_{j,k}$ has depth $O(d\log \frac{1}{\varepsilon_1})$, width $O(1)$, number of parameters $O(d\log \frac{1}{\varepsilon_1})$, and all parameters are bounded by $O(\frac{1}{\delta})$.
			
			\item[$\circ$] $\widetilde{p}_{j,k}$: Each $\widetilde{p}_{j,k}$ has depth $O(\theta\log \frac{1}{\varepsilon_1})$, width $O(\theta d^\theta)$, number of parameters $O(\theta d^\theta\log \frac{1}{\varepsilon_1})$, and all parameters are bounded by $O(2^{\zeta })$. 
		\end{itemize}
		Substituting the value of $\varepsilon_1$ and $\delta$, we get $\widetilde{f}\in \cF(L,w,K,\kappa,M)$ with
		\begin{align}
			&L=O(\log \#\cT(f,\eta)), \ w=O(\#\cT(f,\eta)), \ K=O(\#\cT(f,\eta)\log \#\cT(f,\eta)),\nonumber\\
			&\kappa=O((\#\cT(f,\eta))^{2s}+2^{\zeta }), \  M=R.
		\end{align}
		Note that by (\ref{eq.zeta}) we have
		\begin{align}
			2^{\zeta }<2^{C_2}\cdot 2^{\log \#\cT(f,\eta)} < 2^{C_2}\#\cT(f,\eta). 
		\end{align}
		We have $\kappa=O((\#\cT(f,\eta))^{\max\{2s,1\}})$.
		
		Setting 
		$$
		\#\calT(f,\eta))=\left(\frac{\sqrt{2C_sR_{\cA}^2+6(R^2d^2+1+C_3^2R_p^2\theta^2)+2^{d+3}dR^2}}{\varepsilon}\right)^{\frac{1}{s}}
		$$
		finishes the proof.
	\end{proof}

	\subsection{Proof of Theorem \ref{thm.general}}
	\label{proof.thm.general}
	\begin{proof}[Proof of Theorem \ref{thm.general}]
		Let $\widehat{f}$ be the minimizer of (\ref{eq.empirial.loss}). We decompose the squared generalization error as
		\begin{align}
			\EE_{\cS}\EE_{\xb\sim \rho} |\widehat{f}(\xb)-f(\xb)|^2= &\underbrace{2\EE_{\cS}\left[\frac{1}{n}\sum_{i=1}^n (\widehat{f}(\xb_i)-f(\xb_i))^2 \right]}_{\rm T_1} \\
			&+\underbrace{\EE_{\cS}\left[\int_{X} (\widehat{f}(\xb)-f(\xb))^2d\rho(\xb)\right]- 2\EE_{\cS}\left[\frac{1}{n}\sum_{i=1}^n (\widehat{f}(\xb_i)-f(\xb_i))^2 d \rho(\xb)\right]}_{\rm T_2}.
			\label{eq.general.decom}
		\end{align}
		
		\noindent{\bf $\bullet$ Bounding ${\rm T_1}$.}
		We bound ${\rm T_1}$ as
		\begin{align}
			{\rm T_1}=&2\EE_{\cS}\left[\frac{1}{n}\sum_{i=1}^n (\widehat{f}(\xb_i)-y_i+\xi_i)^2 \right] \nonumber\\
			=& 2\EE_{\cS}\left[\frac{1}{n} \sum_{i=1}^n (\widehat{f}(\xb_i)-y_i)^2\right] + 2\EE_{\cS}\left[\frac{1}{n} \sum_{i=1}^n \xi_i^2\right]+4\EE_{\cS}\left[\frac{1}{n} \sum_{i=1}^n\xi_i(\widehat{f}(\xb_i)-y_i)\right] \nonumber\\
			=& 2\EE_{\cS}\left[\inf_{f_{\rm NN}\in \cF}\left(\frac{1}{n} \sum_{i=1}^n (f_{\rm NN}(\xb_i)-y_i)^2\right)\right] + 2\EE_{\cS}\left[\frac{1}{n} \sum_{i=1}^n \xi_i^2\right]+4\EE_{\cS}\left[\frac{1}{n} \sum_{i=1}^n\xi_i(\widehat{f}(\xb_i)-f(\xb_i)-\xi_i)\right] \nonumber\\
			=& 2\EE_{\cS}\left[\inf_{f_{\rm NN}\in \cF}\left(\frac{1}{n} \sum_{i=1}^n (f_{\rm NN}(\xb_i)-y_i)^2\right)\right] - 2\EE_{\cS}\left[\frac{1}{n} \sum_{i=1}^n \xi_i^2\right] + 4\EE_{\cS}\left[\frac{1}{n} \sum_{i=1}^n\xi_i\widehat{f}(\xb_i))\right] \nonumber\\
			\leq& 2\inf_{f_{\rm NN}\in \cF}\EE_{\cS}\left[\frac{1}{n} \sum_{i=1}^n (f_{\rm NN}(\xb_i)-y_i)^2\right] - 2\EE_{\cS}\left[\frac{1}{n} \sum_{i=1}^n \xi_i^2\right] + 4\EE_{\cS}\left[\frac{1}{n} \sum_{i=1}^n\xi_i\widehat{f}(\xb_i))\right] \nonumber\\
			=& 2\inf_{f_{\rm NN}\in \cF}\EE_{\cS}\left[\frac{1}{n} \sum_{i=1}^n \left((f_{\rm NN}(\xb_i)-f(\xb_i)-\xi_i)^2-\xi_i^2\right)\right]  + 4\EE_{\cS}\left[\frac{1}{n} \sum_{i=1}^n\xi_i\widehat{f}(\xb_i))\right] \nonumber\\
			=& 2\inf_{f_{\rm NN}\in \cF}\EE_{\xb\sim \rho}\left[(f_{\rm NN}(\xb)-f(\xb))^2\right]  + 4\EE_{\cS}\left[\frac{1}{n} \sum_{i=1}^n\xi_i\widehat{f}(\xb_i))\right].
			\label{eq.T1.1}
		\end{align}
		In (\ref{eq.T1.1}), the first term is the network approximation error. The second term is a stochastic error arising from noise. By Theorem \ref{thm.approx}, we have an upper bound for the first term. Let $\varepsilon>0$. By Theorem \ref{thm.approx}, there exists a network architecture $\cF=\cF_{\rm NN}(L,w,K,\kappa,R)$ with
		\begin{align}
			&L=O(\log \frac{1}{\varepsilon}), \ w=O(\varepsilon^{-\frac{1}{s}}), \ K=O(\varepsilon^{-\frac{1}{s}}\log \frac{1}{\varepsilon}),\ \kappa=O(\varepsilon^{-\max\{2,\frac{1}{s}\}}), \  M=R
			\label{eq.general.network.1}
		\end{align}
		so that there is a network function $\widetilde{f}$ with this architecture satisfying
		\begin{align}
			\|\widetilde{f}-f\|_{L^2(\rho)}\leq \varepsilon,
		\end{align}
		where the constant hidden in $O$ depends on $d,C_s,C_{\rho},R,R_p,R_{\cA},\theta$. 
		
		We thus have
		\begin{align}
			2\inf_{f_{\rm NN}\in \cF}\EE_{\xb\sim \rho}\left[(f_{\rm NN}(\xb)-f(\xb))^2\right]\leq &2\EE_{\xb\sim \rho}\left[(\widetilde{f}(\xb)-f(\xb))^2\right] 
			= 2\|\widetilde{f}-f\|_{L^2(\rho)}^2 
			\leq 2\varepsilon^2.
			\label{eq.T1.term1}
		\end{align}
		
		Let $\cN(\delta,\cF,\|\cdot\|_{L^{\infty}(X)})$ be the covering number of $\cF$ under the  $\|\cdot\|_{L^{\infty}(X)}$ metric. Denote $\|f\|_n^2=\frac{1}{n}\sum_{i=1}^n (f(\xb_i))^2$. The following lemma gives an upper bound for the second term in (\ref{eq.T1.1}) (see a proof in Appendix \ref{proof.lem.T1.term2})
		\begin{lemma}\label{lem.T1.term2}
			Under the conditions of Theorem \ref{thm.general}, for any $\delta\in(0,1)$, we have
			\begin{align}
				\EE_{\cS}\left[\frac{1}{n} \sum_{i=1}^n\xi_i\widehat{f}(\xb_i))\right]\leq 2\sigma\left(\sqrt{ \EE_{\cS} \left[ \|\widehat{f}-f\|_n^2 \right]} + \delta \right) \sqrt{\frac{2 \log \cN(\delta,\cF,\|\cdot\|_{L^{\infty}(X)})+3}{n}}+ \delta\sigma.
				\label{eq.T1.term2}
			\end{align}
		\end{lemma}
		
		Substituting (\ref{eq.T1.term1}) and (\ref{eq.T1.term2}) into (\ref{eq.T1.1}) gives rise to
		\begin{align}
			{\rm T_1}=&2\EE_{\cS}\left[\| \widehat{f}-f\|_n^2 \right] \nonumber\\
			\leq &2\varepsilon^2 + 8\sigma\left(\sqrt{ \EE_{\cS} \left[ \|\widehat{f}-f\|_n^2 \right]} + \delta \right) \sqrt{\frac{2 \log \cN(\delta,\cF,\|\cdot\|_{L^{\infty}(X)})+3}{n}}+ 4\delta\sigma.
			\label{eq.T1.relation}
		\end{align}
		Let 
		\begin{align*}
			&v=\sqrt{ \EE_{\cS} \left[ \|\widehat{f}-f\|_n^2 \right]},\\
			&a=2\sigma\sqrt{\frac{2 \log \cN(\delta,\cF,\|\cdot\|_{L^{\infty}(X)})+3}{n}},\\
			&b=\varepsilon^2+4\delta\sigma \sqrt{\frac{2 \log \cN(\delta,\cF,\|\cdot\|_{L^{\infty}(X)})+3}{n}}+ 2\delta\sigma.
		\end{align*}
		Relation (\ref{eq.T1.relation}) can be written as
		$$
		v^2\leq 2av+b,
		$$
		which implies
		$$
		v^2\leq 4a^2+2b.
		$$
		Thus we have
		\begin{align}
			{\rm T_1}=2v^2\leq& 4\varepsilon^2+\left(16\sqrt{\frac{2 \log \cN(\delta,\cF,\|\cdot\|_{L^{\infty}(X)})+3}{n}}+ 8\right)\delta\sigma \nonumber\\
			&+ 16\sigma^2 \frac{2 \log \cN(\delta,\cF,\|\cdot\|_{L^{\infty}(X)})+3}{n}.
			\label{eq.T1}
		\end{align}
		
		\noindent{\bf $\bullet$ Bounding ${\rm T_2}$.}
		
		The following lemma gives an upper bound of ${\rm T_2}$ (see a proof in Appendix \ref{proof.lem.T2}):
		\begin{lemma}\label{lem.T2}
			Under the condition of Theorem \ref{thm.general}, we have
			\begin{align}
				{\rm T_2}\leq  \frac{35R^2}{n}\log \cN\left(\frac{\delta}{4R},\cF,\|\cdot\|_{L^{\infty}(X)}\right) +6\delta.
				\label{eq.T2}
			\end{align}
		\end{lemma}
		
		\noindent{\bf $\bullet$ Putting both terms together.}
		Substituting (\ref{eq.T1}) and (\ref{eq.T2}) into (\ref{eq.general.decom}) gives rise to
		\begin{align}
			&\EE_{\cS}\EE_{\xb\sim \rho} |\widehat{f}(\xb)-f(\xb)|^2 \nonumber\\
			\leq& 4\varepsilon^2+\left(16\sqrt{\frac{2 \log \cN(\delta,\cF,\|\cdot\|_{L^{\infty}(X)})+3}{n}}+ 8\right)\delta\sigma \nonumber\\
			&+ 16\sigma^2 \frac{2 \log \cN(\delta,\cF,\|\cdot\|_{L^{\infty}(X)})+3}{n}+\frac{35R^2}{n}\log \cN\left(\frac{\delta}{4R},\cF,\|\cdot\|_{L^{\infty}(X)}\right) +6\delta\\
			\leq &4\varepsilon^2+\left(16\sqrt{\frac{2 \log \cN\left(\frac{\delta}{4R},\cF,\|\cdot\|_{L^{\infty}(X)}\right)+3}{n}}+ 8\right)\delta\sigma \nonumber\\
			&+ \left(32\sigma^2+35R^2\right) \frac{ \log \cN(\frac{\delta}{4R},\cF,\|\cdot\|_{L^{\infty}(X)})+3}{n} +6\delta.
			\label{eq.general.2}
		\end{align}
		The covering number of $\cF$ can be bounded using network parameters, which is summarized in the following lemma:
		\begin{lemma}[Lemma 6 of \cite{chen2019nonparametric}]
			\label{lem.covering}
			Let $\cF=\cF(L,w,K,\kappa,M)$ be a class of networks: $[0,1]^d\rightarrow [-M,M]$. For any $\delta>0$, the $\delta$-covering number of $\cF_{\rm NN}$ is bounded by 
			\begin{align}
				\cN(\delta,\cF,\|\cdot\|_{L^{\infty}(X)}) \leq \left( \frac{2L^2(w+2)\kappa^Lw^{L+1}}{\delta} \right)^{K}.
			\end{align}
		\end{lemma}
		Substituting the network parameters in (\ref{eq.general.network.1}) into Lemma \ref{lem.covering} gives
		\begin{align}
			\log \cN\left(\frac{\delta}{4R},\cF, \|\cdot\|_{L^{\infty}(X)}\right) \leq C_1 \left( \varepsilon^{-\frac{1}{s}} \log^3 \varepsilon^{-1}+ \log \varepsilon^{-1}\right),
			\label{eq.general.logcover}
		\end{align}
		where $C_1$ is some constant depending on $d,C_s,C_{\rho},R,R_{\cA}$ and $\theta$.
		
		Substituting  (\ref{eq.general.logcover}) into (\ref{eq.general.2}) gives
		\begin{align}
			&\EE_{\cS}\EE_{\xb\sim \rho} |\widehat{f}(\xb)-f(\xb)|^2 \nonumber\\
			\leq & 4\varepsilon^2+\left(16\sqrt{\frac{2 C_1 \left( \varepsilon^{-\frac{1}{s}} \log^3 \varepsilon^{-1}+ \log \delta^{-1}\right)+3}{n}}+ 8\right)\delta\sigma \nonumber\\
			&+ \left(32\sigma^2+35R^2\right) \frac{ C_1 \left( \varepsilon^{-\frac{1}{s}} \log^3 \varepsilon^{-1}+ \log \delta^{-1}\right)+3}{n} +6\delta.
		\end{align}
		Setting $\delta=1/n$ and
		$$
		\varepsilon=n^{-\frac{s}{2s+1}}
		$$
		give rise to
		\begin{align}
			\EE_{\cS}\EE_{\xb\sim \rho} |\widehat{f}(\xb)-f(\xb)|^2\leq C_2n^{-\frac{2s}{1+2s}}\log^3n,
		\end{align}
		for some $C_2$ depending on $\sigma,d,C_s, C_{\rho},\theta,R, R_{\cA}$.
		
		The resulting network $\cF=\cF_{\rm NN}(L,w, K,\kappa,M)$ has parameters
		\begin{align}
			L=O(\log n), \ w=O(n^{\frac{1}{2s+1}}), \ K=O(n^{\frac{1}{2s+1}}\log n),\ \kappa=O(n^{\max\{\frac{2s}{2s+1}, \frac{1}{1+2s}\}}),\ M=R.
		\end{align}
	\end{proof}
	
	\subsection{Proof of the Examples in Section \ref{secexample}}
	\label{sec.exampleproof}
	\subsubsection{Proof of Example \ref{exampleholderinasl}}
	\label{sec.proof.exampleholderinasl}
	\begin{proof}[Proof of Example \ref{exampleholderinasl}]
		We first estimate $\delta_\jk$ for every $\Cjk$. Let $\theta=\lceil r-1\rceil$ and $\cjk$ be the center of the cube $\Cjk$ (each coordinate of $\cjk$ is the midpoint of the corresponding side of $\Cjk$). Denote $\tilde p^{(k)}_\jk$ be the $k$th order Taylor polynomial of $f$ centered at $\cjk$. By analyzing the tail of the Taylor polynomial, we obtain that, for every $x \in \Cjk$,
	\begin{equation} |f(\xb)-\tp_\jk^{(\theta)}(\xb)|  \le \frac{d^{\theta/2} \|f\|_{\calH^{r}(X)} \|\xb-\cjk\|^r}{\theta!} \le \frac{d^{\lceil r\rceil/2} \|f\|_{\calH^{r}(X)} 2^{-(j+1)r}}{\lceil r-1\rceil!}.
		\label{eqtaylorpointerror}
	\end{equation}
	The proof of \eqref{eqtaylorpointerror} is standard and can be found in \citet[Lemma 11.1]{gyorfi2002distribution} and \citet[Lemma 11]{liu2020learning}.
	The point-wise error above implies the following $L^2$ approximation error of the  $p_\jk$ in \eqref{eqpjk}:
	\begin{align}
		\|(f-p_\jk)\chi_{\Cjk}\|_{L^2(\rho)} 
		&\le  \|(f-\tp_\jk^{(\theta)})\chi_{\Cjk}\|_{L^2(\rho)} \nonumber
		\\
		&\le C_{\rho}^{\frac 1 2}2^{-jd/2} \max_{x \in \Cjk} |f(x)-\tp_\jk^{(\theta)}(x)| \nonumber
		\\
		&\le  C_{\rho}^{\frac 1 2}2^{-jd/2}\frac{d^{\lceil r\rceil/2} \|f\|_{\calH^{r}(X)} 2^{-(j+1)r}}{\lceil r-1\rceil!} .
		\label{eqpolynomialerr}
	\end{align}
	As a result, the refinement quantity $\delta_\jk$ satisfies
	\begin{equation}\delta_\jk(f) \le  C_1 2^{-j(r+d/2)}, \ \text{where}\ C_1=\frac{{2^{1-r}C_{\rho}^{\frac 1 2}d^{\lceil r\rceil/2} \|f\|_{\calH^{r}(X)}} }{\lceil r-1\rceil!}.
		\label{eqholderdelta}
	\end{equation}
	Notice that $C_1$ depends on $r,d,C_\rho$ and $ \|f\|_{\calH^{r}(X)}$.
	For any $\eta>0$, the nodes of $\calT$ with $\delta_\jk >\eta$ satisfy $2^{-j} >(\eta/C_1)^{\frac{2}{2r+d}}$. The cardinality of $\calT(f,\eta)$ satisfies
	\begin{align*}\#\calT(f,\eta) 
		&\lesssim
		1+2^d+2^{2d} +\ldots +2^{jd} \ \text{with} \ 2^{-j} >(\eta/C_1)^{\frac{2}{2r+d}}
		\\
		&\le
		\frac{2^d2^{jd}}{2^d-1}
		\le 2\cdot 2^{jd} \le 2 \left(\frac{C_1}{\eta} \right)^{\frac{2d}{2r+d}}.\end{align*}
	Therefore, $\eta^{\frac{2}{2r/d+1}} \#\calT(f,\eta) \le 2C_1
	^{\frac{2d}{2r+d}}$ for any $\eta>0$, so that $f \in \calA^{r/d}_{\lceil r-1\rceil}$. Furthermore, since $C_1$ depends on $r,d,C_\rho$ and $ \|f\|_{\calH^{r}(X)}$, if $\|f\|_{\cH^r(X)}\leq 1$, then $|f|_{\cA^{r/d}_{\lceil r-1 \rceil}}\leq C(r,d,C_{\rho})$ for some $C(r,d,C_{\rho})$ depending on $r,d$ and $C_{\rho}$.
	
\end{proof}

\subsubsection{Proof of Example \ref{examplepieceholderinasl1d}}
\label{sec.proof.examplepieceholderinasl1d}
\begin{proof}[Proof of Example \ref{examplepieceholderinasl1d}]
	We first estimate the $\delta_\jk(f)$ for every interval (1D cube) $\Cjk$. There are two types of intervals: the first type does not intersect with the discontinuities $\cup_k\{t_k\}$ and the second type has intersection with $\cup_k\{t_k\}$. \\
	\noindent {\textit{The first type (Type I):}} When $\Cjk \cap (\cup_k\{t_k\}) = \emptyset$, we have
	$$\delta_\jk(f) \le  C_1 2^{-j(r+1/2)}, \ \text{where}\ C_1=\frac{2^{1-r}C_{\rho}^{\frac 1 2}d^{\lceil r\rceil/2} \max_k\|f\|_{\calH^{r}(t_k,t_{k+1})} }{\lceil r-1\rceil!},$$
	according to \eqref{eqholderdelta}. 
	\\
	\noindent {\textit{The second type (Type II):}} When $\Cjk\cap (\cup_k\{t_k\}) \neq \emptyset$, $f$ is irregular on $I$. 
	We have
	\begin{align}
		\delta_{j,k}=&\left\|\sum_{C_{j+1,k'} \in \calC(\Cjk)} p_{j+1,k'}(f)\chi_{C_{j+1,k'}} - p_\jk(f) \chi_{\Cjk}\right\|_{L^2(\rho)} \nonumber\\
		=& \left\|\sum_{C_{j+1,k'} \in \calC(\Cjk)} (p_{j+1,k'}-f)\chi_{C_{j+1,k'}} - (p_\jk-f) \chi_{\Cjk}\right\|_{L^2(\rho)} \nonumber\\
		\leq& \sum_{C_{j+1,k'} \in \calC(\Cjk)} \left\|(p_{j+1,k'}-f)\chi_{C_{j+1,k'}}\right\|_{L^2(\rho)} + \left\|(p_\jk-f) \chi_{\Cjk}\right\|_{L^2(\rho)} \nonumber\\
		\leq &\sum_{C_{j+1,k'} \in \calC(\Cjk)} \left\|(0-f)\chi_{C_{j+1,k'}}\right\|_{L^2(\rho)} + \left\|(0-f) \chi_{\Cjk}\right\|_{L^2(\rho)} \nonumber\\
		\leq& \sum_{C_{j+1,k'} \in \calC(\Cjk)} \left\|f\chi_{C_{j+1,k'}}\right\|_{L^2(\rho)} + \left\|f \chi_{\Cjk}\right\|_{L^2(\rho)} \nonumber\\
		\leq& 2\|f\|_{L^{\infty}([0,1])}C_{\rho}^{\frac 1 2} |C_{j,k}|^{
			\frac 1 2} \nonumber\\
		= & C_2 2^{-j/2},
		\label{eq.delta.bound}
	\end{align}
	where $C_2 = 2C_{\rho}^{\frac 1 2}\|f\|_{L^{\infty}([0,1])}$.

	For any $\eta>0$, the master tree $\calT$ is truncated to $\calT(f,\eta)$. Consider the leaf node $\Cjk \in \calT(f,\eta)$. The type-I leaf nodes satisfy $2^{-j} > (\eta/C_1)^{\frac{2}{2r+1}}$. There are at most $2^{j_1}$ leaf nodes of Type I where $j_1$ is the largest integer with $2^{-j_1} > (\eta/C_1)^{\frac{2}{2r+1}}$.  
	The type-II leaf nodes in the truncated tree satisfy $2^{-j} > (\eta/C_2)^2$.
	Since there are at most $K$ discontinuity points, there are at most $K$ leaf nodes of Type II at scale $j_2$, where $j_2$ is the largest integer with $2^{-j_2}>(\eta/C_2)^2$, implying  $j_2 <  2\log_2\frac{C_2}{\eta}$. The cardinality of the outer leaf nodes of $\calT(f,\eta)$ can be estimated as
	\begin{align*}
		\#\Lambda(f,\eta)& = \# \left[\text{Outer leaf nodes of } \calT(f,\eta)\right]
		\\
		&\le 2\cdot 2^{j_1} + 2 j_2 K  \\
		&\le 2\cdot 2^{j_1} + 4\log_2\frac{C_2}{\eta}K \\
		&\le 2\left(\frac {C_1} \eta\right)^{\frac{2}{2r+1}} +4K\log_2\frac{C_2}{\eta}\\
		&\le 6\left(\frac C \eta\right)^{\frac{2}{2r+1}}
	\end{align*}
	when $\eta$ is sufficiently small, where $C$ is a constant depending on $r,d,C_{\rho},K$ and $\max_k \|f\|_{\cH^r(t_k,t_{k+1})}$.  Notice that
	$$\#\calT(f,\eta) \le \#\Lambda(f,\eta),$$
	because of \eqref{eq:cardinality}.
	Therefore, $\eta^{\frac{2}{2r+1}} \#\calT(f,\eta)\leq 6C^{\frac{2}{2r+1}}$ and we have $f\in \calA^{r}_{\lceil r-1\rceil}$.
	
	Furthermore, if $\max_k \|f\|_{\cH^r(t_k,t_{k+1})}\leq 1$, we have $|f|_{\cA^r_{r-1}}\leq C(r,d,C_{\rho},K)$ for some $C(r,d,C_{\rho},K)$ depending on $r,d,C_{\rho},K$.
	
\end{proof}

\subsubsection{Proof of Example \ref{examplepieceholderinasl}}
\label{sec.proof.examplepieceholderinasl}

\begin{proof}[Proof of Example \ref{examplepieceholderinasl}]
	We first estimate $\delta_\jk(f)$ for every cube $\Cjk$. There are two types of cubes: the first type belongs to the interior of some $\Omega_t$ and the second type has intersection with some $\partial \Omega_t$ (the boundary of $\Omega_t$). \\
	\noindent {\textit{The first type (Type I):}} When $\Cjk \subset \Omega_t^o$ for some $t$, we have
	$$\delta_\jk(f) \le  C_1 2^{-j(r+d/2)}, \ \text{where}\ C_1=\frac{2^{1-r}C_{\rho}^{\frac 1 2}d^{\lceil r\rceil/2} \max_t\|f\|_{\calH^{r}(\Omega_t^o)} }{\lceil r-1\rceil!},$$
	according to \eqref{eqholderdelta}.
	\\
	\noindent {\textit{The second type (Type II):}} When $\Cjk \cap \partial \Omega_t \neq \emptyset$ for some $t$, $f$ is irregular on $\Cjk$. Similar to (\ref{eq.delta.bound}), we have
	\begin{align*}
		\delta_{j,k}\leq  C_2 2^{-jd/2},
	\end{align*}
	where $C_2 = 2C_{\rho}^{\frac 1 2}\|f\|_{L^{\infty}(\Omega)}$.

	For any $\eta>0$, the master tree $\calT$ is truncated to $\calT(f,\eta)$. Consider the leaf node $\Cjk \in \calT(f,\eta)$. The type-I leaf nodes satisfy $2^{-j} > (\eta/C_1)^{\frac{2}{2r+d}}$. There are at most $2^{j_1d}$ leaf nodes of Type I where $j_1$ is the largest integer with $2^{-j_1} > (\eta/C_1)^{\frac{2}{2r+d}}$.

	We next estimate the number of type-II leaf nodes.
	The type-II leaf nodes satisfy $2^{-j} >\left({\eta}/{C_2}\right)^{\frac{2}{d}}$. Let $j_2$ be the largest integer satisfying 
	$ 2^{-j_2} >\left({\eta}/{C_2}\right)^{\frac{2}{d}}$, which implies $j_2< \frac{2}{d}\log_2 (C_{\rho}/\eta)$. 
	We next count the number of dyadic cubes needed to cover $\cup_t\partial\Omega_t$, considering  $\cup_t\partial\Omega_t $ has an upper Minkowski dimension $d-1$.
	Let $c_M(\cup_t\partial\Omega_t)= \sup_{\varepsilon>0} \cN(\varepsilon,\cup_t\Omega_t,\|\cdot\|_{\infty})\varepsilon^{d-1}$ be the Minkowski dimension constant of $\cup_t\Omega_t$. 
	According to Definition \ref{def.mikowski}, for each $j>0$, there exists a collection of $S$ cubes $\{V_{k}\}_{k=1}^{S}$ of edge length $2^{-j}$ covering $\cup_t\partial\Omega_t$ and $S\leq c_M2^{j(d-1)}$. 
	Each $V_{k}$ at most intersects with $2^{2d}$ dyadic cubes in the master tree $\calT$ at scale $j$. Therefore, there are at most $c_M(\cup_t\partial\Omega_t)$ type-II nodes at scale $j$. In total, the number of type-II leaf node is no more than 
	\begin{align}
		&\sum_{j=0}^{j_2}2^{j(d-1)+2d} = c_M(\cup_t\partial\Omega_t)2^{2d}\frac{2^{j_2(d-1)}\cdot 2^{d-1}-1}{2^{d-1}-1} \leq  \bar{C}2^{2d+1}2^{j_2(d-1)}
	\end{align}
	for some $\bar{C}$ depending on $c_M$ and $d$.

	Finally, we count the outer leaf nodes of $\calT(f,\eta)$.
	The cardinality of the outer leaf nodes of $\calT(f,\eta)$ can be estimated as
	\begin{align*}
		\#\Lambda(f,\eta)& = \# \left[\text{Outer leaf nodes of } \calT(f,\eta)\right]
		\\
		&\le 2^d\cdot 2^{j_1d} + 2^d\cdot \bar{C}2^{2d+1}2^{j_2(d-1)} \\
		&<  2^d\left(\frac {C_1} \eta\right)^{\frac{2d}{2r+d}} +2^d\cdot \bar{C}2^{2d+1}\left(\frac{C_2}{ \eta}\right)^\frac{2(d-1)}{d}
		\\
		&\leq \widetilde{C}\eta^{-\max\left\{\frac{2d}{2r+d},\frac{2(d-1)}{d}\right\}}
	\end{align*}
	for some $\widetilde{C}$ depending on $r,d,c_M(\cup_t\partial\Omega_t),C_{\rho}$ and $\max_t \|f\|_{\cH^r(\Omega_t^o)}$.
	Notice that
	$$\#\calT(f,\eta) \le  \#\Lambda(f,\eta).$$ 
	We have $\eta^{\max\left\{\frac{2d}{2r+d},\frac{2(d-1)}{d}\right\}}\#\calT(f,\eta)\leq \widetilde{C} $. Thus $f\in \cA^s_{\lceil r-1\rceil}$ with $s=\min\left\{ \frac{r}{d}, \frac{1}{2(d-1)}\right\}$. 
	
	Furthermore, if $\max_t \|f\|_{\cH^r(\Omega_t^o)}\leq 1$, then $$|f|_{\cA^s_{\lceil r-1 \rceil}}\leq C(r,d,c_M(\cup_t\partial\Omega_t),C_{\rho})$$ for some $C(r,d,c_M(\cup_t\partial\Omega_t),C_{\rho})$ depending on $r,d,c_M(\cup_t\partial\Omega_t)$ and $C_\rho$. 
\end{proof}

\subsubsection{Proof of Example \ref{examplemeasure0}}
\label{sec.proof.examplemeasure0}

\begin{proof}[Proof of Example \ref{examplemeasure0}]
	We first estimate $\delta_\jk(f)$ for every cube $\Cjk$. There are two types of cubes: the first type belongs to $\Omega_\delta$ and the second type has intersection with $\Omega^\complement_\delta$. \\
	\noindent {\textit{The first type (Type I):}} When $\Cjk \subset \Omega_\delta$, we have
	$$\delta_\jk \le  C_1 2^{-j(r+d/2)},
	$$
	with the $C_1$ given in \eqref{eqholderdelta}.
	\\
	\noindent {\textit{The second type (Type II):}} When $\Cjk \cap \Omega_\delta^\complement \neq \emptyset$, $f$ may be irregular on $\Cjk$ but $\delta_\jk = 0$ since $\Cjk \subset \Omega^\complement$, when $j$ is sufficiently large.

	For sufficiently small $\eta>0$, the master tree $\calT$ is truncated to $\calT(f,\eta)$. The size of the tree is dominated by the nodes within $\Omega_\delta$. Therefore, $\eta^{\frac{2}{2r/d+1}} \#\calT(f,\eta) \le 2C
	^{\frac{2d}{2r+d}}$ and $f\in \calA^{r/d}_{\lceil r-1\rceil}$. Furthermore, since $C$ depends on $r,d,C_\rho$ and $ \|f\|_{\calH^{r}(\Omega_{\delta})}$, if $\|f\|_{\cH^r(\Omega_{\delta})}\leq 1$, then $|f|_{\cA^{r/d}_{\lceil r-1 \rceil}}\leq C(r,d,C_{\rho})$ for some $C(r,d,C_{\rho})$ depending on $r,d$ and $C_{\rho}$.
	
\end{proof}

\subsubsection{Proof of Example \ref{examplelowd}}
\label{sec.proof.examplelowd}

\begin{proof}[Proof of Example \ref{examplelowd}]
	We first estimate $\delta_\jk(f)$ for every cube $\Cjk$. There are two types of cubes: the first type intersects with $\Omega$ and the second type has no intersection with $\Omega$. \\
	\noindent {\textit{The first type:}} When $I \cap \Omega \neq \emptyset$, thanks to \eqref{eqtaylorpointerror}, we have $$\delta_\jk^2(f) \le  \int_{\Cjk \cap \Omega}  \left(\frac{d^{\lceil r\rceil/2} \|f\|_{\calH^{r}(X)} 2^{-(j+1)r}}{\lceil r-1\rceil!}\right)^2 d \rho 
	= C_1^2 2^{-2jr}\rho(\Cjk \cap \Omega),  $$
	where $C_1 = {d^{\lceil r\rceil/2} \|f\|_{\calH^{r}(X)} 2^{-r}}/{\lceil r-1\rceil!}.$
	We next estimate $\rho(\Cjk \cap \Omega)$. Since $\Omega$ is a compact $d_{\rm in}$-dimensional Riemannian manifold isometrically embedded in $X$, $\Omega$ has a positive reach  $\tau>0$ \citep[Proposition 14]{thale200850}. Each $\Cjk$ is a $d$-dimensional cube of side length $2^{-j}$, and therefore $\Cjk \cap \Omega$ is contained in an Euclidean ball of diameter $\sqrt{d} 2^{-j}.$ We denote $\rho_{\Omega}$ as the conditional measure of $\rho$ on $\Omega$. According to \citet[Lemma  19]{maggioni2016multiscale}, when $j$ is sufficiently small such that  $\sqrt{d} 2^{-j}<\tau/8$,
	\begin{equation}
		\rho_\Omega(\Cjk \cap \Omega) \le \left(1+\left(\frac{2\cdot \sqrt{d} 2^{-j}}{\tau - 2\cdot \sqrt{d}  2^{-j}}\right)^2 \right)^{\frac d 2} \frac{{\rm Vol}(B_{\sqrt{d} 2^{-j}} (\mathbf{0}_{d_{\rm in}}))}{|\Omega|} \le  C_2^2  2^{-j {d_{\rm in}}},
		\label{eq:rhoomega}
	\end{equation}
	where $B_{\sqrt{d} 2^{-j}} (\mathbf{0}_{d_{\rm in}})$ denotes the Euclidean ball of radius $\sqrt{d} 2^{-j}$ centered at origin in $\RR^{d_{\rm in}}$, $|\Omega|$ is the surface area of $\Omega$, and $C_2$ is a constant depending on $d,\tau$ and $|\Omega|$.
	\eqref{eq:rhoomega} implies that 
	$$\delta_\jk(f) \le C_1 C_2 2^{-j(r+\frac{d_{\rm in}}{2})}.$$
	\\
	\noindent {\textit{The second type:}} When $\Cjk \cap \Omega = \emptyset$, $\rho_X(\Cjk) =0$ and then $\delta_\jk(f) =0.$

	For any $\eta>0$, the master tree $\calT$ is truncated to $\calT(f,\eta)$. The size of the tree is dominated by the nodes intersecting $\Omega$. The Type-I leaf nodes with $\delta_\jk(f) >\eta$ satisfy $2^{-j} \gtrsim \eta^{\frac{2}{2r+d_{\rm in}}}$. At scale $j$, there are at most $O(2^{jd_{\rm in}})$ Type-I leaf nodes. The cardinality of $\calT(f,\eta)$ satisfies
	\begin{align*}\#\calT(f,\eta) 
		&\lesssim \eta^{-\frac{2d_{\rm in}}{2r+d_{\rm in}}}.\end{align*}
	Therefore, $\sup_{\eta>0} \eta^{\frac{2\din}{2r+\din}} \#\calT(f,\eta) <\infty$, so that $f \in \calA^{r/\din}_{\lceil r-1\rceil}$. Furthermore, if $\|f\|_{\cH^r(X)}\leq 1$, we have $|f|_{\cA^{r/d_{\rm in}}_{\lceil r-1\rceil}} <C(r,d,d_{\rm in},\tau,|\Omega|)$ with $C(r,d,d_{\rm in},\tau,|\Omega|)$ depending on $r,d,d_{\rm in},\tau$ and $|\Omega|$.
	
\end{proof}
\section{Conclusion}
\label{sec.conclusion}
In this paper, we establish approximation and generalization theories for a large function class which is defined by nonlinear tree-based approximation theory. Such a function class allows the regularity of the function to vary at different locations and scales. It covers common function classes, such as H\"{o}lder functions and discontinuous functions, such as piecewise H\"{o}lder functions. Our theory shows that deep neural networks are adaptive to nonuniform regularity of functions and data distributions at different locations and scales.

When deep learning is used for regression, different network architectures can give rise to very different results. The success of deep learning relies on the optimization algorithm, initialization and a proper choice of network architecture. We will leave the computational study as our future research.

\appendix
\section*{Appendix}
\section{Proof of the approximation error in \eqref{eqaslapproxerror}}
\label{appendixproofeqaslapproxerror}

The approximation error in \eqref{eqaslapproxerror} can be proved as follows: 
\begin{align*}
	\|f-p_{\Lambda(f,\eta)}\|_{L^2(\rho)}^2
	& =\sum_{\Cjk \notin \calT(f,\eta)} \|\psi_\jk(f)\|_{L^2(\rho)}^2
	=\sum_{\ell \ge 0} \sum_{\Cjk \in \calT(f,2^{-(\ell+1)}\eta) \setminus \calT(f,2^{-\ell}\eta)}\|\psi_\jk(f)\|_{L^2(\rho)}^2
	\\
	&\le \sum_{\ell \ge 0} (2^{-\ell}\eta)^2 \#[\calT(f,2^{-(\ell+1)}\eta) ]
	\le \sum_{\ell \ge 0} (2^{-\ell}\eta)^2 |f|_{\asj}^m  [2^{-(\ell+1)}\eta]^{-m}
	\\
	& =[2^m \sum_{\ell \ge 0} 2^{\ell(m-2)}] |f|_{\asj}^m\eta^{2-m}= C_s |f|_{\asj}^m\eta^{2-m}
	\le C_s |f|_{\asj}^{2} (\#\calT(f,\eta))^{-{2s}}
\end{align*}
since $\eta^{2-m} \le |f|_{\asj}^{2-m}(\#\calT(f,\eta))^{-\frac{2-m}{m}} =  |f|_{\asj}^{2-m}(\#\calT(f,\eta))^{-{2s}}$ by the definition in \eqref{eqasldef}.

\section{Proof of Lemma \ref{lem.poly.uniform}}
\label{proof.lem.poly.uniform}
\begin{proof}[Proof of Lemma \ref{lem.poly.uniform}]
	
	Let $\cR=\{\xb^{\balpha}: |\balpha|=\alpha_1+\ldots+\alpha_d\leq \theta\}$, and $n_p=\#\cR$ be the cardinality of $\cR$.   
	Denote $\widetilde{\Omega}=[0,1]^d$. We first index the elements in $\cR$ according to $|\balpha|$ in the non-decreasing order. 
	One can obtain a set of orthonormal polynomials on $\widetilde{\Omega}$ from $\mathcal{R}$ by the Gram–Schmidt process. This  set of  polynomials forms an orthonormal basis for polynomials on $\widetilde{\Omega}$ with degree no more than $\theta$. Denote this orthonormal  set of polynomials  by $\{\widetilde{\phi}_{\ell}\}_{\ell=1}^{n_p}$. Each $\widetilde{\phi}_\ell$ can be written as
	\begin{align}
		\widetilde{\phi}_\ell (\xb)=\sum_{|\balpha|\leq \theta} \widetilde{b}_{\ell,\balpha}\xb^{\balpha}
		\label{eq.coeff.bound.basis.01}
	\end{align}
	for some $\{\widetilde{b}_{\ell,\balpha}\}_{\balpha}$. There exists a constant $C_1$ only depending on $\theta$ and $d$   so that 
	\begin{align}
		|\widetilde{b}_{\ell,\balpha}|\leq C_1 \quad \forall \ell=1,...,n_p \mbox{ and } |\balpha|\leq \theta.
		\label{eq.coeff.bound.b}
	\end{align}
	
	For the simplicity of notation, we denote $\Omega=C_{j,k}=[\ab,\bb]$ with $\ab=[a_1,...,a_d], \bb=[b_1,...,b_d]$ and $b_1-a_1=\cdots =b_d-a_d=h=2^{-j}$.
	The idea of this proof is to obtain a set of orthonormal basis on $\Omega$ from \eqref{eq.coeff.bound.basis.01}, where each basis is a linear combination of monomials of $\left(\frac{\xb-\ab}{h}\right)$. Then the coefficients of $\pjk$ can be expressed as inner product between $f$ and each basis.
	Let
	\begin{align}
		\phi_\ell(\xb)=\frac{1}{h^{d/2}}\widetilde{\phi}_\ell\left(\frac{\xb-\ab}{h}\right).
		\label{eq.coeff.bound.basis}
	\end{align}
	The $\phi_\ell$'s form a set of orthonormal polynomials on $\Omega$, since 
	\begin{align*}
		&\langle \phi_{\ell_1},\phi_{\ell_2}\rangle\\
		=&\int_{\Omega} \phi_{\ell_1}(\xb)\phi_{\ell_2}(\xb) d\xb\\
		=& \frac{1}{h^d} \int_{\Omega} \widetilde{\phi}_{\ell_1}\left(\left(\frac{\xb-\ab}{h}\right)\right) \widetilde{\phi}_{\ell_2}\left(\left(\frac{\xb-\ab}{h}\right)\right) d\xb\\
		=&\frac{h^d}{h^d} \int_{\Omega} \widetilde{\phi}_{\ell_1}\left(\left(\frac{\xb-\ab}{h}\right)\right) \widetilde{\phi}_{\ell_2}\left(\left(\frac{\xb-\ab}{h}\right)\right) d \left(\frac{\xb-\ab}{h}\right)
		\\
		=& \int_{\widetilde{\Omega}} \widetilde{\phi}_{\ell_1}\left(\xb\right) \widetilde{\phi}_{\ell_2}\left(\xb\right) d\xb\\
		=& \begin{cases}
			1 \mbox{ if } \ell_1=\ell_2,\\
			0 \mbox{ otherwise}.
		\end{cases}
	\end{align*}
	Thus $\{\phi_\ell\}_{\ell=1}^{n_p}$ form an orthonormal basis for polynomials with degree no more than $\theta$ on $\Omega$.
	The $p_{j,k}$ in (\ref{eqpjk}) has the form
	\begin{align}
		p_{j,k}=\sum_{\ell=1}^{n_p} c_\ell \phi_\ell \quad \mbox{ with } \quad c_\ell=\int_{\Omega} f(\xb) \phi_\ell(\xb) d\xb.
		\label{eq.coeff.bound.coeff}
	\end{align}

	Using H\"{o}lder's inequality, we  have
	\begin{align}
		|c_\ell|= \left|\langle f,\phi_\ell\rangle\right| \leq \|f\|_{L^2(\Omega)} \|\phi_\ell\|_{L^2(\Omega)}\leq R\sqrt{|\Omega|}\leq Rh^{d/2}.
		\label{eq.coeff.bound.c}
	\end{align}
	Substituting (\ref{eq.coeff.bound.basis}) and (\ref{eq.coeff.bound.basis.01}) into (\ref{eq.coeff.bound.coeff}) gives rise to
	\begin{align*}
		p_{j,k}(\xb)=\sum_{\ell=1}^{n_p} c_k \sum_{|\balpha|\leq \theta} \frac{\widetilde{b}_{k,\balpha}}{h^{d/2}}\left(\frac{\xb-\ab}{h}\right)^{\balpha}= \sum_{|\balpha|\leq \theta} \left(\sum_{\ell=1}^{n_p}  \frac{c_\ell\widetilde{b}_{\ell,\balpha}}{h^{d/2}}\right)\left(\frac{\xb-\ab}{h}\right)^{\balpha},
	\end{align*}
	implying that $$a_{\balpha}=\sum_{\ell=1}^{n_p} \frac{c_\ell\widetilde{b}_{\ell,\balpha}}{h^{d/2}}.$$
	Putting (\ref{eq.coeff.bound.b}) and (\ref{eq.coeff.bound.c}) together, we have
	\begin{align}
		|a_{\balpha}|\leq \sum_{k=1}^{n_p} \frac{|c_k| |\widetilde{b}_{k,\balpha}|}{h^{d/2}}\leq  \sum_{k=1}^{n_p} \frac{C_{1}Rh^{d/2} }{h^{d/2}}=C_1n_pR,
	\end{align}
	where $C_1$, as defined in (\ref{eq.coeff.bound.b}), is a constant depending on $\theta$.

	Furthermore, since $\|\left(\frac{\xb-\ab}{h}\right)^{\balpha}\|_{\infty}\leq 1$, we have 
	\begin{align*}
		\|p_{j,k}\|_{L^{\infty}(\Omega)}\leq C_1n_p^2R.
	\end{align*}
\end{proof}

\section{Proof of Lemma \ref{lem.multiprod}}
\label{proof.lem.multiprod}
\begin{proof}[Proof of Lemma \ref{lem.multiprod}]
	Denote the product $\times(a,b)=a\times b$. Let $\ttimes(\cdot,\cdot)$ be the network specified in Lemma \ref{lem.multiprod} with accuracy $\varepsilon$.
	We construct
	\begin{align}
		\widetilde{\Pi}(a_1,...,a_N)=\ttimes(a_1,\ttimes(a_2,\ttimes(\cdots,\ttimes(a_{N-1},a_N) \cdots)))
	\end{align}
	to approximate the multiplication operation  $\prod_{i=1}^N a_i=a_1 \times a_2 \times \cdots \times a_N$.
	The approximation error can be bounded as
	\begin{align}
		&|\widetilde{\Pi}(a_1,...,a_N)-\prod_{i=1}^N a_i| \nonumber\\
		=& |\ttimes(a_1,\ttimes(a_2,\ttimes(\cdots,\ttimes(a_{N-1},a_N) \cdots)))-\times(a_1,\times(a_2,\times(\cdots,\times(a_{N-1},a_N) \cdots)))| \nonumber\\
		\leq & |\ttimes(a_1,\ttimes(a_2,\ttimes(\cdots,\ttimes(a_{N-1},a_N) \cdots)))-\times(a_1,\ttimes(a_2,\ttimes(\cdots,\ttimes(a_{N-1},a_N) \cdots)))| \nonumber\\
		&+|\times(a_1,\ttimes(a_2,\ttimes(\cdots,\ttimes(a_{N-1},a_N) \cdots)))-\times(a_1,\times(a_2,\ttimes(\cdots,\ttimes(a_{N-1},a_N) \cdots)))| \nonumber\\
		&+\cdots \nonumber\\
		&+|\times(a_1,\times(a_2,\times(\cdots,\ttimes(a_{N-1},a_N) \cdots)))-\times(a_1,\times(a_2,\times(\cdots,\times(a_{N-1},a_N) \cdots)))| \nonumber\\
		\leq& N\varepsilon.
	\end{align}
	The network size is specified in (\ref{eq.multiprod.archi}).
\end{proof}

\section{Proof of Lemma \ref{lem.intersectionArea}}
\label{proof.lem.intersectionArea}
\begin{proof}[Proof of Lemma \ref{lem.intersectionArea}]

	Let $j^*$ be the smallest integer so that $2^{dj^*}\geq \#\Lambda_{\cT}$. Based on $\cT$, we first construct a $\cT'$ so that $\#\Lambda_{\cT'}=\#\Lambda_{\cT}$ and $j\leq j^*$ for any $C_{j,k}\in \Lambda_{\cT'}$ by the following procedure. 
	
	Note that if there exists $C_{j_1,k}\in \Lambda_{\cT}$ with $j_1>j^*$, there must be a $C_{j_2,k'}\in \Lambda_{\cT}$ with $j_2<j^*$. Otherwise, we must have $\#\Lambda_{\cT}>2^{dj^*}$, contradicting to the definition of $j^*$. 
	Let $C_{j_1,k}$ be a subcube in the finest scale of $\Lambda_{\cT}$, and $C_{j_2,k'}$ be a subcube in the coarsest scale of $\Lambda_{\cT}$. Suppose $j_1>j^*$. We have $j_1-j_2\geq 2$.     
	Denote the set of children and the parent of $C_{j,k}$ by  $\cC(C_{j,k})$ and $\cP(C_{j,k})$, respectively. Since $C_{j_1,k}$ is at the finest scale, we have $\cC(\cP(C_{j_1,k}))\subset \Lambda_{\cT}$. 
	By replacing $\cC(\cP(C_{j_1,k}))$ by $\cP(C_{j_1,k})$ and $C_{j_2,k'}$ by $\cC(C_{j_2,k'})$, we obtain a new tree $\cT_1$ with $\#\Lambda_{\cT_1}=\#\Lambda_{\cT}\leq 2^{dj^*}$.  Note that the subcubes in $\cC(\cP(C_{j_1,k}))$ have side length $2^{-j_1}$, $C_{j_2,k'}$ has side length $2^{-j_2}$. Since $j_1-j_2\geq 2$, we have
	\begin{align}
		|\cB(\Lambda_{\cT_1})|-|\cB(\Lambda_{\cT})|= d2^{-j_2(d-1)}- d2^{-j_1(d-1)}>0,
	\end{align}
	implying that $|\cB(\Lambda_{\cT})|<|\cB(\Lambda_{\cT_1})|$. Replace $\cT$ by $\cT_1$ and repeat the above procedure, we can generate a set of trees $\{\cT_m\}_{m=1}^M$ for some $M>0$ until $j\leq j^*$ for all $C_{j,k}\in \Lambda_{\cT_M}$. We have 
	$$
	|\cB(\Lambda_{\cT})|< |\cB(\Lambda_{\cT_1})|< \cdots < |\cB(\Lambda_{\cT_M})|.
	$$
	
	All leaf nodes of $\cT_M$ are at scale no larger than $j^*-1$. For any leaf node $C_{j,k}$ of $\cT_M$ with $j<j^*-1$, we partition it into its children. Repeat this process until all leaf nodes of the tree is at scale $j^*-1$, and denote the tree by $\cT^*$. Note that doing so only creates additional common boundaries, thus $\cB(\Lambda_{\cT_M})\subset \cB(\Lambda_{\cT^*})$, and we have
	$$
	|\cB(\Lambda_{\cT})|<|\cB(\Lambda_{\cT_M})|<|\cB(\Lambda_{\cT^*})|.
	$$

	We next compute $|\cB(\Lambda_{\cT^*})|$. Note that $\cB(\Lambda_{\cT^*})$ can be generated sequentially by slicing each cube at scale $j$ for $j=0,...,j^*-1$. When $C_{j-1,k}$ is sliced to get cubes at scale $j$, $d$ hyper-surfaces with area $2^{-(j-1)(d-1)}$ are created as common boundaries. There are in total $2^{d(j-1)}$ cubes at scale $j$. Thus we compute $|\cB(\Lambda_{\cT^*})|$ as: 
	\begin{align}
		|\cB(\Lambda_{\cT^*})|= \sum_{j=1}^{j^*} d2^{-(j-1)(d-1)}2^{d(j-1)}=\sum_{j=1}^{j^*} d2^{j-1}=d(2^{j^*}-1)\leq d2^{j^*}\leq 2^dd(\#\Lambda_{\cT})^{1/d} \leq 2^{d+1}d(\# \cT)^{1/d},
	\end{align}
	where we used  $ \#\Lambda_{\cT}\leq 2^{dj^*}\leq 2^d\#\Lambda_{\cT}$ in the second inequality according to the definition of $j^*$. 
\end{proof}

\section{Proof of Lemma \ref{lem.T1.term2}}
\label{proof.lem.T1.term2}
\begin{proof}[Proof of Lemma \ref{lem.T1.term2}]
	Let $\cF^*=\left\{ f_{{\rm NN},j}\right\}_{j=1}^{\cN(\delta,\cF,\|\cdot\|_{L^{\infty}(X)})}$ be a $\delta$ cover of $\cF$. There exists $f_{\rm NN}^*\in \cF^*$ satisfying $\|f_{\rm NN}^*-\widehat{f}\|_{L^{\infty}(X)}\leq \delta$.
	
	We have
	\begin{align}
		\EE_{\cS}\left[\frac{1}{n} \sum_{i=1}^n\xi_i\widehat{f}(\xb_i))\right]= &\EE_{\cS}\left[\frac{1}{n} \sum_{i=1}^n\xi_i(\widehat{f}(\xb_i)-f_{\rm NN}^*(\xb_i) + f_{\rm NN}^*(\xb_i)-f(\xb_i))\right] \nonumber\\
		\leq & \EE_{\cS}\left[\frac{1}{n} \sum_{i=1}^n\xi_i( f_{\rm NN}^*(\xb_i)-f(\xb_i))\right] + \EE_{\cS}\left[\frac{1}{n} \sum_{i=1}^n|\xi_i|\left|\widehat{f}(\xb_i)-f_{\rm NN}^*(\xb_i)\right|\right] \nonumber\\
		\leq& \EE_{\cS}\left[\frac{\|f_{\rm NN}^*-f\|_n}{\sqrt{n}}\frac{\sum_{i=1}^n\xi_i\left(f_{\rm NN}^*(\xb_i)-f_{\rm NN}^*(\xb_i)\right)}{\sqrt{n}\|f_{\rm NN}^*-f\|_n}\right] + \delta\sigma \nonumber\\
		\leq & \sqrt{2} \EE_{\cS}\left[\frac{\|\widehat{f}-f\|_n+\delta}{\sqrt{n}}\left|\frac{\sum_{i=1}^n\xi_i\left(f_{\rm NN}^*(\xb_i)-f_{\rm NN}^*(\xb_i)\right)}{\sqrt{n}\|f_{\rm NN}^*-f\|_n}\right|\right] + \delta\sigma.
		\label{eq.T2.term2.1}
	\end{align}
	In (\ref{eq.T2.term2.1}), the first inequality follows from Cauchy-Schwarz inequality, the second inequality holds by Jensen's inequality and
	\begin{align}
		\EE_{\cS}\left[\frac{1}{n} \sum_{i=1}^n|\xi_i|\left|\widehat{f}(\xb_i)-f_{\rm NN}^*(\xb_i)\right|\right] \leq & \EE_{\cS}\left[\frac{1}{n} \sum_{i=1}^n|\xi_i|\left\|\widehat{f}(\xb_i)-f_{\rm NN}^*(\xb_i)\right\|_{L^{\infty}(X)}\right] \nonumber\\
		\leq& \delta \frac{1}{n} \sum_{i=1}^n \EE_{\cS}\left[\sqrt{\xi_i^2}\right] \nonumber\\
		\leq& \delta \frac{1}{n} \sum_{i=1}^n \sqrt{\EE_{\cS}\left[\xi_i^2\right]} \nonumber\\
		\leq& \delta \frac{1}{n} \sum_{i=1}^n \sqrt{\sigma^2} \nonumber\\
		=&\delta\sigma,
	\end{align}
	the last inequality holds since
	\begin{align}
		\|f_{\rm NN}^*-f\|_n=&\sqrt{\frac{1}{n} \sum_{i=1}^n (f_{\rm NN}^*(\xb_i)-\widehat{f}(\xb_i)+ \widehat{f}(\xb_i)-f(\xb_i))^2} \nonumber\\
		\leq & \sqrt{\frac{2}{n} \sum_{i=1}^n \left[(f_{\rm NN}^*(\xb_i)-\widehat{f}(\xb_i))^2+ (\widehat{f}(\xb_i)-f(\xb_i))^2\right]} \nonumber\\
		\leq & \sqrt{\frac{2}{n} \sum_{i=1}^n \left[\delta^2+ (\widehat{f}(\xb_i)-f(\xb_i))^2\right]} \nonumber\\
		\leq & \sqrt{2}\|\widehat{f}-f\|_n + \sqrt{2}\delta.
	\end{align}
	Denote $z_j=\frac{\sum_{i=1}^n\xi_i(\widehat{f}(\xb_i)-f_{\rm NN}^*(\xb_i))}{\sqrt{n}\|f_{\rm NN}^*-f\|_n}$. The first term in (\ref{eq.T2.term2.1}) can be bounded as
	\begin{align}
		&\sqrt{2} \EE_{\cS}\left[\frac{\|\widehat{f}-f\|_n+\delta}{\sqrt{n}}\left|\frac{\sum_{i=1}^n\xi_i(\widehat{f}(\xb_i)-f_{\rm NN}^*(\xb_i))}{\sqrt{n}\|f_{\rm NN}^*-f\|_n}\right|\right] \nonumber\\
		\leq& \sqrt{2} \EE_{\cS}\left[\frac{\|\widehat{f}-f\|_n+\delta}{\sqrt{n}}\max_j |z_j|\right] \nonumber\\
		=&\sqrt{2} \EE_{\cS}\left[\frac{\|\widehat{f}-f\|_n}{\sqrt{n}}\max_j |z_j| +\frac{\delta}{\sqrt{n}}\max_j |z_j|\right] \nonumber\\
		= & \sqrt{2} \EE_{\cS}\left[\sqrt{\frac{1}{n}\|\widehat{f}-f\|_n^2}\sqrt{\max_j |z_j|^2} +\frac{\delta}{\sqrt{n}}\sqrt{\max_j |z_j|^2}\right] \nonumber\\
		\leq& \sqrt{2}\sqrt{\frac{1}{n} \EE_{\cS} \left[ \|\widehat{f}-f\|_n^2 \right]} \sqrt{\EE_{\cS} \left[ \max_j |z_j|^2 \right]} +\frac{\delta}{\sqrt{n}}\sqrt{\EE_{\cS} \left[ \max_j |z_j|^2 \right]} \nonumber\\
		=& \sqrt{2}\left(\sqrt{\frac{1}{n} \EE_{\cS} \left[ \|\widehat{f}-f\|_n^2 \right]} + \frac{\delta}{\sqrt{n}} \right)\sqrt{\EE_{\cS} \left[ \max_j |z_j|^2 \right]},
		\label{eq.T2.term2.2}
	\end{align}
	where the second inequality comes from Jensen's inequality and Cauchy-Schwarz inequality.
	
	For given $\{\xb_i\}_{i=1}^n$, $z_j$ is a sub-Gaussian variable with variance proxy $\sigma^2$. 
	
	For any $t>0$, we have
	\begin{align}
		\EE_{\cS} \left[ \max_j |z_j|^2|\xb_1,...,\xb_n\right]= &\frac{1}{t} \log \exp\left( t\EE_{\cS} \left[ \max_j |z_j|^2 | \xb_1,...,\xb_n \right] \right) \nonumber\\
		\leq& \frac{1}{t} \log \EE_{\cS} \left[ \exp\left( t\max_j |z_j|^2 \right) | \xb_1,...,\xb_n\right] \nonumber\\
		\leq & \frac{1}{t} \log \EE_{\cS} \left[ \sum_j \exp\left(t|z_j|^2\right) |\xb_1,...,\xb_n \right] \nonumber\\
		\leq & \frac{1}{t} \log \cN(\delta,\cF,\|\cdot\|_{L^{\infty}(X)}) + \frac{1}{t} \log \EE_{\cS} \left[ \exp(t|z_1|^2\right) |\xb_1,...,\xb_n].
	\end{align}
	
	Since $z_1$ is a sub-Gaussian variable with parameter $\sigma$, we have
	\begin{align}
		\EE_{\cS} \left[ \exp(t|z_1|^2)|\xb_1,...,\xb_n\right]=&1+ \sum_{k=1}^{\infty}  \frac{t^k \EE_{\cS} \left[ z_1^{2k}|\xb_1,...,\xb_n \right]}{k!} \nonumber\\
		=& 1+ \sum_{k=1}^{\infty} \frac{t^k}{k!} \int_{0}^{\infty} \PP\left( |z_1|\geq \lambda^{\frac{1}{2k}}| \xb_1,...,\xb_n\right) d\lambda \nonumber\\
		\leq & 1+ 2 \sum_{k=1}^{\infty} \frac{t^k}{k!} \int_{0}^{\infty} \exp\left( -\frac{\lambda^{1/k}}{2\sigma^2} \right) d\lambda \nonumber\\
		=& 1+ \sum_{k=1}^{\infty}  \frac{2k(2t\sigma^2)^k}{k!} \Gamma_G(k) \nonumber\\
		=& 1+2\sum_{k=1}^{\infty} (2t\sigma^2)^k,
	\end{align}
	where $\Gamma_{G}$ represents the Gamma function. Set $t=(4\sigma^2)^{-1}$, we have
	\begin{align}
		\EE_{\cS} \left[ \max_j |z_j|^2|\xb_1,...,\xb_n\right]\leq 4\sigma^2 \log \cN(\delta,\cF,\|\cdot\|_{L^{\infty}(X)})+4\sigma^2\log 3
		\leq 4\sigma^2 \log \cN(\delta,\cF,\|\cdot\|_{L^{\infty}(X)})+6\sigma^2.
		\label{eq.T2.term2.3}
	\end{align}
	Combining (\ref{eq.T2.term2.1}), (\ref{eq.T2.term2.2}) and (\ref{eq.T2.term2.3}) proves the lemma.
\end{proof}

\section{Proof of Lemma \ref{lem.T2}} \label{proof.lem.T2}

\begin{proof}[Proof of Lemma \ref{lem.T2}]
	Denote $\widehat{g}(\xb)=(\widehat{f}(\xb)-f(\xb))^2$. We have $\|\widehat{g}\|_{L^{\infty}(X)}\leq 4R^2$.
	The term ${\rm T_{2}}$ can be written as
	\begin{align}
		{\rm T_2}=& \EE_{\cS} \left[ \EE_{\xb\sim \rho} [\widehat{g}(\xb)|\cS] -\frac{2}{n} \sum_{i=1}^n \widehat{g}(\xb_i) \right] \nonumber\\
		=& 2\EE_{\cS} \left[ \frac{1}{2}\EE_{\xb\sim \rho} [\widehat{g}(\xb)|\cS]- \frac{1}{n} \sum_{i=1}^n \widehat{g}(\xb_i) \right] \nonumber\\
		=& 2\EE_{\cS} \left[ \EE_{\xb\sim \rho} [\widehat{g}(\xb)|\cS]- \frac{1}{n} \sum_{i=1}^n \widehat{g}(\xb_i) -\frac{1}{2}\EE_{\xb\sim \rho} [\widehat{g}(\xb)|\cS]\right].
		\label{eq.T2.1}
	\end{align}
	A lower bound of $\frac{1}{2}\EE_{\xb\sim \rho} [\widehat{g}(\xb)|\cS]$ is derived as
	\begin{align}
		\EE_{\xb\sim \rho_X} [\widehat{g}(\xb)|\cS]=\EE_{\xb\sim \rho} \left[\frac{4R^2}{4R^2}\widehat{g}(\xb)|\cS\right] \geq \frac{1}{4R^2}\EE_{\xb\sim \rho} [\widehat{g}^2(\xb)|\cS].
		\label{eq.T2.g.lower}
	\end{align}
	Substituting (\ref{eq.T2.g.lower}) into (\ref{eq.T2.1}) gives rise to
	\begin{align}
		{\rm T_2}\leq 2\EE_{\cS} \left[ \EE_{\xb\sim \rho} [\widehat{g}(\xb)|\cS]- \frac{1}{n} \sum_{i=1}^n \widehat{g}(\xb_i) -\frac{1}{8R^2}\EE_{\xb\sim \rho} [\widehat{g}^2(\xb)|\cS]\right].
	\end{align}
	Define the set
	\begin{align}
		\cR=\left\{ g: g(\xb)=(f_{\rm NN}(\xb)-f(\xb) ) \mbox{ for } f_{\rm NN}\in \cF\right\}.
	\end{align}
	Denote $\cS'=\{\xb'_i\}_{i=1}^n$ be an independent copy of $\cS$. We have
	\begin{align}
		{\rm T_2}\leq &2\EE_{\cS} \left[ \sup_{g\in \cR}\left(\EE_{\cS'} \left[\frac{1}{n} \sum_{i=1}^n g(\xb'_i)\right]- \frac{1}{n} \sum_{i=1}^n g(\xb_i) -\frac{1}{8R^2}\EE_{\cS'} \left[\frac{1}{n} \sum_{i=1}^n g^2(\xb'_i)\right]\right)\right] \nonumber\\
		\leq& 2\EE_{\cS} \left[ \sup_{g\in \cR}\left(\EE_{\cS'} \left[\frac{1}{n} \sum_{i=1}^n \left(g(\xb'_i)-g(\xb_i)\right)\right] -\frac{1}{16R^2}\EE_{\cS,\cS'} \left[\frac{1}{n} \sum_{i=1}^n \left(g^2(\xb_i)+g^2(\xb'_i)\right)\right]\right)\right] \nonumber\\
		\leq & 2\EE_{\cS,\cS'}\left[ \sup_{g\in \cR} \left(\frac{1}{n} \sum_{i=1}^n \left( \left(g(\xb_i)-g(\xb_i')\right) - \frac{1}{16R^2} \EE_{\cS,\cS'}\left[g^2(\xb_i)+g^2(\xb'_i)\right] \right) \right) \right].
		\label{eq.T2.2}
	\end{align}
	Let $\cR^*=\left\{ g_j^*\right\}_{j=1}^{\cN(\delta,\cR,\|\cdot\|_{L^{\infty}(X)})}$ be a $\delta$-cover of $\cR$. For any $g\in \cR$, there exists $g^*\in \cR^*$ such that $\|g-g^*\|_{L^{\infty}(X)}\leq \delta$.
	
	We bound (\ref{eq.T2.2}) using $\cR^*$. The first term in (\ref{eq.T2.2}) can be bounded as
	\begin{align}
		g(\xb_i)-g(\xb_i')=& g(\xb_i)-g^*(\xb_i)+g^*(\xb_i)-g^*(\xb_i) + g^*(\xb_i') -g(\xb_i') \nonumber\\
		=& \left( g(\xb_i)-g^*(\xb_i)\right) + \left( g^*(\xb_i)-g^*(\xb_i')\right) + \left( g^*(\xb_i')-g(\xb_i')\right) \nonumber\\
		\leq& \left( g^*(\xb_i)-g^*(\xb_i')\right)+2\delta.
		\label{eq.T2.2.term1}
	\end{align}
	We then lower bound $g^2(\xb_i)+g^2(\xb'_i)$ as
	\begin{align}
		&g^2(\xb_i)+g^2(\xb'_i) \nonumber\\
		=& \left( g^2(\xb_i)-(g^*)^2(\xb_i) \right) + \left( (g^*)^2(\xb_i)-(g^*)^2(\xb'_i) \right) + \left( (g^*)^2(\xb_i')-g^2(\xb'_i) \right) \nonumber\\
		\geq & (g^*)^2(\xb_i)+ (g^*)^2(\xb'_i) - | g(\xb_i)-g^*(\xb_i)| | g(\xb_i)+g^*(\xb_i)| - | g^*(\xb_i')-g(\xb_i')|| g^*(\xb_i')+g(\xb_i')| \nonumber\\
		\geq& (g^*)^2(\xb_i) + (g^*)^2(\xb_i')-16R^2\delta.
		\label{eq.T2.2.term2}
	\end{align}
	Substituting (\ref{eq.T2.2.term1}) and (\ref{eq.T2.2.term2}) into (\ref{eq.T2.2}) gives rise to
	\begin{align}
		{\rm T_2} \leq &2\EE_{\cS,\cS'}\left[ \sup_{g^*\in \cR^*} \left(\frac{1}{n} \sum_{i=1}^n \left( \left(g^*(\xb_i)-g^*(\xb_i')\right) - \frac{1}{16R^2} \EE_{\cS,\cS'}\left[(g^*)^2(\xb_i)+(g^*)^2(\xb'_i)\right] \right) \right) \right]+ 6\delta \nonumber\\
		=&2\EE_{\cS,\cS'}\left[ \max_j \left(\frac{1}{n} \sum_{i=1}^n \left( \left(g^*(\xb_i)-g^*(\xb_i')\right) - \frac{1}{16R^2} \EE_{\cS,\cS'}\left[(g^*)^2(\xb_i)+(g^*)^2(\xb'_i)\right] \right) \right) \right]+ 6\delta.
	\end{align}
	Denote $h_j(\xb_i,\xb_i')=g_j^*(\xb_i)-g_j^*(\xb_i')$. We have
	\begin{align*}
		\EE_{\cS,\cS'}\left[h_j(\xb_i,\xb_i')\right]=&0, \\
		\var\left[h_j(\xb_i,\xb_i' \right]= &\EE_{\cS,\cS'}\left[h_j^2(\xb_i,\xb_i')\right]\\
		= &\EE_{\cS,\cS'}\left[\left(g_j^*(\xb_i)-g_j^*(\xb_i')\right)^2\right] \\
		\leq &2\EE_{\cS,\cS'} \left[ (g_j^*)^2(\xb_i)+ (g_j^*)^2(\xb_i') \right].
	\end{align*}
	Thus ${\rm T_2}$ is bounded as
	\begin{align}
		&{\rm T_2}\leq \widetilde{\rm T}_2+6\delta \nonumber\\
		&\mbox{with } \widetilde{\rm T}_2= 2\EE_{\cS,\cS'}\left[ \max_j \left(\frac{1}{n} \sum_{i=1}^n \left( h_j(\xb_i,\xb_i') - \frac{1}{32R^2} \var[h_j(\xb_i,\xb_i')] \right) \right) \right].
	\end{align}
	Note that $\|h_j(\xb_i,\xb_i')\|_{L^{\infty}(X\times X)}\leq 4R^2$. We next study the moment generating function of $h_j$. For any $0<t<\frac{3}{4R^2}$, we have
	\begin{align}
		\EE_{\cS,\cS'}\left[\exp(th_j(\xb_i,\xb_i') \right]= & \EE_{\cS,\cS'} \left[ 1+th_j(\xb_i,\xb_i')+ \sum_{k=2}^{\infty} \frac{t^kh_j^k(\xb_i,\xb_i')}{k!} \right] \nonumber\\
		\leq& \EE_{\cS,\cS'} \left[1+th_j(\xb-i,\xb_i') + \sum_{k=2}^{\infty} \frac{(4R^2)^{k-2}t^kh_j^{k-2}(\xb_i,\xb_i')}{2\times 3^{k-2}}\right] \nonumber\\
		=&\EE_{\cS,\cS'} \left[ 1+th_j(\xb-i,\xb_i') +  \frac{t^2h_j^2(\xb_i,\xb_i')}{2} \sum_{k=2}^{\infty} \frac{(4R^2)^{k-2}t^{k-2}}{3^{k-2}}\right] \nonumber\\
		=&\EE_{\cS,\cS'}\left[ 1+th_j(\xb-i,\xb_i') +  \frac{t^2h_j^2(\xb_i,\xb_i')}{2} \frac{1}{1-4R^2t/3}\right] \nonumber\\
		=& 1+t^2\var\left[h_j(\xb_i,\xb_i')\right]\frac{1}{2-8R^2t/3} \nonumber\\
		\leq & \exp\left(\var\left[h_j(\xb_i,\xb_i')\right]\frac{3t^2}{6-8R^2t} \right),
		\label{eq.T2.moment}
	\end{align}
	where the last inequality comes from $1+x\leq \exp(x)$ for $x\geq0$.
	
	For $0<\frac{t}{n}<\frac{3}{4R^2}$, we have
	\begin{align}
		&\exp\left(\frac{t\widetilde{\rm T}_2}{2}\right) \nonumber\\
		=& \exp\left( t\EE_{\cS,\cS'} \left[\max_j\left(\frac{1}{n}\sum_{i=1}^n h_j(\xb_i,\xb_i')- \frac{1}{32R^2}\frac{1}{n}\sum_{i=1}^n \var[h_j(\xb_i,\xb_i')]\right)\right]\right) \nonumber\\
		\leq& \EE_{\cS,\cS'} \left[ \exp\left( t\max_j\left( \frac{1}{n} \sum_{i=1}^n h_j(\xb_i,\xb_i')- \frac{1}{32R^2}\frac{1}{n}\sum_{i=1}^n \var[h_j(\xb_i,\xb_i')] \right) \right) \right] \nonumber\\
		\leq& \EE_{\cS,\cS'} \left[ \sum_j\exp\left( \left( \frac{t}{n} \sum_{i=1}^n h_j(\xb_i,\xb_i')- \frac{1}{32R^2}\frac{t}{n}\sum_{i=1}^n \var[h_j(\xb_i,\xb_i')] \right) \right) \right] \nonumber\\
		\leq& \sum_j\exp\left( \left(  \sum_{i=1}^n \frac{3t^2/n^2}{6-8R^3t/n}\var[h_j(\xb_i,\xb_i')]- \frac{1}{32R^2}\frac{t}{n}\sum_{i=1}^n \var[h_j(\xb_i,\xb_i')] \right) \right) \nonumber\\
		=&\sum_j\exp\left( \left(  \sum_{i=1}^n \frac{t}{n}\var[h_j(\xb_i,\xb_i')]\left(\frac{3t^2/n^2}{6-8R^3t/n}- \frac{1}{32R^2}\right) \right) \right),
		\label{eq.T2.3}
	\end{align}
	where the first inequality follows from Jesen's inequality, the third inequality uses (\ref{eq.T2.moment}).
	
	Setting 
	\begin{align*}
		\frac{3t^2/n^2}{6-8R^3t/n}- \frac{1}{32R^2}=0
	\end{align*}
	gives rise to $t=\frac{3n}{52R^2}< \frac{3n}{4R^2}$. Substitute the choice of $t$ into (\ref{eq.T2.3}) gives 
	\begin{align*}
		\frac{t\widetilde{\rm T}_2}{2}\leq \log \sum_j \exp(0)= \log \cN(\delta,\cR,\|\cdot\|_{L^{\infty}(X)}).
	\end{align*}
	Therefore, we have
	\begin{align*}
		\widetilde{\rm T}_2\leq \frac{2}{t}\log \cN(\delta,\cR,\|\cdot\|_{L^{\infty}(X)}) = \frac{104R^2}{3n}\log \cN(\delta,\cR,\|\cdot\|_{L^{\infty}(X)})
	\end{align*}
	and
	\begin{align*}
		T_2\leq \frac{104R^2}{3n}\log \cN(\delta,\cR,\|\cdot\|_{L^{\infty}(X)})+ 6\delta \leq \frac{35R^2}{n}\log \cN(\delta,\cR,\|\cdot\|_{L^{\infty}(X)}) +6\delta.
	\end{align*}
	
	We then derive a relation between the covering number of $\cF_2$ and $\cR$. For any $g,g'\in \cR$, we have
	\begin{align*}
		g(\xb)=(f_{\rm NN}(\xb)-f(\xb))^2, \ g'(\xb)=(f'_{\rm NN}(\xb)-f(\xb))^2
	\end{align*}
	for some $f_{\rm NN},f'_{\rm NN}\in \cF$. We have
	\begin{align*}
		\|g-g'\|_{\infty}=&\sup_{\xb} \left|(f_{\rm NN}(\xb)-f(\xb))^2-(f'_{\rm NN}(\xb)-f(\xb))^2 \right| \nonumber\\
		=& \sup_{\xb} \left|(f_{\rm NN}(\xb)-f'_{\rm NN}(\xb))(f_{\rm NN}(\xb)+f'_{\rm NN}(\xb)- 2f(\xb)) \right| \nonumber\\
		\leq& \sup_{\xb} \left|f_{\rm NN}(\xb)-f'_{\rm NN}(\xb)\right|\left|f_{\rm NN}(\xb)+f'_{\rm NN}(\xb)- 2f(\xb) \right| \nonumber\\
		\leq& 4R\|f_{\rm NN}-f'_{\rm NN}\|_{L^{\infty}(X)}.
	\end{align*}
	As a result, we have
	\begin{align*}
		\cN(\delta,\cR,\|\cdot\|_{L^{\infty}(X)})\leq \cN\left(\frac{\delta}{4R},\cF,\|\cdot\|_{L^{\infty}(X)}\right)
	\end{align*}
	and the lemma is proved.
\end{proof}

\bibliographystyle{ims}
\bibliography{ref,reficml}
\end{document}